\documentclass{article}
\usepackage[english]{babel}
\usepackage[utf8]{inputenc}

\usepackage{microtype}
\usepackage{graphicx}
\usepackage{subfigure}
\usepackage{booktabs} 

\usepackage{amsmath}
\usepackage{amssymb}
\usepackage{mathtools}
\usepackage{amsthm}
\usepackage{algorithm}
\usepackage{algorithmic}
\usepackage{dsfont}
\usepackage{xcolor}

\usepackage{johd}
\usepackage[capitalize,noabbrev]{cleveref}
\theoremstyle{plain}
\newtheorem{theorem}{Theorem}[section]
\newtheorem{proposition}[theorem]{Proposition}
\newtheorem{lemma}[theorem]{Lemma}
\newtheorem{corollary}[theorem]{Corollary}
\theoremstyle{definition}
\newtheorem{definition}[theorem]{Definition}
\newtheorem{assumption}[theorem]{Assumption}
\theoremstyle{remark}
\newtheorem{remark}[theorem]{Remark}
\newcommand*\samethanks[1][\value{footnote}]{\footnotemark[#1]}
\definecolor{myred}{HTML}{ae1908}

\renewcommand{\algorithmicrequire}{\textbf{Input:}}
\renewcommand{\algorithmicensure}{\textbf{Output:}}

\newcommand{\dnote}[1]{\textcolor{red}{#1}}

\newcommand{\bbE}{\ensuremath{\mathbb{E}}}
\newcommand{\bbH}{\ensuremath{\mathbb{H}}}
\newcommand{\bbN}{\ensuremath{\mathbb{N}}}
\newcommand{\bbP}{\ensuremath{\mathbb{P}}}
\newcommand{\bbR}{\ensuremath{\mathbb{R}}}
\newcommand{\calA}{\ensuremath{\mathcal{A}}}
\newcommand{\calB}{\ensuremath{\mathcal{B}}}
\newcommand{\calC}{\ensuremath{\mathcal{C}}}
\newcommand{\calD}{\ensuremath{\mathcal{D}}}
\newcommand{\calE}{\ensuremath{\mathcal{E}}}
\newcommand{\calF}{\ensuremath{\mathcal{F}}}
\newcommand{\calG}{\ensuremath{\mathcal{G}}}

\newcommand{\calM}{\ensuremath{\mathcal{M}}}
\newcommand{\calN}{\ensuremath{\mathcal{N}}}
\newcommand{\calO}{\ensuremath{\mathcal{O}}}

\newcommand{\calR}{\ensuremath{\mathcal{R}}}

\newcommand{\calI}{\mathcal{I}}
\newcommand{\calH}{\mathcal{H}}
\newcommand{\calcolI}{\textcolor{blue}{\mathcal{I}}}
\newcommand{\calcolII}{\textcolor{blue}{\mathcal{II}}}

\newcommand{\lai}{\textcolor{blue}{\mathit{I}}}

\newcommand{\llangle}{\left\langle}
\newcommand{\rrangle}{\right\rangle}
\newcommand{\upA}{\mathbf{A}}
\newcommand{\upa}{\mathbf{a}}
\newcommand{\upB}{\mathbf{B}}
\newcommand{\upb}{\mathbf{b}}

\newcommand{\upD}{\mathbf{D}}
\newcommand{\upe}{\mathbf{e}}

\newcommand{\upH}{\mathbf{H}}
\newcommand{\upI}{\mathbf{I}}

\newcommand{\upq}{\mathbf{q}}

\newcommand{\upR}{\mathbf{R}}

\newcommand{\upu}{\mathbf{u}}
\newcommand{\upV}{\mathbf{V}}
\newcommand{\upv}{\mathbf{v}}
\newcommand{\upw}{\mathbf{w}}

\newcommand{\upx}{\mathbf{x}}

\newcommand{\upz}{\mathbf{z}}
\newcommand{\upLambda}{\mathbf{\Lambda}}
\newcommand{\upSigma}{\mathbf{\Sigma}}
\newcommand{\bi}{\mathrm{bias}}
\newcommand{\var}{\mathrm{variance}}

\newcommand{\sigmin}{\sigma_{\mathrm{min}}}
\newcommand{\Barsigmin}{\Bar{\sigma}_{\mathrm{min}}}
\newcommand{\Tildesigmax}{\Tilde{\sigma}_{\mathrm{max}}}
\newcommand{\Hatsigmax}{\Hat{\sigma}_{\mathrm{max}}}

\DeclareMathOperator*{\diag}{diag}
\DeclareMathOperator*{\tr}{tr}
\DeclareMathOperator*{\Var}{Var}
\DeclareMathOperator*{\Pro}{\mathbb{P}}

\title{Scaling Law for Stochastic Gradient Descent in Quadratically \\Parameterized Linear Regression}

\author{Shihong Ding$^{\dag}$\thanks{Equal Contribution. }\quad Haihan Zhang$^{\dag}$\samethanks\quad Hanzhen Zhao$^{\dag}$\quad Cong Fang$^{\dag}$\\
\\
        \small $^{\dag}$Peking University
        \\\\
}
\date{}

\begin{document}
\maketitle
\begin{abstract} 
\noindent In machine learning, the scaling law describes how the model performance improves with the model and data size scaling up. From a learning theory perspective, this class of results establishes upper and lower generalization bounds for a specific learning algorithm. Here,  
the exact algorithm running using a specific model parameterization often offers a crucial implicit regularization effect, leading to good generalization.
   To characterize the scaling law, previous theoretical studies mainly focus on linear models, whereas,  feature learning, a notable process that contributes to the remarkable empirical success of neural networks, is regretfully vacant.  This paper studies the scaling law over a linear regression with the model being quadratically parameterized. We  consider infinitely dimensional data and slope ground truth,  both signals exhibiting certain power-law decay rates. We study convergence rates for Stochastic Gradient Descent and  demonstrate the learning  rates for variables will automatically adapt to the ground truth.  As a result, in the canonical linear regression,  we provide explicit separations for generalization curves between  SGD with and without feature learning,  and the information-theoretical lower bound that is agnostic to parametrization method and the algorithm. Our analysis for decaying ground truth provides a new characterization for  
the learning dynamic of the 
model. 

\end{abstract}


\section{Introduction}

The rapid advancement of large-scale models has precipitated a paradigm shift across AI field, with the empirical scaling law emerging as a foundational principle guiding practitioners to scale up the model. The \textit{neural scaling law} \citep{kaplan2020scaling,bahri2024explaining} characterized a polynomial-type decay of excess risk against both the model size and training data volume. Originated from empirical observations, this law predict the substantial improvements of the model performance given abundant training resources. Enough powerful validations have supported the law as critical tools for development of model architecture and allocation of computational resources.


From the statistical learning perspective, neural scaling law formalizes an algorithm-dependent generalization that explicitly quantify how excess risk diminishes with increasing model size and sample size. This paradigm diverges from the classical learning theory, which prioritizes algorithm-agnostic guarantees through a uniform convergence argument for the hypotheses. Empirically, the neural scaling law demonstrates a stable polynomial-type decay of excess risk. This phenomenon persists even as model size approaches infinity, challenging the traditional intuitions about variance explosion. Theoretically, this apparent contradiction implies the role of implicit regularization. Learning algorithms, when coupled with specific parameterized architectures, realize good generalization that suppresses variance explosion. The critical interplay between parameterization methods, optimization dynamics, and generalization, positions algorithmic preferences as an implicit regularization governing scalable learning.

Theoretical progress in characterization of the polynomial-type scaling law has largely centered on linear models, motivated by two synergistic insights. First, the Neural Tangent Kernel (NTK) theory \citep{NEURIPS2018_5a4be1fa,arora2019exact} reveals that wide neural networks, when specially scaled and randomly initialized, can be approximated by linearized models, bridging nonlinear architectures to analytically tractable regimes. Second, linear systems allow for precise characterization of learning dynamics.  The excess risk of linear model is associated with two key factors, the covariance operator spectrum and the regularity of ground truth~\citep{lin2024scaling,bahri2024explaining}.   In the Reproducing Kernel Hilbert Space (RKHS) framework, 
these factors can be described by the capacity of the kernel and source conditions of the target function~\citep{caponnetto2007optimal}.


Compared with traditional studies in linear regression, recent analyses have shifted focus to high-dimensional problems with non-uniform and fine-grained covariance spectra and source conditions~\citep{caponnetto2007optimal,bartlett2021deep}.   The NTK spectrum is shown to exhibit power-law decay when the inputs are uniformly distributed on the unit sphere~\citep{bietti2019inductive,bietti2021deep}.  
 In the offline setting, Gradient Descent (GD) and kernel ridge regression (KRR) exhibit the implicit regularization and multiple descents phenomena, under various geometries of the covariance spectrum and source conditions~\citep{gunasekar2017implicit, Bartlett_2020,10.1214/20-AOS1990,zhang2024optimal1}. In the more widely studied online setting, Stochastic Gradient Descent (SGD) has been proven to achieve a polynomial excess risk under a power-law decay covariance spectrum and ground truth parameter~\citep{dieuleveut2016nonparametric,lin2017optimal,wu2022last}.

\begin{figure}[t]
\centering

\subfigure[\scriptsize{Quadratic v.s. Linear Model}]{
\begin{minipage}[t]{0.3\linewidth}
\centering
\includegraphics[width=2 in]{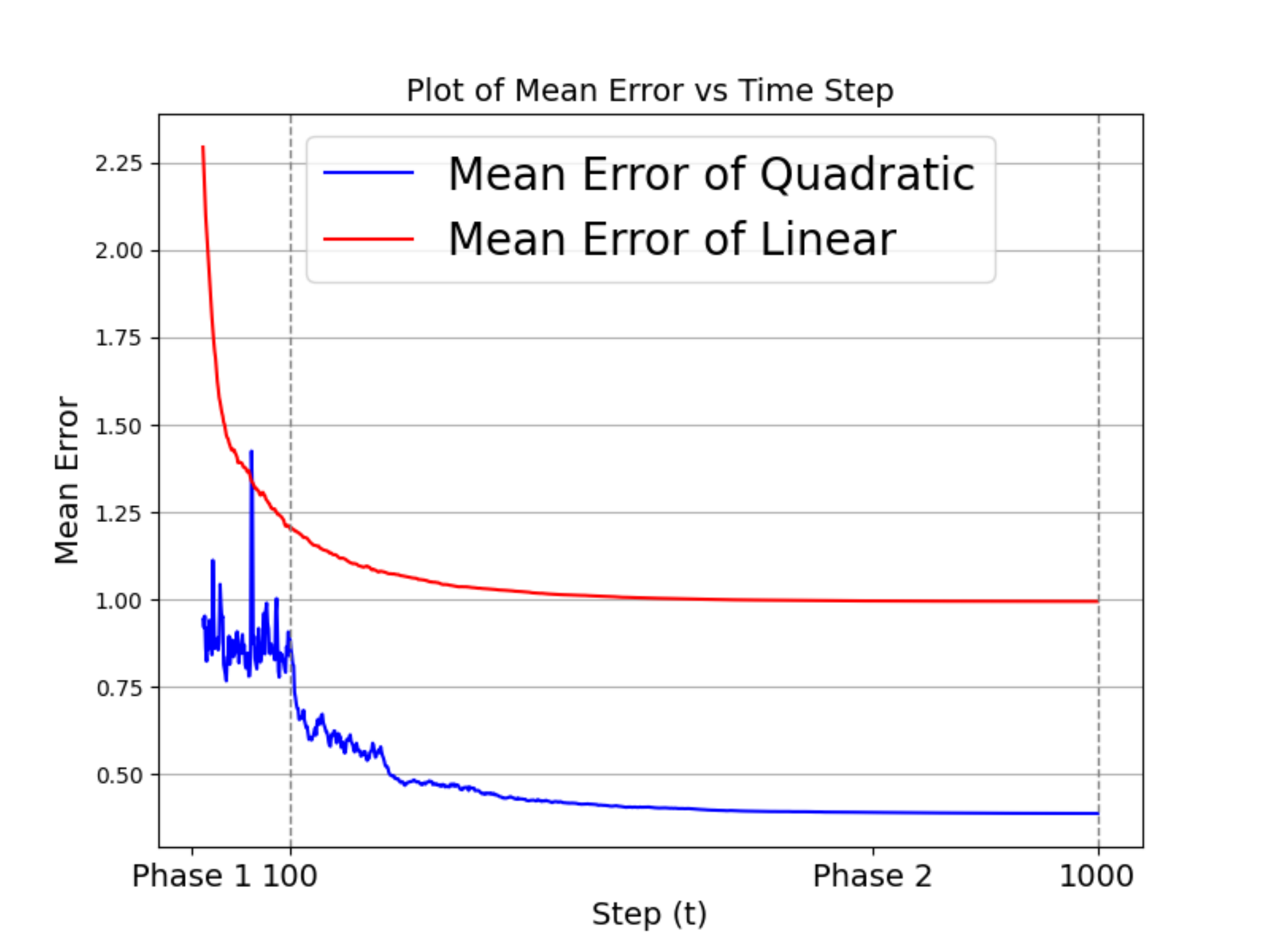}
\end{minipage}%
}%
\subfigure[\scriptsize{Quadratic v.s. Linear Model }]{
\begin{minipage}[t]{0.3\linewidth}
\centering
\includegraphics[width=2 in]{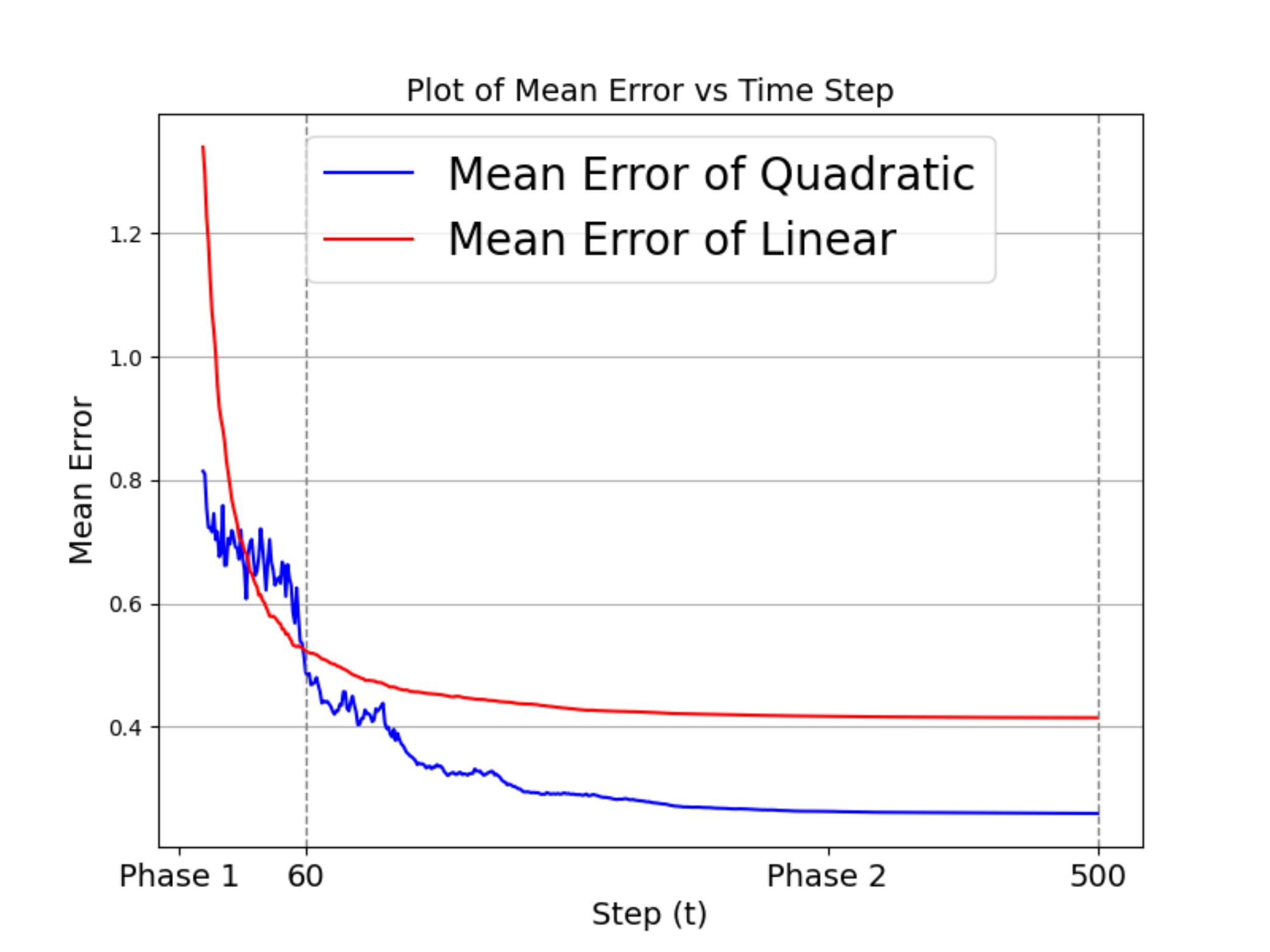}
\end{minipage}%
}%
\subfigure[\scriptsize{Empirical v.s. Theoretical Results}]{
\begin{minipage}[t]{0.3\linewidth}
\centering
\includegraphics[width=2 in]{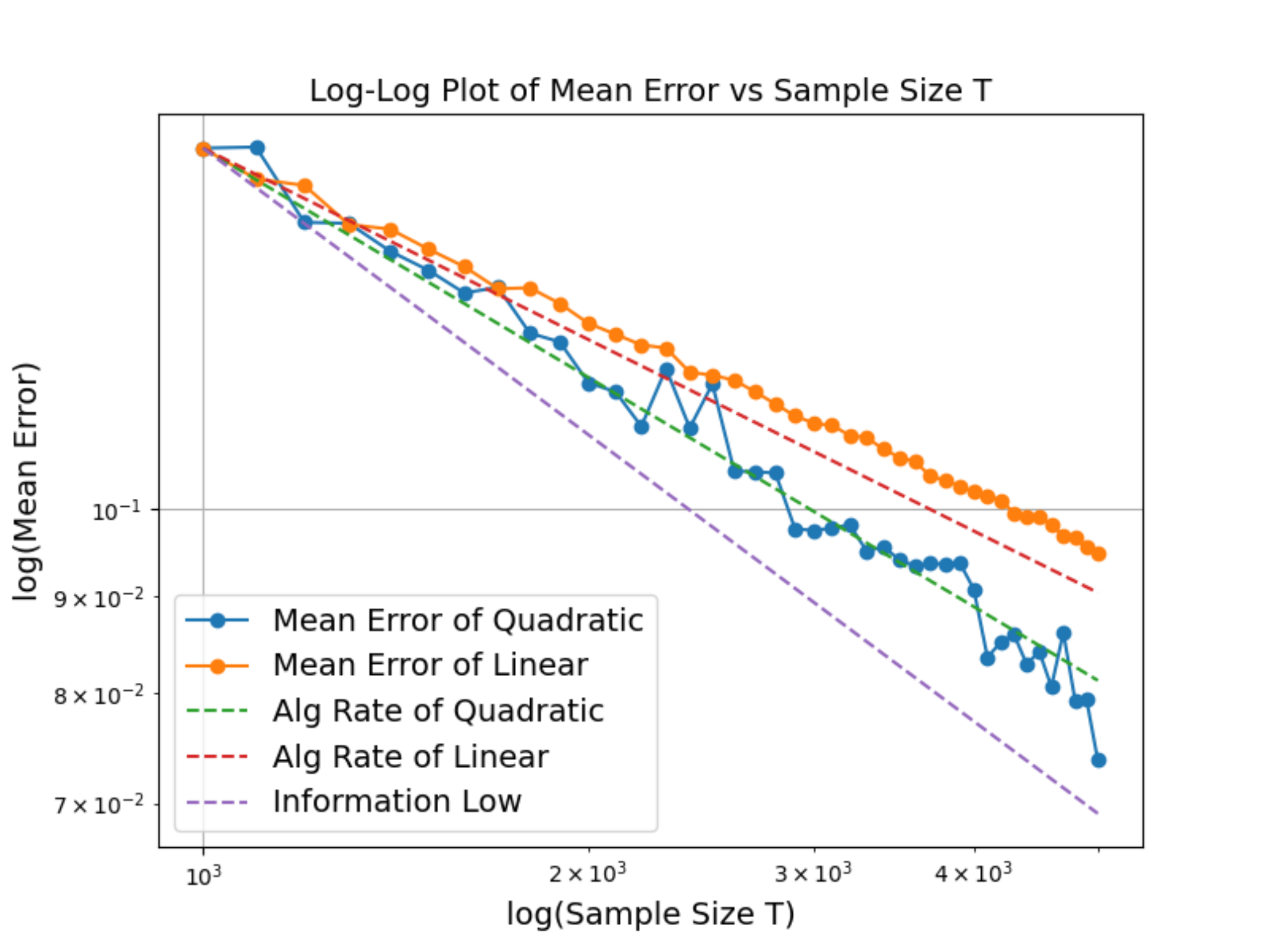}
\end{minipage}%
}%
\centering
\caption{Empirical results on the convergence rate of quadratic model with spectral decay v.s. traditional linear model. (a) and (b) show the curve of mean error against the number of iteration steps, with $\alpha = 2.5, \beta = 1.5$ in (a) and $\alpha =3,\beta=2$ in (b), respectively. (c) show the logarithmic curve of final mean loss against the sample size, where the solid lines represent the empirical results and the dashed lines represent the theoretical rates. }
\label{experiment1}
\end{figure}

However, significant gaps persist in explaining the scaling laws when relying on simplified linear models. A primary limitation of these models is their inability to capture the feature learning process, a mechanism that is widely regarded as crucial to the empirical success of deep neural networks \citep{lecun2015deep}. This process enables neural networks to autonomously extract high-quality hierarchical representations from data, leading to effective generalization. This limitation arises because linear models inherently restrict the capacity to learn feature representations and tend to rapidly diverge from the initial conditions. In linear models, the parameter trajectory under SGD follows a predictable pattern: the estimation bias contracts at a constant rate proportional to the eigenvalue of each feature, while variance accumulates uniformly. However, neural networks are not constrained by an initial feature set; instead, they adaptively reconfigure their internal representations through coordinated parameter updates. The feature learning can often improve the performances. For example,  even the enhanced convolutional neural tangent kernel based on the linearization of neural networks in the infinite-width limit has a performance gap compared to neural networks on the CIFAR10 dataset \citep{li2019enhanced}.

In this paper, we study a quadratically parameterized model:  $f\left (  \mathbf{x} \right ) = \left \langle  \mathbf{x}, \mathbf{v}^{\odot 2} \right \rangle $, where $\mathbf{x}\in \mathbb{H}$ is the input data and 
$\mathbf{v}\in \mathbb{H}$ are the model parameters, 
as an alternative testbed to study the scaling law.  This model can be regarded as a ``diagonal'' linear neural network and exhibits feature learning capabilities. As shown in Figure~\ref{experiment1} (a) and (b), linear models exhibit a empirically suboptimal convergence rate on excess risk under SGD. This suboptimal performance is not solely attributed to the limitations of SGD itself. As demonstrated in Figure~\ref{experiment1} (c), SGD achieves a significantly faster convergence rate on excess risk in quadratically parameterized models, aligning with our theoretical findings. Note that the previous studies for quadratic models  \citep{haochen21shape} often assume a sparse ground
truth for the model where the variance will explode with the number of non-zero elements increasing and no polynomial rates are established. We instead consider an infinitely dimensional data input and ground
truth, whose signal exhibits certain power-law
decay rates.  Specifically,  for constants $\alpha,\beta>1$, we assume that the eigenvalues of the covariance matrix decay as $\lambda_i\asymp i^{-\alpha}$ and that $\mathbf{v}^*_i$ the $i$-th alignment coordinate of the ground truth satisfies $\lambda_i\left(\mathbf{v}^*_i\right)^4\asymp i^{-\beta}$.
Suppose the model has access only to the top $M$-th covariates and their response, we study the excess risk of quadratically parameterized predictor with $M$ parameters and trained by SGD with tail geometric decay schedule of step size, given $T$ training samples. 

We establish the upper bound for the excess risk, demonstrating that its follows a piecewise 
power law with respect to both the model size and the sample size throughout the training process. More concretely, the upper bounds of the excess risk $\mathcal{R}_M(\upv^T)-\bbE[\xi^2]$ behaves as 
\begin{equation}\nonumber
\small
\begin{aligned}
    \underbrace{\frac{1}{M^{\beta-1}}}_{\mathrm{approximation}}+\underbrace{\frac{\sigma^2D}{T}}_{\mathrm{variance}}+\underbrace{\frac{D}{T}+\frac{1}{D^{\beta-1}}\mathds{1}_{D<M}}_{\mathrm{bias}}
\end{aligned}
\end{equation}
where $D=\min\left\{T^{1/\max\{\beta,(\alpha+\beta)/2\}},M\right\}$ serves as the effective dimension. 
The above result
reveals that, for a fixed sample size, increasing the model
size is initially beneficial, but the returns begin to diminish
once a certain threshold is reached. Moreover, when the model size is large enough, SGD achieves the excess risk as $\Tilde{\calO} \left ( T^{-1+\frac{1}{\beta} } \right ) $ when $\alpha\le \beta$, and the excess risk as $\Tilde{\calO} \left ( T^{-\frac{2\beta-2}{\alpha+\beta} } \right ) $.
This indicates that when the true parameter aligns with the covariance spectrum ($\alpha \le \beta$), the quadratic model, similar to the linear model, achieves the optimal rate~\citep{zhang2024optimality}.  On the other hand, when the true parameter opposes the covariance spectrum ($\alpha > \beta$), SGD achieves a rate of $\Tilde{\mathcal{O}} \left( T^{-\frac{2\beta - 2}{\alpha + \beta}} \right)$ in the quadratic model, which outperforms the best rate SGD can achieve in the linear model $\Tilde{\mathcal{O}} \left( T^{-\frac{\beta - 1}{\alpha}} \right)$~\citep{zhang2024optimality}.


In our analysis, we characterize the learning process of SGD into two typical stages. 
In the first ``adaptation'' stage, the algorithm implicitly truncates the first $D$ coordinates 
to form the effective dimension set $\mathcal{S}$, based on the initial conditions. 
The variables within $\mathcal{S}$ grow and oscillate around the ground truth, while the remaining variables are constrained by a constant multiple of the ground truth, leading to an acceptable excess risk. 
In the second ``estimation'' stage, the variables in the effective dimension set $\mathcal{S}$ converge to the ground truth, while the other variables remain within a region that produces a tolerable level of excess risk.  The advantage beyond the linear model is easy to be observed in the “estimation” stage, where the step size is scaled by the certain
magnitude of the ground truth due to the adaption, resulting in a faster convergence rate for the bias term. 

Due to the non-convex nature of the quadratic model, our analysis is much more involved. The main challenge in our analysis is the diverse scaling of the ground truth signals and the anisotropic gradient noise caused by the diverse data eigenvalues. This requires us to provide individual bounds for the model parameters through the analysis and proposes a refined characterization for the learning process.
Both challenges do not exist in the traditional analysis in the quadratic model, since they consider near isotropic input data and $\Theta(1)$ ground truth \citep{haochen21shape}. 

\vspace{0.1in}
We summarize the contribution of this paper as follows: \vspace{-0.05in}
\begin{itemize}
    \item The learning curves of SGD is proposed based on a quadratically parameterized model that emphasizes feature learning. We establish excess risk against sample and model sizes.
\vspace{-0.05in}
    \item A theoretical analysis for the dynamic of the quadratic model is offered, where we propose a new characterization to deal with the decaying ground truth and anisotropic gradient noise.
\end{itemize}

\section{Related Works}

\noindent \textbf{Linear Regression.} 
Linear regression, a cornerstone of statistical learning, achieves information-theoretic optimality $\widetilde{\mathcal{O}}\left(d\sigma^2/T\right)$ in finite dimensions for both offline and online settings \citep{bach2013non,jain2018parallelizing,ge2019step}. Recent advances extend analyses to high-dimensional regimes under eigenvalue regularity conditions and parameter structure \citep{raskutti2014early,gunasekar2017implicit,Bartlett_2020,10.1214/21-AOS2133,tsigler2023benign}. Offline studies characterize implicit bias, benign overfitting, and multi-descent phenomena linked to spectral geometries \citep{liang2020multiple,10.1214/20-AOS1990,mei2022generalization,lu2023optimal,zhang2024optimal1}, while online analyses reveal SGD’s phased complexity release and covariance spectrum-dependent overfitting \citep{dieuleveut2016nonparametric,dieuleveut2017harder,lin2017optimal,ali2020implicit,zou2021benefits,zou2021benign,wu2022last,varre2021last}. Recent work quantifies SGD’s risk scaling under power-law spectral decays \citep{lin2024scaling,bordelon2024a,bahri2024explaining}. 
We follow the geometric decay schedule of the step size \citep{ge2019step,wu2022last,zhang2024optimality}  in Phase  II due to its  superiority in balancing rapid early-phase convergence and stable asymptotic refinement \citep{ge2019step}.  However, in analysis of Phase II, we further require constructing auxiliary sequences to reach the desired convergence rate, which is much more technical.
\vspace{0.2cm}
\noindent \textbf{Feature Learning.}  
The feature learning ability of neural networks is the core mechanism behind their excellent generalization performance. In recent years, theoretical research has primarily focused on two directions: one is the analysis of infinitely wide networks within the mean-field framework, see e.g.~\citet{doi:10.1073/pnas.1806579115,NEURIPS2018_a1afc58c}, and the other is the study of how networks align with low-dimensional objective functions including single-index models~\citep{ba2022high,mousavi2022neural,lee2024neural} and multi-index models \citep{damian2022neural,vural2024pruning}. Although significant progress has been made in these areas, the mean-field mode lacks a clear finite sample convergence rate. Assumptions such as sparse or low-dimensional isotropic objective functions weaken the generality and fail to recover the polynomial decay of generalization error with respect to sample size and model parameters. In this paper, we follow the previous quadratic parameterization~\citep{vaskevicius2019implicit, woodworth2020kernel, haochen21shape}  while develop a generalization error analysis under an anisotropic covariance structure, yielding generalization error results similar to those predicted by the neural scaling law.

\section{Set up}
\subsection{Notation} 




In this section, we introduce the following notations adopted throughout this work. Let $\mathcal{O}(\cdot)$ and $\Omega(\cdot)$ denote upper and lower bounds, respectively, with a universal constants, while $\widetilde{\mathcal{O}}(\cdot)$ and $\widetilde{\Omega}(\cdot)$ ignore polylogarithmic dependencies. For functions $f$ and $g$: $f \lesssim g$ denotes $f = \widetilde{\mathcal{O}}(g)$; $f \gtrsim g$ denotes $f = \widetilde{\Omega}(g)$; $f \asymp g$ indicates $g \lesssim f \lesssim g$. Let $\mathbb{H}$ be a separable Hilbert space with finite or countably infinite dimensions, equipped with an inner product $\langle \cdot, \cdot \rangle$. For any $\mathbf{u}, \mathbf{v}, \mathbf{w} \in \mathbb{H}$, we define the operator $\mathbf{u} \otimes \mathbf{v}: \mathbb{H} \to \mathbb{H}$ as:
\begin{equation}
    (\mathbf{u} \otimes \mathbf{v})(\mathbf{w}) = \langle \mathbf{v}, \mathbf{w} \rangle \mathbf{u}.
\end{equation}
Let $\{\mathbf{e}_i\}_{i=1}^M$ denote the canonical basis in $\mathbb{R}^M$. For an arbitrary orthonormal basis $\{\mathbf{u}_i\}_{i=1}^\infty$ of $\mathbb{H}$, we define the nonlinear operator $\odot^2: \mathbb{H} \to \mathbb{H}$ acting on $\mathbf{v} \in \mathbb{H}$ as:
\begin{equation}
    \mathbf{v}^{\odot 2} := \sum_{i=1}^\infty \langle \mathbf{v}, \mathbf{u}_i \rangle^2 \mathbf{u}_i.
\end{equation}
Specially, in finite-dimensional Euclidean space ($\mathbf{v} \in \mathbb{R}^d$), this reduces to element-wise squaring as: $(\mathbf{v}^{\odot 2})_j = \upv_j^2$ for $j \in [d]$. Given a positive integer $N \leq M$ and a vector $\upv \in \mathbb{R}^M$, denote $\upv_{1:N}$ as $(\upv_1, \ldots, \upv_N)^{\top}$. Moreover, for a diagonal matrix $\diag\{\gamma_1,\cdots,\gamma_M\}=\upH\in\bbR^{M\times M}$, denote $\upH_{n^*:n^{\dagger}}:=\sum_{i=n^*}^{n^{\dagger}}\gamma_i\upe_i\upe_i^{\top}$, where $0\leq n^*\leq n^{\dagger}\leq M$ are two integers. 

\subsection{Quadratically Parameterized Model}
Denote the covariate vector by $\mathbf{x}\in \bbH$, and the corresponding response by $y\in \mathbb{R}$.
We focus on a quadratically parameterized model and measure the population risk of parameter $\mathbf{v}$ by the mean squared loss as:
\begin{equation}\nonumber
    \mathcal{R} \left ( \mathbf{v} \right )=\mathbb{E}_{\left ( \mathbf{x},y  \right ) \sim \mathcal{D} } \left ( \left \langle \mathbf{x},\mathbf{v}^{\odot 2} \right \rangle  -y  \right )^2 ,
\end{equation}
where the expectation is taken over the joint distribution $\mathcal{D}$ of $\left ( \mathbf{x} ,y \right ) $.
In this paper, we study the quadratically parameterized model as $\left \langle \mathbf{x},\mathbf{v}^{\odot 2} \right \rangle$, and assume that the response is generated by $y=\langle \mathbf{x},\mathbf{v}^{*\odot 2}\rangle +\xi $, where $\mathbf{v}^*\in \bbH$ is the ground truth parameter and $\xi$ is the noise independent of the covariate $\mathbf{x}$.
One can generally use the parameterization as $y=\langle\upx,\upv_+^{*\odot2}-\upv_-^{*\odot2}\rangle+\xi$ by the same technique as \citet{woodworth2020kernel}.


We then present a quadratically parameterized predictor with $M$ trainable parameters, which serves as the foundation for our algorithm design and subsequent theoretical analysis. Let $\mathbf{H}=\mathbb{E}[\mathbf{x}\otimes \mathbf{x}] $ denote the data covariance operator. The eigenstructure of $\mathbf{H}$ is given by  $\mathbf{H}=\sum_{i=1}^{\infty}\lambda_i\upu_i\otimes \upu_i$, where $\left \{ \lambda_i \right \} _{i=1}^{\infty }$ are the eigenvalues arranged in non-increasing order, and $\left \{ \mathbf{u}_i \right \} _{i=1}^{\infty }$ are the corresponding eigenvectors of $\mathbf{H}$, which form an orthonormal basis for $\mathbb{H}$. For any vector $\mathbf{y}\in \bbH$, denote $\mathbf{y}_i=\langle\mathbf{y},\mathbf{u}_i\rangle$ as its $i$-th coordinate with respect to the above eigenbasis.
Let $\Pi_M =\sum_{i=1}^M\upe_i\otimes\upu_i$ be a projection operator from $\bbH$ to $\bbR^M$.
We assume that we access only to the top $M$ covariates, denoted by $\Pi_M \mathbf{x} $. Subsequently, we define the predictor with $M$ trainable parameters $\mathbf{v}\in \bbR^M$ as: 
\begin{equation}\nonumber
    f_{\mathbf{v} }\left (  \mathbf{x} \right ) = \left \langle \Pi_M \mathbf{x}, \mathbf{v}^{\odot 2}   \right \rangle ,\ \mathbf{v}\in \bbR^M.
\end{equation}
We also denote the corresponding risk as 
\begin{equation}\label{RM}
    \mathcal{R}_{M}\left ( \mathbf{v}  \right )=\mathcal{R}\left ( \Pi _M\mathbf{v}  \right )=\mathbb{E}_{\left ( \mathbf{x},y  \right ) \sim \mathcal{D} }\left ( \left \langle \Pi_M \mathbf{x}, \mathbf{v}^{\odot 2} \right \rangle -y \right ) ^2. 
\end{equation}

The top $M$ covariates resemble the sketched covariates proposed by \citet{lin2024scaling}. In contrast with linear model, quadratic parameterized model allows discovery of discriminative features through learning towards dominant directions of target. Thus, it models the feature learning mechanism  while ensuring analytical tractability.




\subsection{Data Distribution Assumptions}
We make the following assumptions of data distribution.
\begin{assumption}[Anisotropic Gaussian Data and Sub-Gaussian Noise]\label{ass-d}
    \item[\textbf{[A$_\text{1}$]}] (Independent Gaussian Data) For any $i\ge 1$, the covariate $\upx_i\sim\calN(0,\lambda_i)$. For any $i\neq j$, $\mathbf{x}_i$ and $\mathbf{x}_j$ are independent.
    \item[\textbf{[A$_\text{2}$]}] (Sub-Gaussian Noise) The noise $\xi$ is zero-mean and sub-Gaussian with parameter $\sigma_\xi > 0$, satisfying $\mathbb{E}[e^{\lambda \xi}] \leq e^{\sigma_\xi^2 \lambda^2/2}$ for all $\lambda \in \mathbb{R}$.
\end{assumption}
\begin{remark}\label{remark-1}
The assumption for independent Gaussian data is also used in other analyses for the quadratic model, such as \citet{haochen21shape}, whereas, we allow non-identical covariates.  The independence assumption resembles (is slightly stronger than) the RIP condition, and is widely adopted in feature selection, e.g. \citet{candes2005decoding},  to ensure computational tractability, because in the worst case, finding sparse features is NP-hard \citep{natarajan1995sparse}. To weaken the independence assumption, our analysis may be extended to only assume that the input has a low correlation with some weak regularity condition.
However, the analysis is more involved, and we leave it as a future work. 

\end{remark}
We derive the scaling law for SGD under the following power-law decay assumptions of the covariance spectrum and prior conditions.
\begin{assumption}[Specific Spectral Assumptions]\label{ass-ss}
    \item[\textbf{[A$_\text{3}$]}] (Polynomial Decay Eigenvalues) There exists $\alpha>1$ such that for any $i\ge 1$, the eigenvalue of data covariance $\lambda_i$ satisfy $\lambda_i\asymp i^{-\alpha}$.
    \item[\textbf{[A$_\text{4}$]}] (Source Condition) There exists $\beta>1$ such that the ground truth parameter $\upv^*$ satisfies that for any $i\ge 1$, 
 $\lambda _i\left ( \mathbf{v}_i^{*}  \right ) ^4 \asymp i^{-\beta}$.
\end{assumption}


\begin{remark}
     The polynomial decay of eigenvalues and the ground truth has been widely considered to study the scaling laws for linear models like random feature model \citep{bahri2024explaining, bordelon2024a} and infinite dimensional linear regression \citep{lin2024scaling}, based on empirical observations of NTK spectral decompositions on the realistic dataset \citep{bahri2024explaining,bordelon2021learning}. It is used in slope functional regression \citep{10.1214/009053606000000830}, and also analogous to the capacity and source conditions in RKHS \citep{wainwright2019high,bietti2019inductive}. Given that the optimization trajectory of linear models is intrinsically aligned with the principal directions of the covariate feature space, this alignment motivates us to adopt analogous assumptions for our model, thereby enabling direct comparison of learning dynamics through feature space decomposition.
\end{remark}

\subsection{Algorithm}
We employ SGD with a geometric decay of step size to train the quadratically parameterized predictor $f_{\mathbf{v} }$ to minimize the objective \eqref{RM}. Starting at $\mathbf{v}^0$, the iteration of parameter vector $\mathbf{v}\in \mathbb{R}^M$ can be represented explicitly as follows:
\begin{equation}\nonumber
\small
\begin{aligned}
    \mathbf{v}^t= &\mathbf{v}^{t-1}-\eta _t\left ( f_{\mathbf{v}^{t-1}} \left ( \mathbf{x}^t \right )-y^t \right )\left (\mathbf{v}^{t-1}\odot \Pi_M\mathbf{x}^{t}    \right ) \\
=  &\mathbf{v}^{t-1}-\eta _t\left (\left \langle \left ( \mathbf{v}^{t-1} \right )^{\odot 2},\Pi_M\mathbf{x}^t     \right \rangle -y^t \right )\left(\mathbf{v}^{t-1}\odot \Pi_M\mathbf{x}^{t}\right ),
\end{aligned}
\end{equation}
for $t=1,\dots,T$, where $\left\{\left ( \mathbf{x}^t,y^t\right)\right \}_{t=1}^{T}$ are independent samples from distribution $\mathcal{D}$ and $\left\{ \eta_t\right\}_{t=1}^{T}$ are the step sizes.

We use the tail geometric decay of step size schedule as describe in \citet{wu2022last}. The step size remains constant for the first $T_1+h$ iterations where $h$ denotes the middle phase length and $T_1:=\lfloor(T-h)/\log(T-h)\rfloor$. Then the step size halves every $T_1$ steps. Specifically, the decay schedule of step size is given by:
\begin{equation}\nonumber
\begin{aligned}
\eta_t=\begin{cases}
    \eta,\quad &0\le t\le T_1+h,
    \\
    \eta/2^l,\quad &T_1+h<t\le T,\, \, l=\left \lfloor (t-h)/T_1 \right\rfloor,
\end{cases}
\end{aligned}
\end{equation}

The integration of warm-up with subsequent learning rate decay has become a prevalent technique in deep learning optimization \citep{goyal2017accurate}. Within the decay stage, geometric decay schedules have demonstrated superior empirical efficiency compared to polynomial alternatives, as geometric decay achieves adaptively balancing aggressive early-stage learning with stable late-stage refinement.  \citep{ge2019step}. Motivated by these established advantages, our step size schedule design strategically combines an initial constant stage that enables rapid convergence of dominant feature coordinates to the vicinity of ground truth parameters with a subsequent geometrically decaying stage that ensures fast convergence through progressively refined updates. This hybrid approach inherits the computational benefits of geometric decay while maintaining the stability benefits of warm-up initialization, creating synergistic effects that polynomial decay schedules cannot achieve \citep{bubeck2015convex}.



\begin{algorithm}[t]
	\caption{Stochastic Gradient Descent (SGD)}
    \label{SGD}
	\begin{algorithmic}
        \STATE {\bfseries Input:} Initial weight $\upv_0=\Omega(\min\{1,M^{-(\beta-\alpha)/4}\})\mathbf{1}_M$, initial step-size $\eta$, total sample size $T$, middle phase length $h$, decaying phase length $T_1=\left\lfloor (T-h)/\log(T-h)\right\rfloor$.
		\WHILE{$t\leq T$}
		\IF{$t>h$ and $(t-h)$ mod $T_1=0$}
		\STATE $\eta\leftarrow\eta/2$.
		\ENDIF
		\STATE Sample a fresh data   $(\upx^{t+1},y^{t+1})\sim\mathcal{D}$.
		\STATE $\upv^{t+1}\leftarrow\upv^t-\eta\nabla l_M^t(\upv^t)$.
		\ENDWHILE
	\end{algorithmic}
\end{algorithm}

The algorithm is summarized as Algorithm~\ref{SGD}. 
For simplicity,  at iteration $t$, denote $l^t_M(\upv)=\frac{1}{2}\left(f_{\upv}(\upx^{t+1})-y^{t+1}\right)^2$. 
The initial point $\mathbf{v}_0$ and the initial step size $\eta$ are hyperparameters of Algorithm~\ref{SGD}, and they play a crucial role in determining whether the algorithm can escape saddle points and converge to the optimal solution. 
Starting at an initial point near zero, the constant step size stage allows the algorithm to adaptively extract the important features without explicitly setting the truncation dimensions while keeping the remaining variables close to zero. The subsequent geometric decay of the step size guarantees fast convergence to the ground truth.  

\section{Convergence Analysis}
The upper bound of last iterate instantaneous risk for Algorithm \ref{SGD} can be summarized by the following theorem, which provides the guarantee of global convergence for last iterate SGD with tail geometrically decaying stepsize and a sufficiently small initialization. For simplicity, we define 
$\sigma:=(\sigma_{\xi}^2+\sum_{i=M+1}^{\infty}\lambda_i(\upv_i^*)^4)^{1/2}$.
\begin{theorem}\label{theorem-3}
    Under Assumptions \ref{ass-d} and \ref{ass-ss}, we consider a predictor trained by Algorithm \ref{SGD} with total sample size $T$ and middle phase length $h=\lceil T/\log(T)\rceil$. 
    Let $D\asymp\min\{T^{1/\max\{\beta,(\alpha+\beta)/2\}},M\}$ 
    and $\eta\asymp D^{\min\{0,(\alpha-\beta)/4\}}$. The error of output can be bounded from above by
    \begin{equation}\label{risk-main}
        \begin{split}
        \mathcal{R}_M(\upv^T)-\bbE[\xi^2]\lesssim&\underbrace{\frac{1}{M^{\beta-1}}}_{\mathrm{approximation}}+\underbrace{\frac{\sigma^2D}{T}}_{\mathrm{variance}}+\underbrace{\frac{D}{T}+\frac{1}{D^{\beta-1}}\mathds{1}_{D<M}}_{\mathrm{bias}},
        \end{split}
    \end{equation}
    with probability at least 0.95.
\end{theorem}
Our bound exhibits two key properties: (1) Dimension-free: Eq.~\eqref{risk-main} depends on the effective dimension $D$ rather than ambient dimension $M$. (2) Problem-adaptive: $D$ is governed by the spectral structure of $\diag\{\lambda_1(\upv_1^*)^2,\cdots,\lambda_M(\upv_M^*)^2\}$, which is induced by the multiplicative coupling between the data covariance matrix and optimal solution determined by the problem. 

The risk bound in Eq.~\eqref{risk-main} consists of three components: (1) approximation error term, (2) bias error term originating from $\upv^{T_1}-\upv_{1:M}^*$ at iteration $T_1=\lceil(T-h)/\log(T-h)\rceil$, and (3) variance error term stemming from the multiplicative coupling between additive noise $\xi+\sum_{i\geq M+1}\upx_i(\upv_i^*)^2$ and the diagonal matrix $\diag\{\upv_{1:M}^*\}\in\bbR^{M\times M}$. 

For sufficiently large $M$, Corollary \ref{corollary-1} establishes the convergence rate for Algorithm \ref{SGD} via Theorem \ref{theorem-3}.
\begin{corollary}\label{corollary-1}
    Under the setting of the parameters in Theorem \ref{theorem-3}, if $T^{1/\max\{\beta,(\alpha+\beta)/2\}}\asymp D<M$, we have
    \begin{align}
        \begin{cases}
            \mathcal{R}_M(\upv^T)-\bbE[\xi^2]\lesssim \frac{1}{M^{\beta-1}}+\frac{\sigma^2+1}{T^{1-1/\beta}}, & \text{ if }\beta\ge \alpha>1,
            \\
            \mathcal{R}_M(\upv^T)-\bbE[\xi^2]\lesssim \frac{1}{M^{\beta-1}}+\frac{\sigma^2+1}{T^{(2\beta-2)/(\alpha+\beta)}}, & \text{ if } \alpha>\beta >1,\notag
        \end{cases}
    \end{align}
    with probability at least 0.95.
\end{corollary}
Corollary \ref{corollary-1} demonstrates that under Assumptions \ref{ass-d} and \ref{ass-ss}, when the model size $M$ is sufficiently large, the last iterate instantaneous risk of Algorithm \ref{SGD} exhibits distinct behaviors in two regimes: (I) $\beta\geq\alpha>1$ and (II) $\alpha\geq\beta>1$. 
We consider the total computational budget
 as $B=MT$, reflecting that Algorithm \ref{SGD} queries $M$-dimensional gradients $T$ times. 

\noindent \textbf{Given $B$: } If $\beta>\alpha>1$, the optimal last iterate risk is attained with parameter configurations: $M=\widetilde{\Omega}(B^{\frac{1}{1+\beta}})$ and $T=\widetilde{\Omega}(B^{\frac{\beta}{1+\beta}})$. If $\alpha\geq\beta>1$, the optimal last iterate risk is attained with parameter configurations: $M=\widetilde{\Omega}(B^{\frac{1}{1+(\alpha+\beta)/2}})$ and $T=\widetilde{\Omega}(B^{\frac{(\alpha+\beta)/2}{1+(\alpha+\beta)/2}})$.

\noindent \textbf{Given Total Sample Size $T$: } So as long as  $M\gtrsim T^{1/\max\{\beta,(\alpha+\beta)/2\}}$, Corollary \ref{corollary-1} implicates that the risk can be effectively reduced by increasing the model size $M$ as much as possible.



For smaller $M$, Corollary \ref{corollary-2} provides the convergence rate for Algorithm \ref{SGD} through Theorem \ref{theorem-3}. 
\begin{corollary}\label{corollary-2}
    Under the setting of the parameters in Theorem \ref{theorem-3}, if $M\lesssim T^{1/\max\{\beta,(\alpha+\beta)/2\}}$, we have
    \begin{align}
        \mathcal{R}_M(\upv^T)-\bbE[\xi^2]\lesssim\frac{1}{M^{\beta-1}}+\frac{(\sigma^2+1)M}{T},\notag
    \end{align}
    with probability at least 0.95.
\end{corollary}
The risk bound $\calR_M(\cdot)$ in Corollary \ref{corollary-2} decreases monotonically with increasing $M$.  So  as long as  $M\lesssim T^{1/\max\{\beta,(\alpha+\beta)/2\}}$, our analysis implies to increase the model size $M$ until reaching the  computational budget.


\begin{remark}
The information-theoretic lower bound acts as a bottleneck for the best estimation issues. For any (random) algorithm $\hat{\mathbf{v}}$ based on i.i.d. data $\left \{ \left ( \mathbf{x}_i,y_i  \right )  \right \} _{i=1}^T$ from the true parameter $\mathbf{v}_* \in \mathcal{V}$, the worst-case excess risk convergence rate is limited by this bound. 
The scaling law, however, describes the excess risk trajectory of a specific algorithm in a given context during training. Under the covariate distribution in Assumptions~\ref{ass-d} and~\ref{ass-ss}, and with the true parameter set meeting Assumption~\ref{ass-ss} regularity constraints, prior work~\citep{zhang2024optimality} established the info-theoretic lower bound as $T^{-\frac{1}{\beta}}$. Our analysis shows two distinct regimes: When $\alpha\le \beta$, SGD in linear and quadratic models hits the lower bound, proving statistical optimality. When $\alpha> \beta$, SGD in both misses the bound, yet the quadratic model has better excess risk than the linear one. This shows a capacity gap between the two model types, highlighting the importance of feature learning and model adaptation.

\end{remark}

\begin{remark}
    To demonstrate the near-optimality of our upper bound, we also provide the algorithmic {\bf \emph{lower bound}} of the excess risk for Algorithm \ref{SGD} in Appendix \ref{sec-lower}. The lower bound of the excess risk for Algorithm \ref{SGD} {\bf \emph{matches}} the upper bound in Theorem \ref{theorem-3},  up to logarithmic factors. 
    The construction of the algorithmic lower bound is based on the construction of appropriate submartingales (or supermartingales). 
    Our lower bound analysis reveals that for coordinates $j\geq\widetilde{\mathcal{O}}(D)$, the slow ascent rate inherently prevents $\upv_j^t$ from attaining close proximity to the optimal solution $\upv_j^*$ upon algorithmic termination. This phenomenon induces bias error's scaling as $\widetilde{\Omega}(D^{-\beta+1})$, matching our upper bound characterization, up to logarithmic factors. The bound also implies that the convergence rates for SGD is inherently worse than  information-theoretic lower bounds when $\alpha >\beta$. 
\end{remark}

\section{Proof Sketch of Theorem \ref{theorem-3}} \label{sec:proof sketch}
In this section, we introduce the proof techniques sketch of our main result Theorem \ref{theorem-3}. Specifically, the dynamics and analysis of SGD can be divided into two phases. In \textbf{Phase I (Adaptation)}, SGD autonomously truncates the top $D$ coordinates as $\mathcal{S}$ without requiring explicit selection of $D$. Algorithm \ref{SGD} can converge these coordinates to a neighborhood of their optimal solutions within $T_1$ iterations with high probability, and realize the trade-off between two terms $\frac{D}{T}$ and $\frac{1}{D^{\beta -1}}$ in the risk bound through precise choice of step size. 
In \textbf{Phase II (Estimation)}, global convergence to the risk minimizer is achieved over $T_2$ iterations, which can be approximated as SGD with geometrically decaying step sizes applied to a linear regression problem in the reparameterized feature space $\Pi_M\upx\odot\upv_{1:M}^*$. 


\subsection{Phase I: }
During the ``adaptation" phase, Algorithm \ref{SGD} implicitly identifies the first $D$ coordinates as the effective dimension set $\mathcal{S} = [1:D]$. For each $i \in \mathcal{S}$, $\upv_i^{T_1}$ converges with high probability to a rectangular neighborhood centered at $\upv_i^*$ with half-width $c_1\upv_i^*$. Here, $c_1 \in (0,1)$ denotes a scaling constant. For each $i\in[1:M]\setminus\mathcal{S}$, $\upv_i^{T_1}$ remains bounded above by $\frac{3}{2}\upv_i^*$ with high probability.
\begin{theorem}\label{phase-I-informal}
    Under Assumption \ref{ass-d}, consider a predictor trained via Algorithm \ref{SGD} with initialization $\upv^0$. Let the step size $\eta\leq\eta(D, c_1)$, for the effective dimension $D$ and the scaling constant $c_1 \in (0,1)$. 
    The iteration number $T_1$ requires:
    \begin{align}
        T_1\in\begin{cases}
        \left[T_l(D, c_1), T_u(D, c_1)\right], \quad &\text{ if }D < M,\notag
        \\
        \left[T_l(M, c_1),\infty\right), \quad &\text{ otherwise}.
    \end{cases}
    \end{align}
    Then, with high probability, we have
    \begin{align}\label{proof-sketch-1}
        \begin{cases}
            \upv_i^{T_1}\in\left[(1-c_1)\upv_i^{*}, (1+c_1)\upv_i^{*}\right], &\text{ if }i\in\mathcal{S},
            \\
            \upv_i^{T_1}\in\left[0, \frac{3}{2}\upv_i^*\right], &\text{ otherwise }.
        \end{cases}
    \end{align}
\end{theorem}
To characterize the mainstream dynamic, our analysis employs a probabilistic sequence synchronization technique. That is, from the Algorithm \ref{SGD}-generated sequence $\{\upv^t\}_{t=0}^{T_1}$, we construct a control sequence $\{\upq^t\}_{t=0}^{T_1}$. We first establish Lemmas \ref{phase-1-step-1}--\ref{phase-1-step-3}'s conclusions for the control sequence. Lemma \ref{phase-1-step-1} derives a high-probability upper bound for $\upq^{T_1}$, matching the bound in Theorem \ref{phase-I-informal}. 

In the analysis of Phase I, 
 we need to construct another capped coupling sequence $\{\bar{\upv}^t\}_{t=0}^{T_1}$ to describe the dynamic of iterates, and more importantly, we need to delve into the dynamic processes of the two-part parameters separated by the effective dimension $D$. It is non-trivial because in the traditional analysis of prior work to recover the sparse ground truth \citep{haochen21shape}, it is unnecessary to  introduce $D$. 
In our analysis, we partition the interval $[1:M]$ into $\mathcal{S}$ and $\mathcal{S}^c$. We will separately and precisely estimate the sub-Gaussian parameters for the supermartingales generated from $\{\bar{\upv}^t\}_{t=0}^{T_1}$. The characterization of dynamic is separated into the following three lemmas. In Lemma \ref{phase-1-step-1}, 
we show the last iterate $\upq^{T_1}$ in control sequence $\{\upq^t\}_{t=0}^{T_1}$ satisfies $\upq_i^{T_1}\leq (1+c_1)\upv_i^*$ for any $i\in\mathcal{S}$ and $\upq_i^{T_1}\leq\frac{3}{2}\upv_i^*$ for any $i\in\mathcal{S}^c$ with high probability.

\begin{lemma}\label{phase-1-step-1}
    Under the setting of Theorem \ref{phase-I-informal}, both $\upq_i^{T_1}\leq (1+c_1)\upv_i^*$ for any $i\in\mathcal{S}$ and $\upq_i^{T_1}\leq\frac{3}{2}\upv_i^*$ for any $i\in\mathcal{S}^c$ occur with high probability.
\end{lemma}
Lemmas \ref{phase-1-step-2} and \ref{phase-1-step-3} collectively address the lower bound of $\upq^{T_1}$ in Theorem \ref{phase-I-informal}. To establish Lemma \ref{phase-1-step-2}, we derive a subcoupling sequence $\{\Breve{\upv}^{i,t}\}_{t=0}^{T_1}$ from the parent coupling sequence $\{\Bar{\upv}^t\}_{t=0}^{T_1}$ for any $i\in\mathcal{S}=[1:D]$. Each subcoupling sequence undergoes logarithmic transformation to generate a linearly compensated submartingale $\{-t\log(1+\eta\calO(\lambda_i(\upv_i^*)^2))+\log(\breve{\upv}_i^{i,t})\}_{t=1}^{T_1}$. These $D$ submartingales exhibit monotonic growth with sub-Gaussian increments. Applying concentration inequalities, we obtain $\max_{t\in[1:T_1]}\upq_i^t\geq(1-c_1/2)\upv_i^*$ with high probability for any $i\in\mathcal{S}$.
\begin{lemma}\label{phase-1-step-2}
    Under the setting of Theorem \ref{phase-I-informal}, with high probability, either $\max_{t\leq T_1}\upq_i^t\geq (1-c_1/2)\upv_i^*$ for any $i\in\mathcal{S}$, or at least one of the following statements fails: $\upq_i^{T_1}\leq (1+c_1)\upv_i^*$ for any $i\in\mathcal{S}$ and $\upq_i^{T_1}\leq\frac{3}{2}\upv_i^*$ for any $i\in\mathcal{S}^c$.
\end{lemma}
Lemma \ref{phase-1-step-3} further establishes that when $\max_{t\leq T_1}\upq_i^t$ resides within the $\upv_i^*$-neighborhood, the lower bound satisfies $\upq_i^{T_1} \geq (1 - c_1)\upv_i^*$ with high probability. The proof of Lemma \ref{phase-1-step-3} mirrors that of Lemma \ref{phase-1-step-1}.
\begin{lemma}\label{phase-1-step-3}
Under the setting of Theorem \ref{phase-I-informal}, for any $i\in\mathcal{S}$, with high probability, either $\max_{t\leq T_1}\upq_i^t< (1-c_1/2)\upv_i^*$ or $\upq_i^{T_1}\geq(1-c_1)\upv_i^*$.  
\end{lemma}
According to the high-probability equivalence between $\{\upq^t\}_{t=0}^{T_1}$ and $\{\upv^t\}_{t=0}^{T_1}$, Lemmas \ref{phase-1-step-1}--\ref{phase-1-step-3}'s conclusions transfer to $\upv^{T_1}$ with high-probability guarantees. 

\subsection{Phase II:}
We now start the analysis  of  Phase II for Algorithm \ref{SGD}.  The main idea stems from approximating Algorithm \ref{SGD}'s iterations as SGD 
running over a linear model  with rescaled features $\Pi_M\upx\odot\upv_{1:M}^*$. The adaptive rescale size $\upv_{1:M}^*$
enables the quadratic model to achieve accelerated convergence rates compared to its linear counterpart. To streamline notation, we introduce two key matrices in $\mathbb{R}^{M\times M}$: 
i) $\upH := \diag\{\lambda_i(\upv_i^*)^2\}_{i=1}^M$; ii) $\widehat{\upH} := \diag\{\lambda_1(\upv_1^*)^2, \cdots, \lambda_N(\upv_N^*)^2, \textbf{0}_{M-D}\}$.
\begin{theorem}\label{phase-II-informal}
    Suppose Assumptions \ref{ass-d} and  \ref{ass-ss} hold. By selecting an appropriate step size $\eta_0=\eta(D)$ and middle phase length $h$, we obtain
    \begin{equation}
        \begin{split}
            \mathcal{R}_{M}(\upv^T)\lesssim&\mathcal{R}_{M}(\upv_{1:M}^*)+\frac{\sigma^2N}{T}+\sigma^2\eta_0^2 T\tr\left(\upH_{D+1:M}^2\right)+\frac{D}{T}+\eta_0^2T\tr\left(\upH_{D+1:M}^2\right)
        \\
        &+\llangle\frac{1}{\eta_0 T}\upI_{1:D}+\upH_{D+1:M},\left(\upI-\eta_0\widehat{\upH}\right)^{\frac{2T}{\log(T)}}\upB^0\rrangle,\notag
        \end{split}
    \end{equation}
    with probability at least 0.95.
\end{theorem}
The proof of Theorem \ref{phase-II-informal} is structured in two key parts. In the first part, Theorem \ref{phase-2-step-1} establishes that Algorithm \ref{SGD} iterates $\{\upv^t\}_{t=T_1+1}^{T=T_1+T_2}$ remain confined within the neighborhood $\prod_{i=1}^D[\frac{1}{2}\upv_i^*,\frac{3}{2}\upv_i^*]\times\prod_{i=D+1}^M[0,2\upv_i^*]$ with high probability. 
\begin{theorem}\label{phase-2-step-1}
    Under Assumption \ref{ass-d}, we consider the iterative process of Algorithm \ref{SGD}, beginning from step $T_1$ with the same step size $\eta$ as in Theorem \ref{phase-I-informal}. If $D<M$, let $1\leq T_2\leq T_u(D)$ where $T_u(D)\in\bbN_+$ depends on $D$; otherwise, let $T_2\geq 1$. Then, for any $t\in[1:T_2]$, with high probability, we have
    \begin{align}\label{proof-sketch-2}
        \begin{cases}
            \upv_i^{T_1+t}\in\left[\frac{1}{2}\upv_i^*,\frac{3}{2}\upv_i^*\right],\quad &\text{ if }i\in[1:D],
            \\
            \upv_i^{T_1+t}\in\left[0, 2\upv_i^*\right],\quad &\text{ otherwise }.
        \end{cases}
    \end{align}
\end{theorem}
According to Theorem \ref{phase-I-informal}, $\upv^{T_1}$ satisfies Eq.~\eqref{proof-sketch-1} with high probability by setting $c_1 = \frac{1}{4}$. By employing the construction method for supermartingales, similar to that used in the proofs of Lemma \ref{phase-1-step-1} and Lemma \ref{phase-1-step-3}, we obtain a set of compressed supermartingales depend on the coordinate $i\in[1:M]$. Applying supermartingales concentration inequality, we obtain Eq.~\eqref{proof-sketch-2}.

In the second part, we construct an auxiliary bounded sequence $\{\upw^t\}_{t=1}^{T_2}$ which is the truncation of $\{\upv^{T_1+t}\}_{t=1}^{T_2}$, and provide the last iterate instantaneous risk for $\upw^{T_2}$ which can be extended to $\upv^{T}$. The construction ensures that if $\upv^{T_1+t}$ satisfies Eq.~\eqref{proof-sketch-2} for any $t\in[1:T_2]$, then $\{\upw^t\}_{t=1}^{T_2}$ coincides with $\{\upv^{T_1+t}\}_{t=1}^{T_2}$. The novelty and ingenuity of our analysis based on auxiliary sequence construction lie in the alignment of  $\{\upw^t\}_{t=1}^{T_2}$ and $\{\upv^{T_1+t}\}_{t=1}^{T_2}$ as $\upw^{T_2}=\upv^{T}$ with high probability by Theorem \ref{phase-2-step-1}. Thus the update rule of $\upw^t$ satisfies the following formula with high probability:
\begin{equation}\label{proof-sketch-3}
    \begin{split}
        \upw^{t+1}=\upw^t-\eta_t\upH^t(\upw^t-\upv_{1:M}^*)+\eta_t\upR^t\Pi_M\upx^t,
    \end{split}
\end{equation}
where $\upH^t\in\bbR^{M\times M}$ depends on $\upw^t$ and $\upx^t$, and $\upR^t\in\bbR^{M\times M}$ depends on $\upw^t, \zeta_{M+1:\infty}^t$ and $\xi^t$. Combining Eq.~\eqref{proof-sketch-3} with the constraint of $\{\upw^t\}_{t=1}^{T_2}$, we observe that the update process of $\upw^t$ approximates that of SGD in traditional linear regression problems \citep{wu2022last} with reparameterized features $\Pi_M\upx\odot\upv_{1:M}^*$. The SGD iteration in linear model exhibits structural similarity to Eq.~\eqref{proof-sketch-3}, but differs in that its $\upH^t$ and $\upR^t$ are independent on iterative variables; this independence eliminates the need for truncated sequences in analytical treatments. Our analysis innovatively introduces a truncated sequence $\{\upw^t\}_{t=1}^{T_2}$ to maintain analytical tractability of $\upH^t$ and $\upR^t$.

According to Eq.~\eqref{proof-sketch-3}, we decompose the risk $\mathcal{R}_M(\upw^{T_2})$ as follows:
\begin{align}\label{phase-II-eq1}
    \bbE\left[\mathcal{R}_M(\upw^{T_2})\right]-\mathcal{R}_M(\upv_{1:M}^*)\lesssim\underbrace{\llangle\upH,\upB^{T_2}\rrangle}_{\text{bias error}}+\underbrace{\llangle\upH,\upV^{T_2}\rrangle}_{\text{variance error}}.
\end{align}
For any $t\in[1:T_2]$, $\upB^t$ and $\upV^t$ are $M\times M$ matrices, derived from the bias and variance terms induced by $\upw^t-\upv_{1:M}^*$, respectively. 
Since $\upH^t$ and $\upR^t$ in Eq.~\eqref{proof-sketch-3} are both dependent on $\upw^t$, it is a challenge to directly establish the full-matrix recursion between $\upV^{t+1}$ and $\upV^t$ (or $\upB^{t+1}$ and $\upB^t$) under the SGD iteration process like the similar techniques in linear models \citep{wu2022last}. To resolve this challenge, we turn to the recursive relations between diagonal elements of $\{\upV^t\}_{t=0}^{T_2}$ and $\{\upB^t\}_{t=0}^{T_2}$, which can approximate to describe the iterate of $\{\upV^t\}_{t=0}^{T_2}$ and $\{\upB^t\}_{t=0}^{T_2}$. We novelly consider the diagonal elements of the above sequences across discrete time steps, thereby obtaining the estimation for both variance and bias errors for our linear approximation. 


\section{Conclusions}
In this paper, we construct the theoretical analysis for the dynamic of quadratically parameterized model under decaying ground truth and anisotropic gradient noise. Our technique is based on the precise analysis of two-stage dynamic of SGD, with adaptive selection of the effective dimension set in the first stage and the approximation of linear model in the second stage. Our analysis characterizes the feature learning and model adaptation ability with clear separations for convergence rates in the canonical linear model.

\appendix
\tableofcontents
\newpage


\section{Proofs of Main Results}

In this section, we introduce our proof techniques to prove our main result Theorem \ref{theorem-main-convergence} on the upper bound of the last-iteration instantaneous risk of Algorithm \ref{SGD}. As shown in Section \ref{sec:proof sketch}, the dynamic of SGD and our analysis can be basically divided into two phases. In the \textbf{Phase I} named ``adaption" phase, we demonstrate that SGD can adaptively identify the first $D$ coordinates as the optimal set $\mathcal{S}$ without explicit selection of $D$, and bound such $D$ coordinates near the corresponding optimal solutions by $T_1$ iterations with high probability (refer to Theorem \ref{theorem_1}). The analysis of \textbf{Phase I} can be further separated into two parts:

\begin{itemize}
    \item [1.] We construct a high-probability upper bound of $\upv^{T_1}$. That is for any $i \in \mathcal{S}$, $\upv^{T_1}_i \leq (1+c_1)\upv_i^*$ and for any $i \in \mathcal{S}^c$, $\upv^{T_1}_i \leq \frac{3}{2}\upv_i^*$ (refer to Lemma \ref{lemma_1}).
    \item [2.] We delve into the lower bound of $\max_{t\leq T_1} \upv_i^t$ during $T_1$ iterations. With high probability, for any $i\in\mathcal{S}$, $\max_{t\leq T_1} \upv_i^t$ converges to a neighborhood of $\upv_i^*$ (refer to Lemma \ref{lemma-2}). When $\max_{t\leq T_1}\upv_i^t$ resides within the $\upv_i^*$-neighborhood, the lower bound satisfies $\upv_i^{T_1} \geq (1 - c_1)\upv_i^*$ with high probability (refer to Lemma \ref{lemma-3}).
\end{itemize}

Then we turn to the following \textbf{Phase II} with $T_2$ iterations named ``estimation" phase where we establish the global convergence of Algorithm \ref{SGD} for risk minimization (refer to Theorem \ref{phase-II-main}). The analysis of Algorithm \ref{SGD}'s iterations can be approximated to SGD with geometrically decaying step sizes on a linear regression problem with reparameterized features $\Pi_M\upx\odot\upv_{1:M}^*$. It can also be separated into two parts:

\begin{itemize}
    \item [1.] We demonstrate that $\{\upv^t\}_{t=T_1+1}^{T_1+T_2}$ remain confined within the neighborhood $\prod_{i=1}^D[\frac{1}{2}\upv_i^*,\frac{3}{2}\upv_i^*]\times\prod_{i=D+1}^M[0,2\upv_i^*]$ with high probability (refer to Lemma \ref{high-probability-phase-II}).
    \item [2.] We construct an auxiliary sequence $\{\upw^t\}_{t=1}^{T_2}$ aligned to $\{\upv^{T_1+t}\}_{t=1}^{T_2}$ with high probability. We approximate the update process of $\{\upw^t\}_{t=1}^{T_2}$ to SGD in traditional linear regression, with separated bounds of variance term (refer to Lemma \ref{variance-upper-bound}) and bias term (refer to Lemma \ref{bias-upper-bound}).
\end{itemize}

We propose our proof process step by step according to the above sketch. First, for clarity, we formally define some of the notations to use. We let bold lowercase letters, for example, $\bold{x} \in \mathbb{R}^d$, denote vectors, and bold uppercase letters, for example, $\mathbf{A} \in \mathbb{R}^{m \times n}$, denote matrices. We apply scalar operators to vectors as the coordinate-wise operators of vectors. For vector $\mathbf{x}\in \mathbb{R}^d$, denote  $\mathbf{\left | x  \right | } \in \mathbb{R}^d$ with $\mathbf{\left | x  \right | }_j = \left | \mathbf{x} _j \right | $. For two vectors $\mathbf{x}, \mathbf{y}\in \mathbb{R}^d$, denote $\mathbf{x}\le \mathbf{y}$, if for all $j\in \left [ d \right ]$, $\mathbf{x}_j\le \mathbf{y}_j$. Additionally, we use $\langle\mathbf{x},\mathbf{y}\rangle_{-i}$ to denote $\sum_{j=1\atop j\neq i}^d\mathbf{x}_j\mathbf{y}_j$. For a sequence of real numbers $\{v^t\}_{t=t_1}^{t_2}$ and $a, b \in \mathbb{R}$ with $a \leq b$, denote $v^{t_1:t_2} \in [a, b]$ to represent that $v^t \in [a, b]$ for all $t \in [t_1, t_2]$. 


Considering Assumption \ref{ass-d}, the random variable $\Pi_M \upx \in \mathbb{R}^M$ satisfies the sub-Gaussian condition with parameter $\lambda_i^{1/2}$ for all $i \in [1:M]$, and the noise $\xi$ is zero-mean sub-Gaussian with parameter $\sigma_{\xi}$. Moreover, define the random variable $\zeta_{M+1:\infty}=\sum_{i=M+1}^{\infty}\upx_i(\upv_i^*)^2$. For any $D\in\bbN_+$, for simplification, we define
\begin{align}
    &\sigmin(D):=\min_{j\in[1:D]}\lambda_{j}(\upv_j^*)^2,\quad \Barsigmin(D):=\min_{j\in[1:D]}(\upv_j^*)^2,\notag
    \\
    &\Hatsigmax(D):=\max_{j\in[1:D]}\log^{-1}(\upv_j^0),\quad \Tildesigmax(D):=\max_{j\in[D+1:M]}\lambda_j.\notag
\end{align}
We also denote the matrix $\diag\{\lambda_1, \dots, \lambda_M\}$ as $\Lambda_{1:M}$. For $\mathbf{b} \in \mathbb{R}_+^M$, we define $\mathcal{M}(\mathbf{b}) = (\sum_{j=1}^M$ $ \lambda_j \mathbf{b}_j^4)^{1/2}$ and $\sigma^2 = \sigma_{\xi}^2 + \sigma_{\zeta_{M+1:\infty}}^2$, where $\sigma_{\zeta_{M+1:\infty}} = (\sum_{j=M+1}^\infty \lambda_j (\upv_j^*)^4)^{1/2}$. We denote $$\calF^t=\sigma\{\mathbf{v}^0,(\Pi_M\upx^1,\zeta_{M+1:\infty}^1,\xi^1),\cdots,(\Pi_M\upx^t,\zeta_{M+1:\infty}^t,\xi^t)\}$$ as the filtration involving the full information of all the previous $t$ iterations with  $\sigma\{\cdot\}$. 

\subsection{High-Probability Results Guarantee}
Before the analyses of the two phases, we first introduce the guarantee of our high-probability results. We formally define a series of events for each iteration of Algorithm \ref{SGD}. We demonstrate that these events occur with high probability throughout the whole $T$ iterations, which indicates that the control sequence $\{\upq^t\}_{t=0}^{T}$ we define is aligned with the original sequence $\{\upv^t\}_{t=0}^{T}$ with high probability. This fact is the basis of our high-probability results

We begin with $\upx$, the covariate vector in $\bbH$, and its projection $\Pi_M\upx$. Let $\{\upu_i\}_{i=1}^{\infty}$ be an orthonormal basis of $\bbH$ and $\upx_i=\langle\upx,\upu_i\rangle$ for any $i\in[1:\infty)$. The projection operator $\Pi_M:\upH\rightarrow\bbR^M$ is defined as $\Pi_M=\sum_{i=1}^M\upe_i\otimes\upu_i$. Thus the projection $\Pi_M\upx$ satisfies $\Pi_M\upx=(\upx_1,\cdots,\upx_M)^{\top}$. At the $t$-th iteration, Algorithm \ref{SGD} requires sampling $(\Pi_M\upx^{t+1},y^{t+1})$, where $y^{t+1}=\langle\Pi_M\upx^{t+1},\upv_{1:M}^*\rangle+\zeta_{M+1:\infty}^{t+1}+\xi^{t+1}$.
In order to simply rule out some low-probabilistic unbounded cases, for each iteration $t$, we define the following four events as:
\begin{equation}
\left\{
\begin{aligned}
 \calE_1^{j,t}&:=\left\{\left|\upx_j^t\right|\leq\lambda_j^{1/2}R\right\}, \quad \forall j\in[1:M],\notag \\
 \calE_2^{j,t}(\upv)&:=\left\{\left|\langle\upv^{\odot2}-\upv_{1:M}^{*\odot2},\Pi_M\upx^{t}\rangle_{-j}\right|\leq r_j(\upv)R\right\}, \quad \forall j\in[1:M],\notag \\
 \calE_3^t&:=\left\{\left|\zeta_{M+1:\infty}^t\right|\leq\sigma_{\zeta_{M+1:\infty}}R\right\}, \\
 \calE_4^t&:=\left\{\left|\xi^t\right|\leq\sigma_{\xi}R\right\},\notag 
\end{aligned}
\right\}
\end{equation}
where $R:=\calO(\log(MT/\delta))$ and $r_j(\upv):=\calO(\sum_{i\neq j}\lambda_i[(\upv_i)^4+(\upv_i^*)^4])^{1/2}$ for any $\upv\in\bbR^M$. 

In Algorithm \ref{SGD}, the original sequence $\{\upv^t\}_{t=0}^T$ follows the coordinate-wise update rule as
\begin{equation}
    \begin{split}
    \upv_j^{t+1}=&\upv_j^t-\eta_t\left(\langle\upv^{t\odot2},\Pi_M\upx^{t+1}\rangle-y^{t+1}\right)\upx_j^{t+1}\upv_j^t\\
    \notag
    =&\upv_j^t-\eta_t\llangle\upv^{t\odot2}-\upv_{1:M}^{*\odot2},\Pi_M\upx^{t+1}\rrangle\upx_j^{t+1}\upv_j^t+\eta_t\left(\zeta_{M+1:\infty}^{t+1}+\xi^{t+1}\right)\upx_j^{t+1}\upv_j^{t},
    \end{split}\notag
\end{equation}
for any $j\in[1:M]$. Based on Assumption \ref{ass-d} and Proposition \ref{prop-A5}, we have
\begin{align}
    \min\left\{\bbP\left(\calE_1^{j,t}\right),\bbP\left(\calE_2^{j,t}(\upv^t)\right),\bbP\left(\calE_3^{t}\right),\bbP\left(\calE_4^{t}\right)\right\}\geq 1-\calO\left(\frac{\delta}{MT^2}\right),\notag
\end{align}
for any $j\in[1:M]$ and $t\in[0:T-1]$. Then we define the compound event as 
$$
\calE:=\left\{\bigcap_{t=0}^{T-1}\left(\left(\bigcap_{j=1}^M\calE_1^{j,t}\right)\bigwedge\left(\bigcap_{j=1}^M\calE_2^{j,t}(\upv^t)\right)\bigwedge\calE_3^t\bigwedge\calE_4^t\right)\right\}.
$$
We can directly obtain the probability union bound as follows:
\begin{align}
    \bbP(\calE)=1-\bbP(\calE^c)\geq&1-\sum_{t=1}^T\left(2-\bbP(\calE_3^t)-\bbP(\calE_4^t)+\sum_{j=1}^M\left(2-\bbP\left(\calE_1^{j,t}\right)-\bbP\left(\calE_2^{j,t}(\upv^t)\right)\right)\right)\notag
    \\
    \geq&1-\calO\left(\frac{\delta}{T}\right).
\end{align} 
The high-probability occurrence of event $\calE$ guarantees our analysis of the coordinate-wise update dynamics for the control sequence $\{\upq^t\}_{t=0}^{T}$ defined in $\bbR^M$ as
\begin{align}\label{update-q-primal}
    \upq_j^{t+1}=\upq_j^t&-\eta_t\left(\left(\upq_j^t\right)^2-\left(\upv_j^*\right)^2\right)\left(\upx_j^{t+1}\right)^2\mathds{1}_{|\upx_j^{t+1}|\leq\lambda_j^{1/2}R}\upq_j^t\notag
    \\
    &-\eta_t\llangle\upq^{t\odot2}-\upv_{1:M}^{*\odot2},\Pi_M\upx^{t+1}\rrangle_{-j}\mathds{1}_{|\langle\upq^{t\odot2}-\upv_{1:M}^{*\odot2},\Pi_M\upx^{t+1}\rangle_{-j}|\leq r_j(\upq^t)R}\upx_j^{t+1}\mathds{1}_{|\upx_j^{t+1}|\leq\lambda_j^{1/2}R}\upq_j^t\notag
    \\
    &+\eta_t\left(\zeta_{M+1:\infty}^{t+1}\mathds{1}_{\left|\zeta_{M+1:\infty}^t\right|\leq\sigma_{\zeta_{M+1:\infty}}R}+\xi^{t+1}\mathds{1}_{\left|\xi^t\right|\leq\sigma_{\xi}R}\right)\upx_j^{t+1}\mathds{1}_{|\upx_j^{t+1}|\leq\lambda_j^{1/2}R}\upq_j^t,
\end{align}
for any $j\in[1:M]$ with initialization $\upq^0=\upv^0$ is consistent with the analysis of $\{\upv^t\}_{t=0}^{T}$ with high probability as Proposition \ref{p1}.
\begin{proposition}\label{p1}
    For any $t\in[1:T]$, we have $\upv^{t}=\upq^{t}$ with probability at least $1-\delta/T$.
\end{proposition}

To simplify the representation of $\{\upq^t\}_{t=0}^{T}$, we introduce four truncated random variables as:
\begin{enumerate}
    \item $\widehat{\upx}\in\bbR^M$ with entries $\widehat{\upx}_j=\upx_j\mathds{1}_{|\upx_j|\leq\lambda_j^{1/2}R}$ for any $j\in[1:M]$,
    \item $\widehat{\upz}(\upq)\in\bbR^M$ with entries $\widehat{\upz}_j(\upq)=\llangle\upq^{\odot2}-\upv_{1:M}^{*\odot2},\Pi_M\upx\rrangle_{-j}\mathds{1}_{|\langle\upq^{\odot2}-\upv_{1:M}^{*\odot2},\Pi_M\upx\rangle_{-j}|\leq r_j(\upq)R}$
    \item $\widehat{\zeta}_{M+1:\infty}=\zeta_{M+1:\infty}\mathds{1}_{\zeta_{M+1:\infty}\leq \sigma_{\zeta_{M+1:\infty}}R}$,
    \item $\widehat{\xi}=\xi\mathds{1}_{\xi\leq \sigma_{\xi}R}$. 
\end{enumerate}
Thus, the coordinate-wise update dynamics for $\{\upq^{t}\}_{t=0}^T$ in Eq.~\eqref{update-q-primal} can be represented as:
\begin{align}\label{update-q}
    \upq_j^{t+1}=&\upq_j^t-\eta_t\left(\left(\upq_j^t\right)^2-\left(\upv_j^*\right)^2\right)\left(\widehat{\upx}_j^{t+1}\right)^2\upq_j^t-\eta_t\widehat{\upz}_j^{t+1}(\upq^t)\widehat{\upx}_j^{t+1}\upq_j^t\notag\notag
    \\
    &+\eta_t\left(\widehat{\zeta}_{M+1:\infty}^{t+1}+\widehat{\xi}^{t+1}\right)\widehat{\upx}_j^{t+1}\upq_j^t,
\end{align}
for any $j\in[1:M]$.

\subsection{Proof of Phase I}\label{phase 1}

In this section, we formally propose the proof techniques of \textbf{Phase I} in Theorem \ref{theorem_1}. Theorem \ref{theorem_1} establishes that Algorithm \ref{SGD} adaptively selects a effective dimension $D \in \mathbb{N}_+$ with the following convergence properties: (1) for $j \leq D$, $\upv_j^{T_1}$ converges to an adaptive neighborhood of $\upv_j^*$; (2) for $j > D$, $\upv_j^{T_1}$ is bounded by $\frac{3}{2}\max\{\upv_j^*,2\upv_j^0\}$. Theorem \ref{theorem_1} specifies the intrinsic relationship between Algorithm \ref{SGD}'s key parameters: the recommended step size $\eta$, effective dimension $D$, and total sample size $T$. Furthermore, under Assumption \ref{ass-ss}, Phase II analysis demonstrates the optimality of the effective dimension $D$ selected in Theorem \ref{theorem_1}.


\begin{theorem}\label{theorem_1}[Formal version of Theorem \ref{phase-I-informal} ]
	Under Assumption \ref{ass-d}, consider the dynamic generated via Algorithm \ref{SGD} with initialization $\mathbf{v}_0$. Denote (1) the threshold vector $\hat{\mathbf{v}}^* \in \mathbb{R}^M$ with coordinate $\hat{\mathbf{v}}_j^* = \max\left\{\frac{3}{2}\mathbf{v}_j^*, 3\mathbf{v}_j^0\right\}$ for any $j \in [1:M]$; (2) the composite vector $\mathbf{b} = ((1 + c_1)(\mathbf{v}_{1:D}^*)^{\top}, (\hat{\mathbf{v}}_{D+1:M}^*)^{\top})^{\top}$, where the scaling constant $c_1 \in (0, 1/2)$. Let the step size $\eta$ satisfy $\eta \leq \widetilde{\Omega}\left(\frac{c_1^2 \Barsigmin(\max\{D,M\})}{[\sigma^2 + \mathcal{M}^2(\upb)]^2}\right)$ for the given effective dimension $D \in \mathbb{N}_+$. If the iteration number $T_1$ requires:
    \begin{align}
        T_1\in\begin{cases}
        \left[\widetilde{\mathcal{O}}\left(\frac{\sigma^2 + \mathcal{M}^2(\upb)}{c_1^2 \eta \sigmin(D)}\right): \widetilde{\Omega}\left(\frac{\Tildesigmax^{-1}(D)}{\eta^2[\sigma^2 + \mathcal{M}^2(\upb)]}\right)\right], \quad &\text{ if }D < M,\notag
        \\
\left[\widetilde{\mathcal{O}}\left(\frac{\sigma^2 + \mathcal{M}^2(\upb)}{c_1^2 \eta \sigmin(M)}\right):\infty\right), \quad &\text{ otherwise},
    \end{cases}
    \end{align}
    then the dynamic satisfies the following convergence property:
    \begin{align}\label{thm-1-eq}
        \upv_j^{T_1} \in \begin{cases}
            \left[\upv_j^{*}-c_1\upv_j^*, \upv_j^{*}+c_1\upv_j^*\right], &\text{ if }j\in[1:D],
            \\
            \left[0, \frac
            {3}{2}\max\{\upv_j^*,2\upv_j^0\}\right], &\text{ otherwise },
        \end{cases}
    \end{align}
    with probability at least $1-\delta$.
\end{theorem}

Before the beginning of our proof, we define the $\mathbf{b}$-capped coupling processes used in the following lemmas as below. 
\begin{definition}[$\mathbf{b}$-capped coupling] \label{def:b-capped}
  Let $\{\mathbf{q}^t\}_{t=0}^T$ be a Markov chain in $\mathbb{R}_+^M$ adapted to filtration $\{\mathcal{F} ^{t}\}_{t=0}^T$. Given threshold vector $\mathbf{b}\in\bbR_+^{M}$, the $\mathbf{b}$-capped coupling process $\{\bar{\mathbf{v}}^t\}_{t=0}^T$ with initialization $\bar{\mathbf{v}}^0=\mathbf{q}^0\le \mathbf{b}$ evolves as:
  \begin{enumerate}
      \item Updating state: If $\bar{\mathbf{v}}^t\le \mathbf{b}$, let $\bar{\mathbf{v}}^{t+1}=\mathbf{q}^{t+1}$,
      \item Absorbing state: Otherwise, maintain $\bar{\mathbf{v}}^{t+1}=\bar{\mathbf{v}}^{t}$. 
  \end{enumerate} 
\end{definition}

\subsubsection{Part I: the coordinate-wise upper bounds of $\upv^{T_1}$.}
In this part, we establish coordinate-wise upper bounds for $\upv^{T_1}$ in Lemma \ref{lemma_1}. For each coordinate $i \in [1:M]$, we develop a geometrically compensated supermartingale $\{u_i^t:=(1-\eta\Theta(\lambda_i(\upv_i^*)^2))^{-t}(\Bar{\upv}_i^t-\upv_i^*)\}_{t=1}^{T_1}$ using the $\mathbf{b}$-capped coupling sequence $\{\Bar{\upv}^t\}_{t=0}^{T_1}$ derived from the control sequence $\{\upq^t\}_{t=0}^{T_1}$. We precisely calculate the sub-Gussian parameters of the supermartingale increments through geometric series summation over $\mathcal{S}$ and linear summation over $\mathcal{S}^c$. The analysis enables the application of Bernstein-type inequalities to establish the claimed concentration results in Lemma \ref{lemma_1}.

\begin{lemma}\label{lemma_1}[Formal version of Lemma \ref{phase-1-step-1}]
Under the setting of Theorem \ref{theorem_1}, let $\{\mathbf{q}^t\}_{t=0}^{T_1}$ be a Markov chain with its $\mathbf{b}$-capped coupling process $\{\bar{\mathbf{v}}^t\}_{t=0}^{T_1}$. When $\eta\leq\widetilde{\Omega}\left(\frac{1}{\sigma^2+\calM^2(\upb)}\right)$, the inequality $\bar{\upv}^t\geq\textbf{0}$ holds for any $t\in[1:T_1]$. For any $\mathbf{v}\in \mathbb{R}^M$, define the truncation event $\calA(\mathbf{v}):=\{\mathbf{v}\leq\mathbf{b}\}$. For $\delta\in(0,1)$, the following conditions guarantee that $\calA(\Bar{\mathbf{v}}^{T_1})$ holds with probability at least $1-\frac{\delta}{6}$:
\begin{enumerate}
    \item Dominant coordinates condition: $\Barsigmin(D)\geq\frac{\eta}{c_1^2}\calO([\sigma^2+\calM^2(\upb)]\log^5(MT_1/\delta))$,
    \item Residual spectrum condition: $\Tildesigmax(D)\geq T_1\eta^2\calO([\sigma^2+\calM^2(\upb)]\log(\max\{M-D,0\}T_1/\delta)$ $\log^4(MT_1/\delta))$.
\end{enumerate}
\end{lemma}

\begin{proof}
Define the random variable 
$$
p_j^{t+1}:=\left(\left((\bar{\mathbf{v}}_j^{t})^2-(\mathbf{v}_{j}^{*})^2\right)\hat{\mathbf{x}}_{j}^{t+1}+\hat{\upz}_j^{t+1}(\bar{\upv}^t)-\hat{\zeta}_{M+1:\infty}^{t+1}-\hat{\xi}^{t+1}\right)\hat{\mathbf{x}}_j^{t+1}
$$
for any $j\in[1:M]$ and $t\in[0:T_1-1]$. Then in the updating state of $\{\bar{\upv}^t\}_{t=0}^{T_1}$, we have
\begin{align}
    \bar{\upv}_j^{t+1}=(1-\eta p_j^{t+1})\bar{\upv}_j^{t},\quad \forall j\in[1:M].
\end{align}
Based on the boundedness of $p_j^{t+1}$ and the appropriately chosen step size $\eta\leq\widetilde{\Omega}\left(\frac{1}{\sigma^2+\calM^2(\upb)}\right)$, if $\bar{\upv}^t>\textbf{0}$, then we have $\bar{\upv}^{t+1}\geq\frac{1}{2}\bar{\upv}^t$. Since $\bar{\upv}^0>0$,  we have $\bar{\upv}^t>0$ for any $t\in[1:T_1]$ by induction.
Let $\bar{\tau}_{\mathbf{b} }$ be the stopping time when $\bar{\mathbf{v}}_j^{\bar{\tau}_{\mathbf{b} }} > \mathbf{b}_j$ for a certain coordinate $j\in[1:M]$, i.e.,
\begin{equation}\nonumber
      \bar{\tau}_{\mathbf{b} }=\inf_{t}\left \{ t:\exists j\in[1:M], \text{ s.t. }\bar{\mathbf{v}}_j^t > \mathbf{b}_j \right\}.
\end{equation}
For each coordinate $1\le j\le M$, let $\bar{\tau}_{\mathbf{b},j }$ be the stopping time when $\bar{\mathbf{v}}^{\bar{\tau}_{\mathbf{b},j }}_j > \mathbf{b}_j$, i.e.,
\begin{equation}\nonumber
\bar{\tau}_{\mathbf{b},j }=\inf_{t}\left \{ t:\bar{\mathbf{v}}_j^t > \mathbf{b}_j \right \}.
\end{equation}
Based on Definition \ref{def:b-capped}, when the stopping time $\bar{\tau}_{\mathbf{b}}=t_2$ occurs for some $t_2\in[1:T_1]$, the coupling process satisfies $\bar{\mathbf{v}}^t=\bar{\mathbf{v}}^{t_2}$ for all $t>t_2$. We categorize the following two cases and analyze the probability bound respectively. 

\noindent \textbf{Case I:} Suppose there exists $j\in[1:D]$ such that $\bar{\tau}_{\mathbf{b},j }=t_2$. 
That is, the event 
$\mathcal{A} \left ( \bar{\mathbf{v}} ^t \right ) $ holds for all $t\in[0:t_2-1]$. 
The boundedness of $p_j^{t+1}$ and the dominant coordinates condition of $\eta$ in Lemma \ref{lemma_1} indicate that $\bar{\upv}_j^t$ must traverse in and out of the threshold interval $\left[\frac{1}{1+c_1}\mathbf{b}_j,\mathbf{b}_j\right]$ before exceeding $\upb_j$.
We aim to estimate the following probability for coordinates $j\in[1:D]$ and time pairs $t_1<t_2\in[1:T_1]$:
\begin{equation}\nonumber
\bbP\left(\mathcal{B}_{t_1}^{\bar{\tau}_{\mathbf{b},j }=t_2}(j)=\left \{ \bar{\mathbf{v}}_j^{t_1}\leq\frac{1+c_1/2}{1+c_1}\mathbf{
b} _j\bigwedge \bar{\mathbf{v}}_j^{t_1:t_2-1}\in\left[\frac{1}{1+c_1}\mathbf{b}_j,\mathbf{b}_j\right]\bigwedge\bar{\mathbf{v}}_j^{t_2}>\mathbf{b}_j  \right \}\right). 
\end{equation}
For any $t\in[t_1:t_2-1]$, we have 
\begin{equation}\label{mar-concen-1}
    \begin{aligned}
        \bbE\left[\bar{\mathbf{v}}_j^{t+1}-\mathbf{v}_j^*\mid\calF^{t}\right]=&\bbE_{\mathbf{x}_{1:M}^{t+1},\xi^{t+1},\zeta_{M+1:\infty}^{t+1}}\left[\bar{\mathbf{v}}_j^{t}-\mathbf{v}_j^*-\eta p_j^{t+1}\bar{\mathbf{v}}_j^{t}\right]
		\\
		\overset{\text{(a)}}{\leq}&\left(1-\frac{1}{2}\eta\lambda_j\bar{\mathbf{v}}_j^{t}(\bar{\mathbf{v}}_j^{t}+\mathbf{v}_j^*)\right)(\bar{\mathbf{v}}_j^{t}-\mathbf{v}_j^*)
		\\
		\leq&\left(1-\frac{1+c_1/2}{\left(1+c_1\right)^2}\eta\lambda_j(\mathbf{v}_j^*)^2\right)(\bar{\mathbf{v}}_j^{t}-\mathbf{v}_j^*),
    \end{aligned}
\end{equation}
where (a) is due to Assumption \ref{ass-d} and Lemma \ref{aux-1.1}. By applying Lemma \ref{aux-1} to $p_j^t$, we demonstrate that $p_j^t$ satisfies the sub-Gaussian property for all $t\in[0:T_1-1]$. Thus we have
\begin{equation}
\small
\begin{split}
    \bbE\left[\exp\left\{\lambda\left(\bar{\mathbf{v}}_j^{t+1}-\bbE[\bar{\mathbf{v}}_j^{t+1}\mid\calF^t]\right)\right\}\mid\calF^t\right]\leq \exp\left\{\frac{\lambda^2\eta^2\lambda_j(\upv_j^*)^2\calO\left(\left[\sigma^2+\mathcal{M}^2(\upb)\right]\log^4(MT_1/\delta)\right)}{2}\right\},\notag
\end{split}
\end{equation}
for any $\lambda\in\bbR$.
Combining Lemma \ref{aux-2} with Eq.~\eqref{mar-concen-1}, we can establish the probability bound for event $\mathcal{B}_{t_1}^{\bar{\tau}_{\mathbf{b},j }=t_2}(j)$ for any time pair $t_1<t_2\in[1:T_1]$ as
\begin{align}\label{probability-1}
\Pro\left(\mathcal{B}_{t_1}^{\bar{\tau}_{\mathbf{b},j }=t_2}(j)\right )\leq&\exp\left\{-\frac{c_1^2(\mathbf{v}_j^*)^2}{\eta\calO\left([\sigma^2+\mathcal{M}^2(\upb)]\log^4(MT_1/\delta)\right)}\right\}.
\end{align}
	
\noindent \textbf{Case II:} Suppose there exists $j\in[D+1:M]$ such that $\bar{\tau}_{\mathbf{b},j }=t_2$.  Similarly, $\bar{\upv}_j^t$ must traverse in and out of the threshold interval $[\frac{2}{3}\upb_j,\upb_j]$ before exceeding $\upb_j$. Therefore, we aim to estimate the following probability for coordinates $j\in[D+1:M]$ and time pairs $t_1<t_2\in[1:T_1]$:
\begin{equation}\nonumber
    \bbP\left(\calC_{t_1}^{\bar{\tau}_{\mathbf{b},j }=t_2}(j)=\left\{\bar{\mathbf{v}}_j^{t_1}\leq\frac{3}{4}\mathbf{b}_j\bigwedge\bar{\mathbf{v}}_j^{t_1:t_2-1}\in\left[\frac{2}{3}\upb_j,\mathbf{b}_j\right]\bigwedge \bar{\mathbf{v}}_j^{t_2}>\mathbf{b}_j \right\}\right).
\end{equation}
For any $t\in[t_1:t_2-1]$, we have 
\begin{equation}\label{mar-concen-2}
    \begin{aligned}
         \bbE\left[\bar{\mathbf{v}}_j^{t+1}-\mathbf{v}_j^*\mid\calF^{t}\right]=&\bbE_{\mathbf{x}_{1:M}^{t+1},\xi^{t+1},\zeta_{M+1:\infty}^{t+1}}\left[\bar{\mathbf{v}}_j^{t}-\mathbf{v}_j^*-\eta p_j^{t+1}\bar{\mathbf{v}}_j^{t}\right]
		\\
		\leq&\left(1-\frac{1}{2}\eta\lambda_j\bar{\mathbf{v}}_j^{t}(\bar{\mathbf{v}}_j^{t}+\mathbf{v}_j^*)\right)(\bar{\mathbf{v}}_j^{t}-\mathbf{v}_j^*)
		\\
		\leq &\bar{\mathbf{v}}_j^{t}-\mathbf{v}_j^*.
    \end{aligned}
\end{equation}
Similarly, based on Lemma \ref{aux-1}, we have
\begin{equation}
    \small
    \begin{split}
        \bbE\left[\exp\left\{\lambda(\bar{\upv}_j^{t+1}-\bbE[\bar{\upv}_j^{t+1}\mid\calF^t])\right\}\mid\calF^t\right]\leq\exp\left\{\frac{\lambda^2\eta^2\lambda_j(\upb_j^*)^2\calO\left(\left[\sigma^2+\mathcal{M}^2(\upb)\right]\log^4(MT_1/\delta)\right)}{2}\right\},\notag
    \end{split}
\end{equation}
for any $\lambda\in\bbR$. Combining Lemma \ref{aux-2} with Eq.~\eqref{mar-concen-2}, we can establish the  probability bound for event $\mathcal{C}_{t_1}^{\bar{\tau}_{\mathbf{b},j }=t_2}(j)$ for any time pair $t_1<t_2\in[1:T_1]$ as
	\begin{align}\label{probability-2}
		\Pro\left (\calC_{t_1}^{\bar{\tau}_{\mathbf{b},j }=t_2}(j)\right )\leq\exp\left\{-\frac{1}{T\eta^2\lambda_j\calO\left([\sigma^2+\calM^2(\upb)]\log^2(MT_1/\delta)\right)}\right\}.
	\end{align}

Finally, combining the probability bounds Eq.~\eqref{probability-1} and Eq.\eqref{probability-2} with the dominant coordinates condition and residual spectrum condition in Lemma \ref{lemma_1}, we obtain the following probability bound for complement event $\calA^c(\bar{\mathbf{v}}^{T_1})$:
	\begin{align}\label{union-1}
		\Pro\left(\calA^c(\bar{\mathbf{v}}^{T_1})\right ) \leq&\sum_{j=1}^D\sum_{1\leq t_1<t_2\leq T_1}\Pro\left(\mathcal{B}_{t_1}^{\bar{\tau}_{\mathbf{b},j }=t_2}(j)\right)\notag+\sum_{j=D+1}^M\sum_{1\leq t_1<t_2\leq T_1}\Pro\left(\calC_{t_1}^{\bar{\tau}_{\mathbf{b},j }=t_2}(j)\right) \notag
		\\
		\leq&\frac{NT_1^2}{2}\exp\left\{-\frac{c_1^2\min_{1\le j\le D}(\mathbf{v}_j^*)^2}{\eta\calO\left(\left[\sigma^2+\calM^2(\upb)\right]\log^4(MT_1/\delta)\right)}\right\}\notag
		\\
		&+\max\{M-D,0\}T_1\exp\left\{-\frac{\min_{D+1\le j\le M}\lambda_j^{-1}}{T_1\eta^2\calO\left([\sigma^2+\calM^2(\upb)]\log^4(MT_1/\delta)\right)}\right\}\notag
		\\
		\leq&\frac{\delta}{12}.
	\end{align}
\end{proof}

Lemma \ref{lemma_1} establishes the adaptive high-probability upper bounds for each coordinate of $\bar{\upv}^{T_1}$. According to the construction methodology of the coupling process $\{\bar{\upv}^t\}_{t=0}^{T_1}$, these bounds can be naturally extended to $\upq^{T_1}$. Moreover, the high-probability consistency between control sequence $\{\upq^t\}_{t=0}^{T}$ and original sequence $\{\upv^t\}_{t=0}^{T}$ (refer to Proposition \ref{p1}) allows the direct application of Lemma \ref{lemma_1} to $\upv^{T_1}$. It similarly holds for  Lemmas \ref{lemma-2} and \ref{lemma-3}, respectively.

\subsubsection{Part II: The coordinate-wise lower bounds of $\bar{\upv}^{T_1}$}
Deriving a direct high-probability lower bound for $\bar{\upv}^{T_1}$ proves to be a challenge. We turn to the lower bound of $\max_{t\leq T_1} \bar{\upv}_i^t$ during $T_1$ iterations. First  we propose Lemma \ref{lemma-2} to construct such bounds for $\max_{t\leq T_1}\bar{\upv}_j^t$ adaptively over $j\in[1:D]$. We derive a subcoupling sequence $\{\Breve{\upv}^{i,t}\}_{t=0}^{T_1}$ from the original coupling sequence $\{\Bar{\upv}^t\}_{t=0}^{T_1}$ for any $i\in\mathcal{S}$. Each subcoupling sequence undergoes logarithmic transformation to generate a linearly compensated submartingale $\{-t\log(1+\eta\calO(\lambda_i(\upv_i^*)^2))+\log(\breve{\upv}_i^{i,t})\}_{t=1}^{T_1}$. These $|\mathcal{S}|$ submartingales exhibit monotonic growth with sub-Gaussian increments. Applying Bernstein-type concentration inequalities, we obtain $\max_{t\leq T_1}\upv_i^t\geq(1-c_1/2)\upv_i^*$ with high probability for any $i\in\mathcal{S}$ in Lemma \ref{lemma-2}.  
 
\begin{lemma}\label{lemma-2}[Formal version of Lemma \ref{phase-1-step-2}]
	Under the setting of Lemma \ref{lemma_1}, let  $$\eta\leq\frac{c_1\log^{-4}(MT_1/\delta)\min_{j\in[1:D]}(\upv_j^*)^2}{\calO(\sigma^2+(1+C)\mathcal{M}^2(\upb))}$$ and 
    $$
    T_1\geq\max\left\{\frac{\calO\left(\max_{j\in[1:D]}-\log(\upv_j^0)\right)}{c_1\eta\sigmin(D)},\frac{\calO\left([\sigma^2+\mathcal{M}^2(\upb)]\log^8(MT_1/\delta)\right)}{c_1^2\min_{j\in[1:D]}(\lambda_j(\upv_j^*)^4)}\right\}
    $$
    The combined event set satisfies $\mathbb{P}\left(\left(\bigcap_{j=1}^D\calE_{1,j}\right)\bigcup\calE_2\right) \geq 1 - \frac{\delta}{6}$. where $$\calE_{1,j} := \left\{\max_{t\leq T_1}\bar{\upv}_j^t\geq\frac{1-c_1/2}{1+c_1}\upb_j \right\},\quad \forall j\in[1:D],$$ and $\calE_2 := \left\{ \calA^c(\bar{\upv}^{T_1})\right\}$.
\end{lemma}
\begin{proof}
    For a fixed $j\in[1:D]$, we define 
    the subcoupling $\{\breve{\upv}^t\}_{t=0}^{T_1}$ with initialization $\breve{\upv}^0=\bar{\upv}^0$ as follows: 
    \begin{enumerate}
        \item Updating state: If 
        event $\calB_t(j)=\left\{\calA(\breve{\upv}^t)\bigwedge\breve{\upv}_j^t<\frac{1-c_1/2}{1+c_1}\upb_j\right\}$ holds, let $\breve{\upv}^{t+1}=\bar{\upv}^{t+1}$,
        \item Multiplicative scaling state: Otherwise, let $\breve{\upv}^{t+1}=\left(1+\frac{c_1(1-c_1)\eta}{2}\lambda_{j}(\upv_{j}^*)^2\right)\breve{\upv}^{t}$.
    \end{enumerate}
    We aim to demonstrate that $-t\log(1+\frac{c_1(1-c_1)\eta}{2}\lambda_{j}(\upv_{j}^*)^2)+\log(\breve{\upv}_j^t)$ is a submartingale. 
    If event $\calB_t^c(j)$ holds, we directly obtain $\bbE[\log(\breve{\upv}_j^{t+1}) \mid \calF^t] \geq \log(1+\frac{c_1(1-c_1)\eta}{2}\lambda_j(\upv_j^*)^2) + \log(\breve{\upv}_j^t)$. Otherwise, letting 
    $$
    w_j^t:=\hat{z}_j^t(\bar{\upv}^{t-1})-\hat{\zeta}_{M+1:\infty}^t-\hat{\xi}^t,\quad \forall t\in[1:T_1],
    $$ 
    we have
	\begin{align}
		\bbE\left[\log(\breve{\upv}_j^{t+1})\mid\calF^t\right]=&\bbE\left[\log(\bar{\upv}_j^{t+1})\mid\calF^t\right]\notag
		\\
		=&\bbE_{\upx_{1:M}^{t+1},\xi^{t+1},\zeta_{M+1:\infty}^{t+1}}\left[\log\left(1-\eta\left((\bar{\upv}_j^t)^2-(\upv_j^*)^2\right)(\hat{\upx}_j^{t+1})^2-\eta w_j^{t+1}\hat{\upx}_j^{t+1}\right)\right]\notag
        \\
        &+\log(\bar{\upv}_j^t)\notag
	\\
	\overset{\text{(a)}}       {\geq}&\log\left(1+\frac{3c_1(1-c_1)\eta}{4}\lambda_j(\upv_j^*)^2\right)\notag
        \\
        &-\eta^2\lambda_j\calO\left(\left[\sigma^2+(1+C)\mathcal{M}^2(\upb)\right]\log^4(MT_1/\delta)\right)+\log(\bar{\upv}_j^t)\notag
        \\
        \overset{\text{(b)}}{\geq}&\log\left(1+\frac{c_1(1-c_1)\eta}{2}\lambda_j(\upv_j^*)^2\right)+\log(\bar{\upv}_j^t)\notag
		\\
		\overset{\text{(c)}}{=}&\log\left(1+\frac{c_1(1-c_1)\eta}{2}\lambda_j(\upv_j^*)^2\right)+\log(\breve{\upv}_j^t),\notag
	\end{align} 
    where (a) is based on the following three facts: 1) the Taylor expansion of $\log(a+\cdot)$ with $a=1+\eta((\upv_j^*)^2-(\bar{\upv}_j^t)^2)\bbE[(\hat{\upx}_j^{t+1})^2]$; 2) the property that $Y_j^{t+1}$ is zero-mean and independent of $\hat{\upx}_j^{t+1}$; and 3) the step size $\eta\leq\frac{c_1(\upv_j^*)^2\log^{-4}(MT_1/\delta)}{\calO(\sigma^2+(1+C)\mathcal{M}^2(\upb))}$ ensures that $1-\tau\eta((\bar{\upv}_j^t)^2-(\upv_j^*)^2)[(\hat{\upx}_j^{t+1})^2-\bbE[(\hat{\upx}_j^{t+1})^2]]-\tau\eta w_j^{t+1}\hat{\upx}_j^{t+1}\geq1/2$ for any $\tau\in[0,1]$, 
    (b) is due to the inequality $\log(1+\frac{c_1(1-c_1)\eta}{16}\lambda_j(\upv_j^*)^2)\geq\eta^2\lambda_j\calO([\sigma^2+(1+C)\mathcal{M}^2(\upb)]$ $\log^4(MT_1/\delta))$, and (c) relies on the temporal exclusivity property that if event $\calB_t^c(j)$ occurs at time $t$, then $\calB_{t}(j)$ is permanently excluded for all subsequent times $t' > t$. Therefore, based on the submartingale, we obtain
    \begin{align}\label{lower-2}
        &\Pro\left\{\breve{\upv}_j^{T_1}<\frac{1-c_1/2}{1+c_1}\upb_j\right\}\notag
        \\
        \overset{\text{(d)}}{\leq}&\exp\left\{-\frac{2\left({T_1}\log\left(1+\frac{c_1(1-c_1)\eta}{2}\lambda_j(\upv_j^*)^2\right)+\log(v_j^0)-\log\left(\frac{1-c_1/2}{1+c_1}\upb_j\right)\right)^2}{{T_1}\eta^2\lambda_j\calO\left([\sigma^2+\mathcal{M}^2(\upb)]\log^6(MT_1/\delta)\right)}\right\}\notag
        \\
        \overset{\text{(e)}}{\leq}&\exp\left\{-\frac{{T_1}\log^2\left(1+\frac{c_1(1-c_1)\eta}{2}\lambda_j(\upv_j^*)^2\right)}{\eta^2\lambda_j\calO\left([\sigma^2+\mathcal{M}^2(\upb)]\log^6(MT_1/\delta)\right)}\right\}\notag
        \\
        \overset{\text{(f)}}{\leq}&\frac{\delta}{12N}
    \end{align}
    where (d) is derived from Azuma's inequality and the estimation of $\left|\log(\breve{\upv}_j^{t+1})-\log(\breve{\upv}_j^{t})\right|$ below:
    \begin{align}
        \left|\log(\breve{\upv}_j^{t+1})-\log(\breve{\upv}_j^{t})\right|\leq\eta\lambda_j^{1/2}\calO\left(\left[\sigma^2+\mathcal{M}^2(\upb)\right]^{1/2}\log^4(MT_1/\delta)\right),
    \end{align}
    which implies that
    \begin{align}
        \left|\log(\breve{\upv}_j^{t+1})-\log\left(1+\frac{c_1(1-c_1)\eta}{2}\lambda_{j}(\upv_{j}^*)^2\right)-\log(\breve{\upv}_j^{t})\right|^2\leq\eta^2\lambda_j\calO\left([\sigma^2+\mathcal{M}^2(\upb)]\log^8(MT_1/\delta)\right).\notag
    \end{align}
    Moreover, since ${T_1}\log(1+\frac{c_1(1-c_1)\eta}{2}\lambda_j(\upv_j^*)^2)/4\geq-\log(v_j^0)$ and $c_1^2{T_1}\lambda_j(\upv_j^*)^4\geq\calO([\sigma^2+\mathcal{M}^2(\upb)]$ $\log^8(MT_1/\delta))$, we obtain inequalities (e) and (f). 
    If $\calA(\bar{\upv}^{T_1})$ holds, Eq.~\eqref{lower-2} illustrates that $\bbP(\calE_{1,j}^c)\leq\frac{\delta}{12N}$. Thus, we have $\bbP(\calE_{1,j}^c\bigcap\calE_2^c)\leq\frac{\delta}{6N}$.
\end{proof}

Second, we construct the high-probability lower bound for $\bar{\mathbf{v}}_j^{T_1}$ for any $j\in[1:D]$ in Lemma \ref{lemma-3}. The proof technique of Lemma \ref{lemma-3} mirrors that of Lemma \ref{lemma_1}.
By contrast, we construct geometrically compensated supermartingale $\{-u_i^t\}_{t=1}^{T_1}$ for each $i\in\mathcal{S}$.
The proof is finished by applying Bernstein-type concentration inequalities to these constructed supermartingales, yielding the required probabilistic bounds.

\begin{lemma}\label{lemma-3}[Formal version of Lemma \ref{phase-1-step-3}]
    Under the setting of Lemma \ref{lemma_1}, let 
    $$
    \eta\leq\frac{c_1^2\Barsigmin(D)}{\calO([\sigma^2+\calM^2(\upb)]\log^4(MT_1/\delta))}.
    $$ 
    The combined event set satisfies $\mathbb{P}\left(\bigcap_{j=1}^D\left(\bigcup_{k={3,4}}\calE_{k,j}\right)\right) \geq 1 - \frac{\delta}{6}$, where 
    $$
    \calE_{3,j} := \left\{\max_{t\leq T_1}\bar{\mathbf{v}}_j^t<\frac{1-c_1/2}{1+c_1}\mathbf{b}_j \right\},\quad\calE_{4,j}:= \left\{\bar{\mathbf{v}}_j^{T_1}\geq\frac{1-c_1}{1+c_1}\mathbf{b}_j\right\},\quad \forall j\in[1:D].
    $$
\end{lemma}
\begin{proof}
For any $j \in [1:D]$, if $\mathcal{E}_{3,j}^c$ occurs, there exists $t \in [1:T_1]$ such that $\bar{\mathbf{v}}_j^t \geq \frac{1 - c_1/2}{1 + c_1}\mathbf{b}_j$. 
Define $\tau_{0,j}$ as the stopping time satisfying $\bar{\mathbf{v}}_j^{\tau_{0,j}} \geq \frac{1 - c_1/2}{1 + c_1}\mathbf{b}_j$ as:
\begin{equation}\nonumber
   \tau_{0,j}=\inf_{t}\left \{ t:\bar{\mathbf{v}}_j^{t}\ge \frac{1-c_1/2}{1+c_1}\mathbf{b}_j\right \}.
\end{equation}
We also define $\tau_{1,j}$ as the stopping time satisfying $\bar{\mathbf{v}}_j^{\tau_{1,j}}<\frac{1-c_1}{1+c_1}\mathbf{b}_j$ after $\tau_{0,j}$ as:
\begin{equation}\nonumber
     \tau_{1,j} =\inf_{t>\tau_{0,j}}\left \{ t:\bar{\mathbf{v}}_j^{t}<\frac{1-c_1}{1+c_1}\mathbf{b}_j\right \}.
\end{equation}
Based on the definition of $\{\bar{\mathbf{v}}^t\}_{t=0}^{T_1}$, once the event $\calA^c(\bar{\mathbf{v}}^t)$ occurs, the coupling process satisfies $\bar{\mathbf{v}}^{t'}=\bar{\mathbf{v}}^{t}$ for any $t'>t$. Therefore, $\mathcal{A}(\bar{\mathbf{v}}^t)$ holds for all $t \leq \tau_{1,j}$. Moreover, $\bar{\upv}_j^t$ must traverse in and out of the threshold interval $\left[\frac{1-c_1}{1+c_1}\upb_j,\frac{1}{1+c_1}\upb_j\right]$ before subceeding $\frac{1-c_1}{1+c_1}\upb_j$. We aim to estimate the following probability for coordinates $j\in[1:D]$ and time pairs $t_0<t_1\in[1:T_1]$:
\begin{equation}\nonumber
    \bbP\left(\bar{\mathcal{D} }_{\tau_0=t_0}^{\tau_1=t_1}\left ( j \right )=\left\{\bar{\mathbf{v}}_j^{t_0}\geq\frac{1-c_1/2}{1+c_1}\mathbf{b}_j\bigwedge\bar{\mathbf{v}}_j^{t_0:t_1-1}\in\left[\frac{1-c_1}{1+c_1}\mathbf{b}_j,\frac{1}{1+c_1}\mathbf{b}_j\right]\bigwedge\bar{\mathbf{v}}_j^{t_1}<\frac{1-c_1}{1+c_1}\mathbf{b}_j\right\}\right).
\end{equation}
For any $t\in[t_0:t_1-1]$, we have 
\begin{equation}\label{mar-concen-3}
    \begin{aligned}
        \bbE\left[\mathbf{v}_j^*-\bar{\mathbf{v}}_j^{t+1}\mid\mathcal{F}^t\right]=&\bbE_{{\mathbf{x}}_{1:M}^{t+1},{\xi}^{t+1},{\zeta}_{M+1:\infty}^{t+1}}\left[\mathbf{v}_j^*-\bar{\mathbf{v}}_j^{t}+\eta\left(\left((\bar{\upv}_j^t)^2-(\upv_j^*)^2\right)\hat{\upx}_j^{t+1}+\hat{\upz}_j^{t+1}(\bar{\upv}^t)\right.\right.
        \\
        &\, \, \, \, \, \, \, \, \, \, \, \, \, \, \, \, \, \, \, \, \, \, \, \, \, \, \, \, \, \, \, \, \, \, \, \, \, \, \, \, \, \, \, \, \, \left.\left.-\hat{\zeta}_{M+1:\infty}^{t+1}-\hat{\xi}^{t+1}\right)\hat{\mathbf{x}}_j^{t+1}\bar{\mathbf{v}}_j^{t+1}\right]
		\\
		\leq&\left(1-\frac{1-c_1}{(1+c_1)^2}\eta\lambda_{j}(\mathbf{v}_{j}^*)^2\right)(\mathbf{v}_j^*-\bar{\mathbf{v}}_j^{t}).
    \end{aligned}
\end{equation}
Applying Lemma \ref{aux-1} to $(((\bar{\upv}_j^t)^2-(\upv_j^*)^2)\hat{\upx}_j^{t+1}+\hat{\upz}_j^{t+1}(\bar{\upv}^t)-\hat{\zeta}_{M+1:\infty}^{t+1}-\hat{\xi}^{t+1})\hat{\mathbf{x}}_j^{t+1}$, we have
\begin{equation}
\small
    \begin{split}
        \bbE\left[\exp\left\{\lambda\left(\bbE[\bar{\mathbf{v}}_j^{t+1}\mid\calF^t]-\bar{\mathbf{v}}_j^{t+1}\right)\right\}\mid\calF^t\right]\leq \exp\left\{\frac{\lambda^2\eta^2\lambda_j(\upv_j^*)^2\calO\left(\left[\sigma^2+\mathcal{M}^2(\upb)\right]\log^4(MT_1/\delta)\right)}{2}\right\},\notag
    \end{split}
\end{equation}
for any $\lambda\in\bbR$. Therefore, combining Lemma \ref{aux-2} with Eq.~\eqref{mar-concen-3}, we establish the probability bound for event $\bar{\mathcal{D} }_{\tau_0=t_0}^{\tau_1=t_1}\left(j\right)$ with any time pair $t_0<t_1\in[1:T_1]$ as 
	\begin{align}\nonumber
		\Pro\left(\bar{\mathcal{D} }_{\tau_0=t_0}^{\tau_1=t_1}\left ( j \right )\right)\leq\exp\left\{\frac{-c_1^2(\mathbf{v}_j^*)^2}{\eta\calO\left(\left[\sigma^2+\calM^2(\upb)\right]\log^4(MT_1/\delta)\right)}\right\}.
	\end{align}
 Notice that the occurrence of  $\calE_{3,j}^c\bigwedge\calE_{4,j}^c$ implies $\bar{\mathcal{D} }_{\tau_0=t_0}^{\tau_1=t_1}\left ( j \right )$  must hold for certain $t_0<t_1\in[1:T_1]$. Therefore, we have
	\begin{align}
		\Pro\left(\calE_{3,j}^c\bigwedge\calE_{4,j}^c\right)\leq&\sum_{1\leq t_1<t_2\leq T_1}	\Pro\left(\bar{\mathcal{D} }_{\tau_0=t_0}^{\tau_1=t_1}\left ( j \right )\right)\notag
		\\
		\leq&\frac{T_1^2}{2}\exp\left\{\frac{-c_1^2(\mathbf{v}_j^*)^2}{\eta\calO\left(\left[\sigma^2+\calM^2(\upb)\right]\log^4(MT_1/\delta)\right)}\right\}\notag
		\\
		\leq&\frac{T_1^2}{2}\exp\left\{\frac{-c_1^2\min_{j\in[1:D]}(\mathbf{v}_{j}^*)^2}{\eta\calO\left(\left[\sigma^2+\calM^2(\upb)\right]\log^4(MT_1/\delta)\right)}\right\}\notag
		\\
		\leq&\frac{\delta}{6N}.
	\end{align}

\end{proof}

Combining Lemma \ref{lemma_1} in \textbf{Part I} and Lemma \ref{lemma-2}, Lemma \ref{lemma-3} in \textbf{Part II}, we have now completed the proof of Theorem \ref{theorem_1}.

\begin{proof}[Proof of Theorem \ref{theorem_1}]
    First, we notice that in the setting of Theorem \ref{theorem_1}, 
	\begin{align}\label{eta-choice}
		\eta\leq\frac{c_1^2\Barsigmin(D)}{\calO\left(\left[\sigma^2+\mathcal{M}^2(\upb)\right]\log^4(MT_1/\delta)\right)},
	\end{align} 
	and  
	\begin{align}\label{T-choice}
		\begin{cases}
			\frac{\calO\left([\sigma^2+\mathcal{M}^2(\upb)]\log^8(MT_1/\delta)-\min_{j\in[1:D]}\log(\upv_j^0)\right)}{c_1^2\eta\sigmin(D)}\leq T_1\leq\frac{\log^{-4}(MT_1/\delta)\log((M-D)T_1/\delta)}{\eta^2\Tildesigmax(D)\calO\left([\sigma^2+\mathcal{M}^2(\upb)]\right)}, & \text{ if }M>D,
			\\
			\frac{\calO\left([\sigma^2+\mathcal{M}^2(\upb)]\log^8(MT_1/\delta)-\min_{j\in[1:D]}\log(\upv_j^0)\right)}{c_1^2\eta\sigmin(D)}\leq T_1, & \text{ otherwise }.
		\end{cases}
	\end{align}
	 satisfy all assumptions in Lemmas \ref{lemma_1}-\ref{lemma-3}. Thus we can use all results in Lemma \ref{lemma_1}-\ref{lemma-3}. \ref{lemma_1} yields $\Pro\{\bar{\upv}^{T_1}>\upb\}\leq\frac{\delta}{6}$. Lemma \ref{lemma-2} implies that $\Pro\{\min_{j\in[1:D]}\max_{t\leq {T_1}}(\bar{\upv}_j^t-\frac{1-c_1/2}{1+c_1}\upb_j)<0\bigwedge\bar{\upv}^{T_1}\leq\upb\}\leq\frac{\delta}{6}$. Combining Lemma \ref{lemma_1} and \ref{lemma-2}, we have $\Pro\{\min_{j\in[1:D]}$ $\max_{t\leq {T_1}}(\bar{\upv}_j^t-\frac{1-c_1/2}{1+c_1}\upb_j)<0\}\leq\frac{\delta}{3}$. Moreover, Lemma \ref{lemma-3} indicates that $\Pro\{\min_{j\in[1:D]}$ $\max_{t\leq {T_1}}(\bar{\upv}_j^t-\frac{1-c_1/2}{1+c_1}\upb_j)\geq0\bigwedge\min_{j\in[1:D]}(\bar{\upv}_j^{T_1}-\frac{1-c_1}{1+c_1}\upb_j)<0\}\leq\frac{\delta}{6}$. Combining these results, we establish the final probability bound: $\Pro\{|\bar{\upv}_{1:D}^{T_1}-\upv_{1:D}^*|\leq\frac{c_1}{1+c_1}\upb_{1:D}\bigwedge\bar{\upv}_{D+1:M}^{T_1}\leq\upb_{D+1:M}\}\geq1-\frac{2}{3}\delta$, and this bound can be extended to $\upq^{T_1}$ by the definition of capped coupling process in Definition \ref{def:b-capped}. By Proposition \ref{p1}, we complete the proof.
\end{proof}

\subsection{Proof of Phase II}\label{phase-2}

In this section, we introduce the proof techniques of \textbf{Phase II} in Theorem \ref{phase-II-main}, where we construct the global convergence analysis of Algorithm \ref{SGD} for risk minimization. We demonstrate that after Phase I (i.e., $t>T_1$), the iterations of $\upv^t$ are confined within a neighborhood of $\upv_{1:M}^*$ with high probability. Therefore, the SGD dynamics for the quadratic model can be well approximated by the dynamics for the linear model with high probability. Therefore, we can extend the analytical techniques for SGD in the linear model to obtain the conclusion of Theorem \ref{phase-II-main}. 

Theorem \ref{theorem_1} illustrates that the output of Algorithm \ref{SGD} after $T_1$ iterations lies in the neighborhood of the ground truth within a constant factor, namely, $|\upv_{1:D}-\upv_{1:D}^*|\leq c_1\upv_{1:D}^*$. Thus, we use $\upv^{T_1}$, which satisfies Eq.~\eqref{thm-1-eq}, as the initial point for the SGD iterations in \textbf{Phase II}, and set the annealing learning rate to guarantee the output of Algorithm \ref{SGD} fully converges to $\calR_M(\upv_{1:M}^*)$. Before we formal propose Theorem \ref{phase-II-main}, we preliminarily introduce some of the coupling process, auxiliary function, and notations used for our statement of Theorem \ref{phase-II-main} and analysis in Phase II. We introduce the truncated coupling $\{\widehat{\upv}^t\}_{t=0}^{T_2}$ as follows:
\begin{align}
	\begin{cases}
		\widehat{\upv}^{t+1}=\upv^{T_1+t+1}, & \text{ if }\calG(\widehat{\upv}^t)\text{ occurs },
		\\
		\widehat{\upv}^{\tau+1}=\frac{13}{4}\upv_{1:M}^*, \quad\forall \tau\geq t, & \text{ otherwise },\notag
	\end{cases}
\end{align}
with initialization $\widehat{\upv}^0=\upv^{T_1}$ which satisfies Eq.~\eqref{thm-1-eq}, where event 
\begin{align}\label{event-G}
\calG(\upv):=\left\{\widehat{\upv}_j^t\in\left[\frac{1}{2}\upv_j^*,\frac{3}{2}\upv_j^*\right],\, \, \forall j\in[1:D]\bigwedge\widehat{\upv}_j^t\in\left[0,2\upv_j^*\right],\, \, \forall j\in[D+1:M]\right\},
\end{align}
for any $\upv\in\bbR^M$ and $T_2=T-T_1$. Moreover, we define the auxiliary function $\psi:\bbR^M\rightarrow\bbR^M$ as:
\begin{align}
	\psi(\upv)=\begin{cases}
		\upv, & \text{ if }\calG(\upv)\text{ occurs },
		\\
		\upv_{1:M}^*, & \text{ otherwise}.
	\end{cases}\notag
\end{align}
Thus we construct the truncated sequence  $\{\upw^t=\psi(\widehat{\upv}^{t})\}_{t=0}^{T_2}$.
In this phase, our analysis primarily focuses on the trajectory of $\upw^t$. Based on the generation mechanism of the sequence $\{\upw^t\}_{t=0}^{T_2}$, the update from $\upw^{t}$ to $\upw^{t+1}$ can be categorized into two cases:
Case I) $\upw^{t+1}$ remains updated, with its iteration closely approximating SGD updates in linear models \citep{wu2022last};
Case II) For any $\tau\geq t$, $\upw^{\tau+1}$ does not update and remains constant at $\upv_{1:M}^*$.

We also define some notations for simplifying the representation. For any $\upv,\upu\in\bbR^d$, we define $\upv\odot\upu=(\upv_1\upu_1,\cdots,\upv_d\upu_d)^{\top}$ and $\diag\{\upv\}=\diag\{\upv_1,\cdots,\upv_d\}\in\bbR^{d\times d}$. Let $\upH=\frac{25}{4}\upLambda\diag\{\widehat{\upb}\odot\widehat{\upb}\}$ with $\widehat{\upb}^{\top}=\left((\upv_{1:D}^*)^{\top},(\widehat{\upv}_{D+1:M}^*)^{\top}\right)$ and $\hat{\upv}^*$ satisfies $\hat{\upv}_j^*=\max\left\{\frac{3}{2}\upv_j^*,3\upv_j^0\right\}$. We also denote $\upH_{\upw}^t=(\upw^t\odot\Pi_M\upx^{t+1})\otimes((\upw^t+\upv_{1:M}^*)\odot\Pi_M\upx^{t+1})$ and $\upR_{\upw}^t=(\xi^{t+1}+\zeta_{M+1:\infty}^{t+1})\diag\{\upw^t\}$ for simplicity.
We denote the following linear operators that will be used in the proof:
\begin{align}
	\calI:=\upI\otimes\upI,\quad &\calH_{\upw}^t:=\bbE_t\left[\upH_{\upw}^t\otimes(\upH_{\upw}^t)^{\top}\right],\quad \widetilde{\calH}_{\upw}^t:=\bbE_t\left[\upH_{\upw}^t\right]\otimes\bbE_t\left[\upH_{\upw}^t\right],\notag
	\\
	\calG_{\upw}^t:=\bbE_t\left[\upH_{\upw}^t\right]\otimes\upI+\upI\otimes&\bbE_t\left[\upH_{\upw}^t\right]-\eta_t\calH_{\upw}^t,\quad \widetilde{\calG}_{\upw}^t:=\bbE_t\left[\upH_{\upw}^t\right]\otimes\upI+\upI\otimes\bbE_t\left[\upH_{\upw}^t\right]-\eta_t\widetilde{\calH}_{\upw}^t,\notag
\end{align}
where $\bbE_t[\cdot]=\bbE[\cdot\mid\calF^t]$. For any operator $\calA$, we use $\calA\circ\upA$ to denote $\calA$ acting on a symmetric matrix $\upA$. It's easy to directly verify the following rules for above operators acting on a symmetric matrix $\upA$:
\begin{align}
	\calI\circ\upA=\upA,\quad \calH_{\upw}^t\circ\upA=&\bbE_t\left[\upH_{\upw}^t\upA(\upH_{\upw}^t)^{\top}\right],\quad\widetilde{\calH}_{\upw}^t\circ\upA=\bbE_t\left[\upH_{\upw}^t\right]\upA\bbE_t\left[\upH_{\upw}^t\right],\notag
	\\
	\left(\calI-\eta_t\calG_{\upw}^t\right)\circ\upA=&\bbE_t\left[\left(\upI-\eta_t\upH_{\upw}^t\right)\upA\left(\upI-\eta_t\upH_{\upw}^t\right)\right],\notag
	\\
	\left(\calI-\eta_t\widetilde{\calG}_{\upw}^t\right)\circ\upA=&\left(\upI-\eta_t\bbE_t\left[\upH_{\upw}^t\right]\right)\upA\left(\upI-\eta_t\bbE_t\left[\upH_{\upw}^t\right]\right).\notag
\end{align}
The following is the formalized expression of the iteration process for $\upw^t$. For all $t\in[0:T_2-1]$, if $\upw^{t+1}=\upv^{T_1+t+1}$ (i.e., event $\calG(\upv^{T_1+t+1})$ occurs), $\upw^{t+1}$ follows the update rule as:
\begin{align}\label{iteration}
	\upw^{t+1}-\upv_{1:M}^*=\upw^t-\upv_{1:M}^*-\eta_t\upH_{\upw}^t\left(\upw^t-\upv_{1:M}^*\right)+\eta_t\upR_{\upw}^t\Pi_M\upx^t.
\end{align}
Otherwise, we have 
$$
\upw^{\tau+1}=\upv_{1:M}^*,\quad \forall \tau\geq t.
$$  
Since $\upw^{t+1}$ = $\upv^{T_1+t+1}$ implies $\upw^t = \upv^{T_1+t}$, but the converse does not necessarily hold, we derive the recurrence process as:
\begin{align}
	&\bbE\left[\left(\upw^{t+1}-\upv_{1:M}^*\right)^{\otimes2}\right]\preceq\bbE\left[\left(\upw^t-\upv_{1:M}^*-\eta_t\upH_{\upw}^t\left(\upw^t-\upv_{1:M}^*\right)+\eta_t\upR_{\upw}^t\Pi_M\upx^t\right)^{\otimes2}\mathds{1}_{\upw^t=\upv^{T_1+t}}\right].\notag
\end{align}
Define $\widehat{\upw}^t:=\upw^t-\upv_{1:M}^*$. The iterative update of $\widehat{\upw}^t$ can be decomposed into two random processes,
\begin{align}
	\widehat{\upw}^{t}=\mathds{1}_{\upw^t=\upv^{T_1+t}}\cdot\widehat{\upw}_{\bi}^t+\mathds{1}_{\upw^t=\upv^{T_1+t}}\cdot\widehat{\upw}_{\var}^t,\quad \forall t\in[0:T_2],
\end{align}
where $\{\widehat{\upw}_{\var}^t\}_{t=1}^{T_2}$ is recursively defined by
\begin{align}
	\begin{cases}
		\widehat{\upw}_{\var}^{t+1}=\left(\upI-\eta_t\upH_{\upw}^t\right)\widehat{\upw}_{\var}^t+\eta_t\upR_{\upw}^t\Pi_M\upx^t,  & \text{ if }\upw^t=\upv^{T_1+t},
		\\\widehat{\upw}_{\var}^{t+1}=\mathbf{0}, & \text{ otherwise },
	\end{cases}\notag
\end{align}
for any $t\in[0:T_2-1]$ with $\widehat{\upw}_{\var}^0=\textbf{0}$ and $\{\widehat{\upw}_{\bi}^t\}_{t=1}^{T_2}$ is recursively defined by
\begin{align}
	\begin{cases}
		\widehat{\upw}_{\bi}^{t+1}=\left(\upI-\eta_t\upH_{\upw}^t\right)\widehat{\upw}_{\bi}^t, & \text{ if }\upw^t=\upv^{T_1+t},
		\\
		\widehat{\upw}_{\bi}^{t+1}=\mathbf{0}, & \text{ otherwise },\notag
	\end{cases}
\end{align}
for any $t\in[0:T_2-1]$ with  $\widehat{\upw}_{\bi}^0=\upw^0-\upv_{1:M}^*$. We define the $t$-th step bias iteration as $\upB^t=\bbE\left[\widehat{\upw}_{\bi}^t\otimes\widehat{\upw}_{\bi}^t\right]$ and $t$-th step variance iteration as $\upV^t=\bbE\left[\widehat{\upw}_{\var}^t\otimes\widehat{\upw}_{\var}^t\right]$.
Therefore, we can derive the following relations for 
$\{\upB^t\}_{t=0}^{T_2}$ and $\{\upV^t\}_{t=0}^{T_2}$:
\begin{align}
	\begin{cases}\label{bias-and-variance}
        \upB^{t+1}&\preceq\bbE\left[\left(\calI-\eta_t\calG_{\upw}^t\right)\circ\left(\widehat{\upw}_{\bi}^t\otimes\widehat{\upw}_{\bi}^t\right)\right],
		\\
		\upV^{t+1}&\preceq\bbE\left[\left(\calI-\eta_t\calG_{\upw}^t\right)\circ\left(\widehat{\upw}_{\var}^t\otimes\widehat{\upw}_{\var}^t\right)\right]+\eta_t^2\upSigma_{\upw}^t,
	\end{cases}\quad \forall t\in[0:T_2-1],
\end{align}
with $\upB^0=\left(\upw^0-\upv_{1:M}^*\right)\left(\upw^0-\upv_{1:M}^*\right)^{\top}$ and $\upV^0=\mathbf{0}$, where $\upSigma_{\upw}^t=\sigma^2\upLambda\bbE\left[\diag\{\upw^t\odot\upw^t\}\right]$. 

We formally propose Theorem \ref{phase-II-main} as below.
\begin{theorem}\label{phase-II-main}[Formal version of Theorem \ref{phase-II-informal}]
    Suppose Assumption \ref{ass-d} and \ref{ass-ss} hold, and let $T_1=\lceil(T-h)/\log(T-h)\rceil$ and $h=\lceil T/\log(T)\rceil$.
    Under the following setting
    \begin{enumerate}
        \item There exists $D<M$ such that $\eta_0\leq\widetilde{\Omega}(\min\{\tr^{-1}(\upH),\Barsigmin(D)\})$ and $T_1=\widetilde{\calO}(\frac{\sigma^2+\calM^2(\widehat{\upb})}{\eta_0\sigmin(D)})$,
        \item Let $D=M$, $\eta_0\leq\widetilde{\Omega}(\min\{\tr^{-1}(\upH),\Barsigmin(M)\})$, and $T_1\geq\widetilde{\calO}(\frac{\sigma^2+\calM^2(\widehat{\upb})}{\eta_0\sigmin(M)})$,
    \end{enumerate}
    we have
    \begin{align}\label{phaseII-p1}
        \bbE\left[\calR_M(\upw^{T_2})-\calR_M(\upv_{1:M}^*)\right]\lesssim&\sigma^2\left(\frac{N_0'}{K}+\eta_0\sum_{i=N_0'+1}^{N_0}\lambda_i(\upv_i^*)^2\right)\notag
        \\
        &+\sigma^2\eta_0^2(h+T_1)\sum_{i=N_0+1}^M\lambda_i^2(\widehat{\upb}_i^*)^4\notag
        \\
        &+\llangle\frac{1}{\eta_0 T_1}\upI_{1:N_1}+\upH_{N_1+1:M},\left(\upI-\eta_0\widehat{\upH}\right)^{2h}\upB^0\rrangle\notag
        \\
        &+\Gamma(\upH)\llangle\frac{1}{\eta_0h}\upI_{1:N_1'}+\upH_{N_1'+1:M},\upB^0\rrangle,
    \end{align}
    for arbitrary $D\geq N_0\geq N_0'\geq 0$ and $D\geq N_1\geq N_1'\geq 0$, where $\Gamma(\upH):=(\frac{625N_1'}{T_1}+\frac{25\eta_0h}{T_1}\tr(\upH_{N_1'+1:N_1})+\eta_0^2h\tr(\upH_{N_1+1:M}^2))$ and $T_2=T-T_1$. Specially, we have 
    \begin{equation}\label{phaseII-p2}
        \begin{split}
        \calR_M(\upv^{T})-\calR_M(\upv_{1:M}^*)\lesssim&\frac{\sigma^2N}{T_1}+\sigma^2\eta_0^2(h+T_1)\sum_{i=D+1}^M\lambda_i^2(\widehat{\upb}_i^*)^4
        \\
        &+\llangle\frac{1}{\eta_0 T_1}\upI_{1:D}+\upH_{D+1:M},\left(\upI-\eta_0\widehat{\upH}\right)^{2h}\upB^0\rrangle
        \\
        &+\left(\frac{D}{T_1}+\eta_0^2h\tr(\upH_{D+1:M}^2)\right)\llangle\frac{1}{\eta_0h}\upI_{1:D}+\upH_{D+1:M},\upB^0\rrangle,
        \end{split}
    \end{equation}
    with probability at least 0.95.
\end{theorem}

Before the beginning of our proof, we define the $(c\mathbf{v}^*_{1:D},\mathbf{b})$-neighbor coupling process which will be used in the following lemma as below.

\begin{definition}\label{def-neighbor-couple}[$(c\mathbf{v}^*_{1:D},\mathbf{b})$-neighbor coupling]
Let $\{\mathbf{q}^t\}_{t=0}^T$ be a Markov chain in $\mathbb{R}_+^M$ adapted to filtration $\{\mathcal{F} ^{t}\}_{t=0}^T$.  Given parameters: 1) Dimension index $D\in \mathbb{Z}_{+}$; 2) Tolerance $c>0$; 3) Threshold vector $\mathbf{b}\in\bbR_+^{M-D}$. With initial condition $\bar{\mathbf{v}}^0=\upq^0$, $\left | \bar{\mathbf{v}}_{1:D}^0-\mathbf{v}^*_{1:D} \right | \le c\mathbf{v}^*_{1:D}$ and  $\mathbf{0}\le\bar{\mathbf{v}}_{D+1:M}^0\le \mathbf{b}$, the $(c\mathbf{v}^*_{1:D},\mathbf{b})$-neighbor coupling
process $\{\bar{\mathbf{v}}^t\}_{t=0}^T$ evolves as:
\begin{enumerate}
    \item Updating state: If $\left | \bar{\mathbf{v}}^t_{1:D}-\mathbf{v}^*_{1:D} \right | \le c\mathbf{v}^*_{1:D}$ and $\mathbf{0}\le \bar{\mathbf{v}}^t_{D+1:M}\le \mathbf{b}$, let $\bar{\mathbf{v}}^{t+1}=\mathbf{v}^{t+1}$,
    \item Absorbing state: Otherwise, maintain $\bar{\mathbf{v}}^{t+1}=\bar{\mathbf{v}}^{t}$.
\end{enumerate}
\end{definition}

\subsubsection{Part I: Bound the Output of Phase I}
In this part, we demonstrate that the output of \textbf{Phase I} remains confirmed within the neighborhood of the ground truth with high probability in Lemma \ref{high-probability-phase-II}. Specifically, by constructing similar supermartingales to that in the proofs of Lemma \ref{lemma_1} and Lemma \ref{lemma-3}, we obtain a set of compressed supermartingales dependent on the coordinate $i\in[1:M]$. Combining the compression properties of these supermartingales with the sub-Gaussian property of their difference sequences, through concentration inequality, we obtain Lemma \ref{high-probability-phase-II} as below.

\begin{theorem}\label{high-probability-phase-II}[Formal version of Theorem \ref{phase-2-step-1}]
Under Assumption \ref{ass-d}, we consider the $T_1$-th step of Algorithm \ref{SGD} and its subsequent iterative process. Let $D \in \mathbb{N}+$ represent the effective dimension. Define $\eta_0 \leq \widetilde{\Omega}\left(\frac{\Barsigmin(\max\{D,M\})}{\sigma^2 + \calM^2(\upb)}\right)$, and let $\{\widetilde{\upv}^t\}_{t=0}^{T_2}$ be an $(1/2,2)$-$\upv^*$ neighbor coupling process based on the control sequence $\{\upq^{T_1+t}\}_{t=0}^{T_2}$. Recall the definition \eqref{event-G} of event $\calG(\upv)$ for any $\upv\in\bbR^M$. If $D < M$, set the iteration number $T_2 \in \left[1: \widetilde{\Omega}\left(\frac{\Tildesigmax^{-1}(D)}{\eta_0^2 [\sigma^2 + \calM^2(\upb)]}\right)\right]$. Otherwise, set $T_2$ be an arbitrary positive integer. Then,   $\bigcap_{t=0}^{T_2}\calG(\upv^{T_1+t})$ holds with probability at least $1 - \delta$.
\end{theorem}
\begin{proof}
	Setting $c_1=\frac{1}{4}$ in Theorem \ref{theorem_1}, we have  $|\upq_{1:D}^{T_1}-\upv_{1:D}^*|\leq\frac{1}{4}\upv_{1:D}^*$ and $\boldsymbol{0}_{D+1:M}\leq\upq_{D+1:M}^{T_1}\leq\frac{3}{2}\upv_{D+1:M}^*$ with probability at least $1-\delta/6$. Without loss of generality, we assume $\upq^{T_1}$ satisfies $|\upq_{1:D}^{T_1}-\upv_{1:D}^*|\leq\frac{1}{4}\upv_{1:D}^*$ and $\boldsymbol{0}_{D+1:M}\leq\upq_{D+1:M}^{T_1}\leq\frac{3}{2}\upv_{D+1:M}^*$. Let $\hat{\tau}$ be the stopping time satisfying $\calG^c(\widetilde{\upv}^{\hat{\tau}})$, i.e.,
    \begin{align}
        \hat{\tau}=\inf\left\{t:\exists j\in[1:D], \text{ s.t. }\left|\widetilde{\upv}_j^t-\upv_j^*\right|>\frac{1}{2}\upv_j^*\text{ or }\exists j\in[D+1:M], \text{ s.t. }\widetilde{\upv}_j^t>2\upv_j^*\right\},\notag
    \end{align}
    For each coordinate $j\in[1:D]$, let $\hat{\tau}_{[1:D],j}^u$ and $\hat{\tau}_{[1:D],j}^l$ be the stopping time satisfying $\widetilde{\upv}_j^{\hat{\tau}_{[1:D],j}^u}>\frac{3}{2}\upv_j^*$ and $\widetilde{\upv}_j^{\hat{\tau}_{[D+1:M],j}^l}<\frac{1}{2}\upv_j^*$, respectively, i.e.,
    $$
    \hat{\tau}_{[1:D],j}^u=\inf\left\{t: \widetilde{\upv}_j^t>\frac{3}{2}\upv_j^*\right\},\quad \hat{\tau}_{[1:D],j}^l=\inf\left\{t: \widetilde{\upv}_j^t<\frac{1}{2}\upv_j^*\right\}.
    $$
    For each coordinate $j\in[D+1:M]$, let $\hat{\tau}_{[D+1:M],j}$ be the stopping time satisfying $\widetilde{\upv}_j^{\hat{\tau}_{[D+1:M],j}}>2\upv_j^*$, i.e.,
    $$
        \hat{\tau}_{[D+1:M],j}=\inf\left\{t: \widetilde{\upv}_j^t>2\upv_j^*\right\}.
    $$
    Based on Defnition \ref{def-neighbor-couple}, once the stopping time $\hat{\tau}=t_2$ occurs for certain $t_2\in[1:T_2]$, the coupling process satisfies $\widetilde{\upv}^t=\widetilde{\upv}^{t_2}$ for all $t>t_2$. Suppose there exists a certain $j\in[1:D]$ such that $\hat{\tau}_{[1:D],j}^u=t_2$. Thus, the event $\calG(\widetilde{\upv}^t)$ holds for all $t\in[0:t_2-1]$. Similar to the proof of Lemma \ref{lemma_1} and \ref{lemma-3}, $\widetilde{\upv}_j^t$ must traverse in and out of the threshold interval $[\upv_j^*,\frac{3}{2}\upv_j^*]$ before exceeding $\frac{3}{2}\upv_j^*$. We aim to estimate the following probability for coordinates $j\in[1:D]$ and time pairs $t_1<t_2\in[0:T_2]$ as:
    \begin{align}
        \bbP\left(\calB_{t_1}^{\hat{\tau}_{[1:D],j}^u=t_2}(j)=\left\{\widetilde{\upv}_j^{t_1}\leq\frac{5}{4}\upv_j^*\bigwedge\widetilde{\upv}_j^{t_1:t_2-1}\in\left[\upv_j^*,\frac{3}{2}\upv_j^*\right]\right\}\right).\notag
    \end{align}
    For any $t\in[t_1:t_2-1]$, we have
    \begin{equation}\label{mar-concen-4}
        \begin{split}
		\bbE\left[\widetilde{\upv}_j^{t+1}-\upv_j^*\mid\calF^t\right]=&\bbE_{\upx_{1:M}^{t+1},\xi^{t+1},\zeta_{M+1:\infty}^{t+1}}\left[\widetilde{\upv}_j^t-\upv_j^*-\eta\left(\left((\widetilde{\upv}_j^t)^2-(\upv_j^*)^2\right)\hat{\upx}_j^{t+1}+\hat{\upz}_j^{t+1}(\widetilde{\upv}^t)\right.\right.
        \\
        &\, \, \, \, \, \, \, \, \, \, \, \, \, \, \, \, \, \, \, \, \, \, \, \, \, \, \, \, \, \, \, \, \, \, \, \, \, \, \, \, \, \, \, \, \, \left.\left.-\hat{\zeta}_{M+1:\infty}^{t+1}-\hat{\xi}^{t+1}\right)\hat{\mathbf{x}}_j^{t+1}\widetilde{\mathbf{v}}_j^{t+1}\right]
		\\
		\leq&\left(1-\frac{3\eta_t}{8}\lambda_j(\upv_j^*)^2\right)(\widetilde{\upv}_j^t-\upv_j^*).
        \end{split}
	\end{equation}
	Applying Lemma \ref{aux-1} to $(((\widetilde{\upv}_j^t)^2-(\upv_j^*)^2)\hat{\upx}_j^{t+1}+\hat{\upz}_j^{t+1}(\widetilde{\upv}^t)-\hat{\zeta}_{M+1:\infty}^{t+1}-\hat{\xi}^{t+1})\hat{\mathbf{x}}_j^{t+1}\widetilde{\mathbf{v}}_j^{t+1}$, we obtain
	\begin{equation}
        \small
        \begin{split}
            \bbE\left[\exp\left\{\lambda\left(\widetilde{\upv}_j^{t+1}-\bbE\left[\widetilde{\upv}_j^{t+1}\mid\calF^t\right]\right)\right\}\mid\calF^t\right]\leq\exp\left\{\frac{\lambda^2\eta_t^2\lambda_j(\upv_j^*)^2\calO\left(\left[\sigma^2+\mathcal{M}^2(\upb)\right]\log^4(MT_2/\delta)\right)}{2}\right\},\notag
        \end{split}
	\end{equation}
	for any $\lambda\in\bbR$. Therefore, based on Lemma \ref{aux-2} and Eq.~\eqref{mar-concen-4}, we establish the probability bound for event $\calB_{t_1}^{\hat{\tau}_{[1:D],j}^u=t_2}(j)$ for any time pair $t_1<t_2\in[0:T_2]$ as:
	\begin{align}\label{phase_II_in_truncation}
		\Pro\left\{\calB_{t_1}^{\hat{\tau}_{[1:D],j}^u=t_2}(j)\right\}\leq\exp\left\{-\frac{(\upv_{j}^*)^2}{V_j}\right\},
	\end{align}
	where $V_j$ is denoted as
	\begin{align}
		V_j=\lambda_j(\upv_j^*)^2\calO\left([\sigma^2+\mathcal{M}^2(\upb)]\log^4(MT_2/\delta)\right)\sum_{t=0}^{T_2-1}\left(\prod_{i=t+1}^{T_2-1}(1-\frac{3\eta_{i}}{4}\lambda_j(\upv_j^*)^2)^{2}\right)(\eta_t)^2.\notag
	\end{align}
	By Lemma \ref{aux-3}, we have $V_j\leq\calO(\eta_0[\sigma^2+\mathcal{M}^2(\upb)]\log^4(MT_2/\delta))$. Therefore, using Eq.~\eqref{phase_II_in_truncation}, we can derive 
    \begin{align}\label{prob-1}
    \Pro\left\{\calB_{t_1}^{\hat{\tau}_{[1:D],j}^u=t_2}(j)\right\}\leq\exp\left\{-\frac{(\upv_j^*)^2}{\eta_0\calO\left([\sigma^2+\mathcal{M}^2(\upb)]\log^4(MT_2/\delta)\right)}\right\}.
    \end{align}
    
    Similarly, suppose there exists a certain $j\in[1:D]$ such that $\hat{\tau}_{[1:D],j}^l=t_2$. Thus, the event $\calG(\widetilde{\upv}^t)$ holds for all $t\in[0:t_2-1]$. $\widetilde{\upv}_j^t$ must traverse in and out of the threshold interval $[\frac{1}{2}\upv_j^*,\upv_j^*]$ before subceeding $\frac{1}{2}\upv_j^*$. We aim to estimate the following probability for coordinates $j\in[1:D]$ and time pairs $t_1<t_2\in[1:T_2]$: 
    \begin{align}
        \bbP\left(\calC_{t_1}^{\hat{\tau}_{[1:D],j}^l=t_2}(j)=\left\{\widetilde{\upv}_j^{t_1}\geq\frac{3}{4}\upv_j^*\bigwedge\widetilde{\upv}_j^{t_1:t_2-1}\in\left[\frac{1}{2}\upv_j^*,\upv_j^*\right]\right\}\right).\notag
    \end{align}
    For any $t\in[t_1:t_2-1]$, we have
    \begin{align}
        \bbE\left[\upv_j^*-\widetilde{\upv}_j^{t+1}\mid\calF^t\right]\leq\left(1-\frac{3\eta_t}{8}\lambda_j(\upv_j^*)^2\right)(\upv_j^*-\widetilde{\upv}_j^t).\notag
    \end{align}
    Based on Lemmas \ref{aux-1}, \ref{aux-2}, and \ref{aux-3} sequentially,  we obtain the probability bound for event $\calC_{t_1}^{\hat{\tau}_{[1:D],j}^l=t_2}(j)$ for any time pair $t_1<t_2\in[0:T_2]$ as:
    \begin{align}\label{prob-2}
		\Pro\left\{\calC_{t_1}^{\hat{\tau}_{[1:D],j}^l=t_2}(j)\right\}\leq\exp\left\{-\frac{(\upv_j^*)^2}{V_j}\right\}\leq\exp\left\{-\frac{(\upv_j^*)^2}{\eta_0\calO\left([\sigma^2+\mathcal{M}^2(\upb)]\log^4(MT_2/\delta)\right)}\right\}.
	\end{align}
	
    For the third stopping time, we also suppose there exists a certain $j\in[D+1:M]$ such that $\hat{\tau}_{[D+1:M],j}=t_2$. Thus, the event $\calG(\widetilde{\upv}^t)$ holds for all $t\in[0:t_2-1]$. Similarly, $\widetilde{\upv}_j^t$ must traverse in and out of the threshold interval $[\upv_j^*,2\upv_j^*]$ before exceeding $2\upv_j^*$. We aim to estimate the following probability for coordinates $j\in[D+1:M]$ and time pairs $t_1<t_2\in[0:T_2]$ as:
    \begin{align}
        \bbP\left(\calD_{t_1}^{\hat{\tau}_{[D+1:M],j}=t_2}(j)=\left\{\widetilde{\upv}_j^{t_1}\leq\frac{3}{2}\upv_j^*\bigwedge\widetilde{\upv}_j^{t_1:t_2-1}\in[\upv_j^*,2\upv_j^*]\right\}\right).\notag
    \end{align}
    For any $t\in[t_1:t_2-1]$, we have
	\begin{align}\label{mar-concen-5}
		\bbE\left[\widetilde{\upv}_j^{t+1}-\upv_j^*\mid\calF^t\right]\leq\widetilde{\upv}_j^{t}-\upv_j^*.
	\end{align}
	Applying Lemma \ref{aux-1} to $(((\widetilde{\upv}_j^t)^2-(\upv_j^*)^2)\hat{\upx}_j^{t+1}+\hat{\upz}_j^{t+1}(\widetilde{\upv}^t)-\hat{\zeta}_{M+1:\infty}^{t+1}-\hat{\xi}^{t+1})\hat{\mathbf{x}}_j^{t+1}\widetilde{\mathbf{v}}_j^{t+1}$, we obtain 
    \begin{equation}
    \small
    \begin{split}
        \bbE\left[\exp\left\{\lambda\left(\widetilde{\upv}_j^{t+1}-\bbE\left[\widetilde{\upv}_j^{t+1}\mid\calF^t\right]\right)\right\}\mid\calF^t\right]\leq\exp\left\{\frac{\lambda^2\eta_t^2\lambda_j(\upv_j^*)^2\calO\left(\left[\sigma^2+\mathcal{M}^2(\upb)\right]\log^4(MT_2/\delta)\right)}{2}\right\},\notag
    \end{split}
    \end{equation}
    for any $\lambda\in\bbR$.
    Based on Lemma \ref{aux-2} and Eq.~\eqref{mar-concen-5}, we establish the probability bound for the event $\calD_{t_1}^{\hat{\tau}_{[D+1:M],j}=t_2}(j)$ for any time pair $t_1<t_2\in[0:T_2]$ as:
	\begin{align}\label{prob-3}
		\Pro\left\{\calD_{t_1}^{\hat{\tau}_{[D+1:M],j}=t_2}(j)\right\}\leq\exp\left\{-\frac{(\upv_j^*)^2}{V_j}\right\}\overset{\text{(a)}}{\leq}\exp\left\{-\frac{\log^{-4}(MT_2/\delta)}{T_2\eta_0^2\lambda_j\calO\left(\left[\sigma^2+\mathcal{M}^2(\upb)\right]\right)}\right\},
	\end{align}
	where (a) is derived from $V_j\leq T_2\eta_0^2\lambda_j\calO\left(\left[\sigma^2+\mathcal{M}^2(\upb)\right]\log^4(MT_2/\delta)\right)$.
	
	Then, it is easy to notice that $\calG^c(\widetilde{\upv}^{T_2})$ indicates that one of the following situation happens: 
    \begin{enumerate}
        \item For a certain coordinate $j\in[1:D]$ and time pairs $t_1<t_2\in[0:T_2]$, either $\calB_{t_1}^{\hat{\tau}_{[1:D],j}^u=t_2}(j)$ or $\calC_{t_1}^{\hat{\tau}_{[1:D],j}^l=t_2}(j)$ occurs,
        \item For a certain coordinate $j\in[1:D]$ and time pairs $t_1<t_2\in[0:T_2]$, $\calD_{t_1}^{\hat{\tau}_{[D+1:M],j}=t_2}(j)$ occurs.
    \end{enumerate}
    Therefore, by the setting of $\eta_0$ in Lemma \ref{high-probability-phase-II}, we derive the following probability bound of event $\calG^c(\widetilde{\upv}^{T_2})$:
	\begin{align}
		\Pro\{\calG^c(\widetilde{\upv}^{T_2})\}\leq&\sum_{t_1<t_2}\left[\sum_{j\in[1:D]}\left(\Pro\left\{\calB_{t_1}^{\hat{\tau}_{[1:D],j}^u=t_2}(j)\right\}+\Pro\left\{\calC_{t_1}^{\hat{\tau}_{[1:D],j}^l=t_2}(j)\right\}\right)\right.\notag
		\\
		&\, \, \, \, \, \, \, \, \, \, \, \, \, \, \left.+\sum_{j\in[D+1:M]}\Pro\left\{\calD_{t_1}^{\hat{\tau}_{[D+1:M],j}=t_2}(j)\right\}\right]
		\notag
		\\
		\leq&2T_2^2N\exp\left\{-\frac{\min_{j\in\calN}(\upv_j^*)^2}{\eta_0\calO\left([\sigma^2+\mathcal{M}^2(\upb)]\log^4(MT_2)\right)}\right\}
		\notag
		\\	
		&+T_2^2(\max\{M,D\}-D)\exp\left\{-\frac{\log^{-4}(MT_2/\delta)}{T_2\eta_0^2\max_{j\in\bar{\calN}}\lambda_j\calO\left(\left[\sigma^2+\mathcal{M}^2(\upb)\right]\right)}\right\}\notag
		\\
		\leq&\delta/2.\notag
	\end{align}
	 According to the construction of the coupling process $\{\widetilde{\upv}^t\}_{t=0}^{T_2}$ in Definition \ref{def-neighbor-couple}, we have $\bigcap_{t=T_1}^{T_1+T_2}\calG(\upq^t)$ holds with probability at least $1-\delta/2$. By Proposition \ref{p1}, the proof is completed. 
\end{proof}

\subsubsection{Part II: Linear Approximation of the dynamic}
In part I, we have proved that $\bigcap_{t=0}^{T_2}\calG(\upv^{T_1+t})$ occurs with high probability, which implies the truncated sequence $\{\upw^t\}_{t=1}^{T_2}$ aligned to $\{\upv^{T_1+t}\}_{t=1}^{T_2}$ with high probability. Then we approximate the update process of $\{\upw^t\}_{t=1}^{T_2}$ to SGD in traditional linear regression, with respective bounds of variance term and bias term. 

We estimate the risk between the last-step function value and the ground truth as:
\begin{align}
	\bbE\left[\calR_M(\upw^{T_2})-\calR_M(\upv_{1:M}^*)\right]\overset{\text{(a)}}{\leq}\llangle\upH,\bbE\left[\widehat{\upw}^{T_2}\otimes\widehat{\upw}^{T_2}\right]\rrangle\leq2\llangle\upH,\upB^{T_2}\rrangle+2\llangle\upH,\upV^{T_2}\rrangle,
\end{align}
where $\upH=\frac{25}{4}\upLambda\diag\{\widehat{\upb}\odot\widehat{\upb}\}$ and $\widehat{\upb}^{\top}=\left((\upv_{1:D}^*)^{\top},(\widehat{\upv}_{D+1:M}^*)^{\top}\right)$. Here, (a) is derived from combining
\begin{align}
    \bbE\left[\calR_M(\upw^{T_2})-\calR_M(\upv_{1:M}^*)\right]=\bbE\left[\sum_{i=1}^M\lambda_i(\upw_i^{T_2}+\upv_i^*)^2(\upw_i^{T_2}-\upv_i^*)^2\right]\notag,
\end{align}
with the uniform boundedness of $\upw^t$ over $t\in[0:T_2]$. According to the definitions of $\upw^t$ and $\upH_{\upw}^t$, we have $\bbE[\upH_{\upw}^t]\preceq\upH$. Use $\widehat{\upH}$ to denote $\frac{1}{4}\upLambda\diag\{\overline{\upb}\odot\overline{\upb}\}$ where $\overline{\upb}^{\top}=\left((\upv_{1:D}^*)^{\top},\mathbf{0}^{\top}\right)$, and define $\widehat{\calG}:=\widehat{\upH}\otimes\upI+\upI\otimes\widehat{\upH}-\eta\widehat{\upH}\otimes\widehat{\upH}$. For simplicity, we let $K=T_1$. Moreover, we use $C$ to denote the constant such that $\mathbb{E}[|\mathbf{x}_i|^4] \leq C \mathbb{E}[|\mathbf{x}_i|^2]$ for any $i\geq1$. Then we respectively bound the variance and bias to obtain the estimation of $\calR_M(\upv^{T_2})-\calR_M(\upv_{1:M}^*)$.

\vspace{0.3cm}
\noindent\textbf{Bound of Variance}:
Lemma \ref{primal-var-estimation} provides a uniform upper bound for $\upV^t$ over $t\in[0:T_2]$. 
\begin{lemma}\label{primal-var-estimation}
    Suppose Assumption \ref{ass-d} holds. Under the setting of Theorem \ref{phase-II-main}, for any $t\in[0:T_2]$, we obtain
    \begin{align}
        \upV_{\diag}^{t}\precsim\eta_0\sigma^2\upI.
    \end{align}
\end{lemma}
\begin{proof}
    The definition of $\upSigma_{\upw}^t$ and the boundedness of $\upw^t$ implicate that $\upSigma_{\upw}^t\preceq\sigma^2\bbE[\upH_{\upw}^t]\preceq\upH$ given $\upv_{1:M}^*\geq\mathbf{0}$. The proof relies on induction. At $t=0$, it follows that $\upV_{\diag}^0=\mathbf{0}\precsim\eta_0\sigma^2\upI$. Assuming $\upV_{\diag}^{\tau}\precsim\eta_0\sigma^2\upI$ for any $\tau\leq t$, we proceed to estimate $\upV^{t+1}$ by combining Eq.~\eqref{bias-and-variance} as,
    \begin{align}
        \upV_{\diag}^{t+1}\preceq&\left(\bbE\left[\left(\calI-\eta_t\calG_{\upw}^t\right)\circ\left(\widehat{\upw}_{\var}^t\otimes\widehat{\upw}_{\var}^t\right)\right]\right)_{\diag}+\eta_t^2\upSigma_{\upw}^t\notag\\
        \preceq&\left(\calI-\eta_t\widehat{\upH}\otimes\upI-\eta_t\upI\otimes\widehat{\upH}\right)\circ\upV_{\diag}^{t}\notag
        \\
        &+\eta_t^2\left(\bbE\left[\calH_{\upw}^t\circ\left(\widehat{\upw}_{\var}^t\otimes\widehat{\upw}_{\var}^t\right)\right]\right)_{\diag}+\eta_t^2\sigma^2\upH\notag
        \\
        \overset{\text{(a)}}{\preceq}&\left(\upI-2\eta_t\widehat{\upH}\right)\upV_{\diag}^{t}+\calO\left(\eta_t^2 (C+2)\langle\upH,\upV_{\diag}^{t}\rangle\upH+\eta_t^2\sigma^2\upH\right)\notag
        \\
        \preceq&\left(\upI-2\eta_t\widehat{\upH}\right)\upV_{\diag}^{t}+\widetilde{\calO}\left(\eta_t^2\eta_0\sigma^2 (C+2)\tr(\upH)\upH+\eta_t^2\sigma^2\upH\right),\notag
    \end{align}
    where (a) is derived from Lemma \ref{aux-6} with $\upA=\diag\{\upv_{1:M}^*+\upw^t\}$ and $\upB=\widehat{\upw}_{\var}^t\otimes\widehat{\upw}_{\var}^t$. For $i\in[1:D]$, we have
    \begin{align}\label{eq-var-1}
        \left(\upV_{\diag}^{t+1}\right)_{i,i}\leq\left(1-2\eta_t\widehat{\upH}_{i,i}\right)\left(\upV_{\diag}^{t}\right)_{i,i}+\widetilde{\calO}\left(\eta_t^2\sigma^2\widehat{\upH}_{i,i}\right).
    \end{align}
    The recursion given by Eq.~\eqref{eq-var-1} implies that $(\upV_{\diag}^{t+1})_{i,i}\lesssim\eta_0\sigma^2$ for any $i\in[1:D]$, using Lemma \ref{aux-3}. For $i\in[D+1:M]$, we obtain
    \begin{align}
        \left(\upV_{\diag}^{t+1}\right)_{i,i}\lesssim\sigma^2\widehat{\upH}_{i,i}\sum_{k=0}^t\eta_k^2\lesssim\eta_0\sigma^2.
    \end{align}
    Therefore, we complete the induction.
\end{proof}
\begin{lemma}\label{variance-upper-bound}
    Suppose Assumption \ref{ass-d} holds. Under the setting of Theorem \ref{phase-II-main}, we have
    \begin{align}
        \llangle\upH,\upV^{T_2}\rrangle\lesssim&\sigma^2\left(\frac{N_0'}{K}+\eta_0\sum_{i=N_0'+1}^{N_0}\lambda_i(\upv_i^*)^2\right)+\sigma^2\eta_0^2(h+K)\sum_{i=N_0+1}^M\lambda_i^2(\widehat{\upb}_i^*)^4,
    \end{align}
    for arbitrary $D\geq N_0\geq N_0'\geq0$.
\end{lemma}
\begin{proof}
    Applying Eq.~\eqref{bias-and-variance}, we obtain
    \begin{align}\label{express-V}
        \upV_{\diag}^{t+1}\preceq&\left(\calI-\eta_t\widehat{\calG}\right)\circ\upV_{\diag}^t+\eta_t^2\left(\calH_{\upw}^t\circ\left(\widehat{\upw}_{\var}^t\otimes\widehat{\upw}_{\var}^t\right)\right)_{\diag}+\eta_t^2\sigma^2\bbE[\upH_{\upw}^t]\notag
        \\
        \overset{\text{(a)}}{\preceq}&\left(\calI-\eta_t\widehat{\calG}\right)\circ\upV_{\diag}^t+\widetilde{\calO}\left(\eta_t^2\sigma^2\eta_0(C+2)\tr(\upH)\upH+\eta_t^2\sigma^2\upH\right)\notag
        \\
        =&\left(\calI-\eta_t\widehat{\calG}\right)\circ\upV_{\diag}^t+\widetilde{\calO}\left(\eta_t^2\sigma^2\upH\right),
    \end{align}
    where (a) is derived from Lemma \ref{primal-var-estimation}. Therefore, the recursion for $\upV_{\diag}^{T_2}$ can be directly derived by incorporating Eq.~\eqref{express-V} as
    \begin{align}\label{recur-variance}
        \upV_{\diag}^{T_2}\precsim\sigma^2\sum_{t=0}^{T_2}\eta_t^2\prod_{i=t+1}^{T_2}\left(\calI-\eta_i\widehat{\calG}\right)\circ\upH
        \overset{\text(b)}{\precsim}\sigma^2\underbrace{\sum_{t=0}^{T_2}\eta_t^2\prod_{i=t+1}^{T_2}\left(\upI-\eta_i\widehat{\upH}\right)\upH}_{\lai},
    \end{align}
    where (b) is based on the inequality $(1-\eta c_2)^2c_3\leq(1-\eta c_2)c_3$, which holds for any $\eta\leq c_2^{-1}$ given fixed constants $c_2,c_3>0$. 
    According to the update rule for $\eta_t$ defined in Algorithm \ref{SGD}, we obtain
    \begin{align}\label{eq-variance-I}
        \lai=&\eta_0^2\sum_{i=1}^h\left(\upI-\eta_0\widehat{\upH}\right)^{h-i}\prod_{j=1}^{L}\left(\upI-\frac{\eta_0}{2^j}\widehat{\upH}\right)^K\upH\notag
        \\
        &+\sum_{l=1}^{L}\left(\frac{\eta_0}{2^l}\right)^2\sum_{i=1}^K\left(\upI-\frac{\eta_0}{2^l}\widehat{\upH}\right)^{K-i}\prod_{j=l+1}^{L}\left(\upI-\frac{\eta_0}{2^j}\widehat{\upH}\right)^K\upH\notag
        \\
        \preceq&4\left(\left(\frac{\eta_0}{2}\right)^2\sum_{i=1}^{h+K}\left(\upI-\frac{\eta_0}{2}\widehat{\upH}\right)^{h+K-i}\prod_{j=1}^{L-1}\left(\upI-\frac{\eta_0}{2^{1+j}}\widehat{\upH}\right)^K\upH\right.\notag
        \\
        &\, \, \, \, \, \, \, \, +\left.\sum_{l=1}^{L-1}\left(\frac{\eta_0}{2^{1+l}}\right)^2\sum_{i=1}^K\left(\upI-\frac{\eta_0}{2^{1+l}}\widehat{\upH}\right)^{K-i}\prod_{j=l+1}^{L-1}\left(\upI-\frac{\eta_0}{2^{1+j}}\widehat{\upH}\right)^K\upH\right)\notag
        \\
        \preceq&100\left(\frac{\eta_0}{2}\left(\upI-\left(\upI-\frac{\eta_0}{2}\widehat{\upH}_{1:D}\right)^{s+K}\right)\prod_{j=1}^{L-1}\left(\upI-\frac{\eta_0}{2^{1+j}}\widehat{\upH}_{1:D}\right)^K\right.\notag
        \\
        &\, \, \, \, \, \, \, \, \, \, \, \, \, \, +\left.\sum_{l=1}^{L-1}\frac{\eta_0}{2^{1+l}}\left(\upI-\left(\upI-\frac{\eta_0}{2^{1+l}}\widehat{\upH}_{1:D}\right)^K\right)\prod_{j=l+1}^{L-1}\left(\upI-\frac{\eta_0}{2^{1+j}}\widehat{\upH}_{1:D}\right)^K\right)\notag
        \\
        &+2\eta_0^2(h+K)\upH_{D+1:M}.
    \end{align}
    Then, we define the following scalar function
    \begin{align}
        f(x):=x\left(1-(1-x)^{h+K}\right)\prod_{j=1}^{L-1}\left(1-\frac{x}{2^j}\right)^K+\sum_{l=1}^{L-1}\frac{x}{2^l}\left(1-\left(1-\frac{x}{2^l}\right)^K\right)\prod_{j=l+1}^{L-1}\left(1-\frac{x}{2^j}\right)^K,\notag
    \end{align}
    as similar as that in [Lemma C.2, \citet{wu2022last}]. Moreover, the following inequality can be directly derived
    \begin{align}\label{aux-scalar-func}
        f\left(\frac{\eta_0}{2}\widehat{\upH}_{1:D}\right)\preceq\frac{8}{K}\upI_{1:N_0'}+\eta_0\widehat{\upH}_{N_0'+1:N_0}+\frac{\eta_0^2}{2}(h+K)\widehat{\upH}_{N_0+1:D}^2,
    \end{align}
    for arbitrary $D\geq N_0\geq N_0'\geq0$ by [Lemma C.3, \citet{wu2022last}]. Applying Eq.~\eqref{aux-scalar-func} to Eq.~\eqref{eq-variance-I} and combining Eq.~\eqref{recur-variance}, we obtain
    \begin{align}
        \upV_{\diag}^{T_2}\precsim&\sigma^2\left(\frac{1}{K}\widehat{\upH}_{1:N_0'}^{-1}+\eta_0\upI_{N_0'+1:N_0}+\eta_0^2(h+K)\widehat{\upH}_{N_0+1:D}+\eta_0^2(h+K)\upH_{D+1:M}\right).
    \end{align}
    Consequently, we have
    \begin{align}
        \llangle\upH,\upV^{T_2}\rrangle\lesssim&\sigma^2\left(\frac{N_0'}{K}+\eta_0\tr\left(\widehat{\upH}_{N_0'+1:N_0}\right)+\eta_0^2(h+K)\tr\left(\widehat{\upH}_{N_0+1:D}^2\right)\right)\notag
        \\
        &+\sigma^2\eta_0^2(h+K)\tr\left(\upH_{D+1:M}^2\right)\notag
        \\
        \lesssim&\sigma^2\left(\frac{N_0'}{K}+\eta_0\sum_{i=N_0'+1}^{N_0}\lambda_i(\upv_i^*)^2\right)+\sigma^2\eta_0^2(h+K)\sum_{i=N_0+1}^M\lambda_i^2(\widehat{\upb}_i^*)^4.
    \end{align}
\end{proof}

\noindent\textbf{Bound of Bias}:
We begin with an analysis of the bias error during a single period of Algorithm \ref{SGD}, where the bias iterations are updated using a constant step size $\eta_t\equiv\eta$ over $\hat{T}$ steps. Based on Eq.~\eqref{bias-and-variance}, the bias iterations are updated according to the following rule:
\begin{align}\label{recur-bias}
    \upB^{t+1}\preceq\bbE\left[\left(\calI-\eta\calG_{\upw}^t\right)\circ\left(\widehat{\upw}_{\bi}^t\otimes\widehat{\upw}_{\bi}^t\right)\right],\quad \forall t\in[0:\hat{T}-1].
\end{align}
Combining Eq.~\eqref{recur-bias}, we have
\begin{align}\label{period-bias}
    \upB_{\diag}^{t+1}\preceq&\left(\calI-\eta\widehat{\calG}\right)\circ\upB_{\diag}^t+\eta^2\bbE\left(\left[\calH_{\upw}^t\circ\upB^t\right]\right)_{\diag}\notag
    \\
    \preceq&\prod_{i=0}^t\left(\calI-\eta\widehat{\calG}\right)\circ\upB_{\diag}^0+\eta^2\sum_{i=0}^{t}\prod_{j=i+1}^{t}\left(\calI-\eta\widehat{\calG}\right)\circ\bbE\left(\left[\calH_{\upw}^t\circ\upB^t\right]\right)_{\diag}\notag
    \\
    \overset{\text{(a)}}{\preceq}&\prod_{i=0}^t\left(\calI-\eta\widehat{\calG}\right)\circ\upB_{\diag}^0+(C+2)\eta^2\sum_{i=0}^{t}\prod_{j=i+1}^{t}\left(\calI-\eta\widehat{\calG}\right)\circ\upH\llangle\upH,\upB^i\rrangle.
\end{align}
where (a) is derived from Lemma \ref{aux-6} by selecting $\upA=\frac{5}{2}\diag\{\widehat{\upb}\}$ and $\upB=\upB^i$. According to Eq.~\eqref{period-bias}, we have
\begin{align}\label{diag-bias}
    \upB_{\diag}^{t+1}\preceq\left(\calI-\eta\widehat{\calG}\right)^{t+1}\circ\upB_{\diag}^0+(C+2)\eta^2\sum_{i=0}^{t}\left(\upI-\eta\widehat{\upH}\right)^{2(t-i)}\upH\llangle\upH,\upB^i\rrangle.
\end{align}
We utilize the following lemma to estimate $\llangle\upH,\upB^{\hat{T}}\rrangle$ under bias iteration defined in Eq.~\eqref{recur-bias}.
\begin{lemma}\label{period-bias-lemma}
    Suppose Assumption \ref{ass-d} and Assumption \ref{ass-ss} hold, and $\upB^t$ is recursively defined by Eq.~\eqref{recur-bias}. Under the setting of Theorem \ref{phase-II-main}, letting $1\leq\hat{T}\leq T$ and $\eta\leq\eta_0$, we have
    \begin{align}
        \llangle\upH,\upB^{\hat{T}}\rrangle\leq\frac{2}{1-\widetilde{\calO}(C+2)\eta\tr(\upH)}\llangle\frac{25}{\eta \hat{T}}\upI_{1:N_0}+\upH_{N_0+1:M},\upB^0\rrangle,
    \end{align}
    where $N_0\in[0:D]$ is an arbitrary integer.
\end{lemma}
\begin{proof}
    By Lemma \ref{aux-5}, we can derive $\eta(\upI-\eta\widehat{\upH})^{2t}\upH\preceq\frac{25}{t+1}\upI$. Applying this to Eq.~\eqref{diag-bias}, we obtain
    \begin{align}\label{bias-eq-1}
        \upB_{\diag}^{t+1}\preceq\left(\calI-\eta\widehat{\calG}\right)^{t+1}\circ\upB_{\diag}^0+25(C+2)\eta\sum_{i=0}^t\frac{\llangle\upH,\upB^i\rrangle}{t+1-i}\cdot\upI,
    \end{align}
    for any $t\in[0:\hat{T}-1]$. Therefore, based on Lemma \ref{aux-7}, we have
    \begin{align}\label{equation-sum}
        \sum_{i=0}^{t}\frac{\llangle\upH,\upB^i\rrangle}{t+1-i}\leq\llangle\sum_{i=0}^t\frac{(\upI-\eta\widehat{\upH})^{2i}\upH}{t+1-i},\upB^0\rrangle+\widetilde{\calO}(C+2)\eta\tr(\upH)\sum_{i=0}^t\frac{\llangle\upH,\upB^i\rrangle}{t+1-i},
    \end{align}
    for any $t\in[1:\hat{T}]$. Eq.~\eqref{equation-sum} implicates that 
    \begin{align}\label{bias-eq-2}
        \sum_{t=0}^{\hat{T}-1}\frac{\llangle\upH,\upB^t\rrangle}{\hat{T}-t}\leq\frac{1}{1-\widetilde{\calO}(C+2)\eta\tr(\upH)}\llangle\sum_{t=0}^{\hat{T}-1}\frac{(\upI-\eta\widehat{\upH})^{2t}\upH}{\hat{T}-t},\upB^0\rrangle,
    \end{align}
    since $\widetilde{\calO}\eta(C+2)\tr(\upH)<1$. Combining Eq.~\eqref{bias-eq-1} with Eq.~\eqref{bias-eq-2}, we obtain
    \begin{align}\label{bias-eq-3}
        \llangle\upH,\upB^{\hat{T}}\rrangle\leq&\llangle(\upI-\eta\widehat{\upH})^{2\hat{T}}\upH,\upB^0\rrangle+\frac{{\calO}(C+2)\eta\tr(\upH)}{1-\widetilde{\calO}(C+2)\eta\tr(\upH)}\llangle\sum_{t=0}^{\hat{T}-1}\frac{(\upI-\eta\widehat{\upH})^{2t}\upH}{\hat{T}-t},\upB^0\rrangle\notag
        \\
        \overset{\text{(a)}}{\leq}&\llangle(\upI-\eta\widehat{\upH})^{2\hat{T}}\upH,\upB^0\rrangle\notag
        \\
        &+\frac{{\calO}(C+2)\eta\tr(\upH)}{1-\widetilde{\calO}(C+2)\eta\tr(\upH)}\llangle\frac{\upI_{1:D}-(\upI_{1:D}-\eta\widehat{\upH}_{1:D})^{\hat{T}}}{\eta \hat{T}}+(\upI_{1:D}-\eta\widehat{\upH}_{1:D})^{\hat{T}}\widehat{\upH}_{1:D},\upB^0\rrangle\notag
        \\
        &+\frac{{\calO}(C+2)\eta\tr(\upH)}{1-\widetilde{\calO}(C+2)\eta\tr(\upH)}\llangle\upH_{D+1:M},\upB^0\rrangle\notag
        \\
        \overset{\text(b)}{\leq}&\frac{2}{1-\widetilde{\calO}(C+2)\eta\tr(\upH)}\llangle\frac{25}{\eta T}\upI_{1:N_0}+\upH_{N_0+1:M},\upB^0\rrangle,
    \end{align}
    where $N_0\in[0:D]$ is an arbitrary integer,  (a) follows the technique in [Lemma C.4, \citet{wu2022last}], and (b) is derived from the invariant scaling relationship between $\widehat{\upH}_{1:D}$ and $\upH_{1:D}$. 
\end{proof}

\begin{lemma}\label{bias-diag}
    Suppose Assumption \ref{ass-d} and Assumption \ref{ass-ss} hold. Under the setting of Theorem \ref{phase-II-main}, letting $2\leq\hat{T}\leq T$ and $\eta\leq\eta_0$, we have
    \begin{align}
        \upB_{\diag}^{\hat{T}}\preceq\left(\upI-\eta\widehat{\upH}\right)^{\hat{T}}\upB_{\diag}^0\left(\upI-\eta\widehat{\upH}\right)^{\hat{T}}+\frac{\widetilde{\calO}(C+2)\eta^2\hat{T}}{1-\widetilde{\calO}(C+2)\eta\tr(\upH)}\llangle\widetilde{\upH}^{\hat{T}},\upB^0\rrangle\overline{\upH}^{\hat{T}},
    \end{align}
    where $\widetilde{\upH}^{t}:=\frac{25}{\eta t}\upI_{1:N_0}+\upH_{N_0+1:M}$, and $\overline{\upH}^{t}:=\frac{25}{\eta t}\upI_{1:N_0'}+\upH_{N_0'+1:M}$ for any $t\geq1$, and $N_0,N_0'\in[0:D]$ could be arbitrary integer.
\end{lemma}
\begin{proof}
    Applying Lemma \ref{period-bias-lemma} into Eq.~\eqref{diag-bias}, we obtain
    \begin{align}\label{bias-lemma-1-main}
        \upB_{\diag}^{\hat{T}}\preceq&\left(\calI-\eta\widehat{\calG}\right)^{\hat{T}}\circ\upB_{\diag}^0+(C+2)\eta^2\left(\upI-\eta\widehat{\upH}\right)^{2(\hat{T}-1)}\upH\llangle\upH,\upB^0\rrangle\notag
        \\
        &+(C+2)\eta^2\sum_{t=1}^{\hat{T}-1}\left(\upI-\eta\widehat{\upH}\right)^{2(\hat{T}-1-t)}\upH\llangle\upH,\upB^t\rrangle\notag
        \\
        \preceq&\left(\calI-\eta\widehat{\calG}\right)^{\hat{T}}\circ\upB_{\diag}^0+(C+2)\eta^2\underbrace{\left(\upI-\eta\widehat{\upH}\right)^{2(\hat{T}-1)}\upH\llangle\upH,\upB^0\rrangle}_{\calcolI}\notag
        \\
        &+\frac{2(C+2)\eta^2}{1-2\widetilde{O}(C+2)\eta\tr(\upH)}\underbrace{\sum_{t=1}^{\hat{T}-1}\left(\upI-\eta\widehat{\upH}\right)^{2(\hat{T}-1-t)}\upH\llangle\widetilde{\upH}^t,\upB^0\rrangle}_{\calcolII},
    \end{align}
    We then provide a bound of term $\calcolII$ as follows:
    \begin{align}
        \calcolII=&\left(\sum_{t=1}^{\hat{T}-1}\llangle\widetilde{\upH}^t,\upB^0\rrangle\right)\upH_{D+1:M}+25\sum_{t=1}^{\hat{T}-1}\left(\upI_{1:D}-\eta\widehat{\upH}_{1:D}\right)^{2(\hat{T}-1-t)}\widehat{\upH}_{1:D}\llangle\widetilde{\upH}^t,\upB^0\rrangle\notag
        \\
        \preceq&\hat{T}\log(\hat{T})\llangle\widetilde{\upH}^{\hat{T}},\upB^0\rrangle\upH_{D+1:M}+25\left(\sum_{t=1}^{\hat{T}/2-1}\left(\upI_{1:D}-\eta\widehat{\upH}_{1:D}\right)^{\hat{T}}\widehat{\upH}_{1:D}\llangle\widetilde{\upH}^t,\upB^0\rrangle\right.\notag
        \\
        &+\left.\sum_{t=\hat{T}/2}^{\hat{T}-1}\left(\upI_{1:D}-\eta\widehat{\upH}_{1:D}\right)^{\hat{T}-1-t}\widehat{\upH}_{1:D}\llangle\widetilde{\upH}^{\hat{T}/2},\upB^0\rrangle\right)\notag
        \\
        =&\hat{T}\log(\hat{T})\llangle\widetilde{\upH}^{\hat{T}},\upB^0\rrangle\upH_{D+1:M}+25\left(\left(\upI_{1:D}-\eta\widehat{\upH}_{1:D}\right)^{\hat{T}}\widehat{\upH}_{1:D}\llangle\sum_{t=1}^{\hat{T}/2-1}\widetilde{\upH}^t,\upB^0\rrangle\right.\notag
        \\
        &+\left.\frac{\upI_{1:D}-\left(\upI_{1:D}-\eta\widehat{\upH}_{1:D}\right)^{\hat{T}/2}}{\eta}\llangle\widetilde{\upH}^{\hat{T}/2},\upB^0\rrangle\right)\notag
        \\
        \preceq&\hat{T}\log(\hat{T})\llangle\widetilde{\upH}^{\hat{T}},\upB^0\rrangle\upH_{D+1:M}+25\left(\hat{T}\log(\hat{T})\left(\upI_{1:D}-\eta\widehat{\upH}_{1:D}\right)^{\hat{T}}\widehat{\upH}_{1:D}\llangle\widetilde{\upH}^{\hat{T}},\upB^0\rrangle\right.\notag
        \\
        &+2\left.\frac{\upI_{1:D}-\left(\upI_{1:D}-\eta\widehat{\upH}_{1:D}\right)^{\hat{T}/2}}{\eta}\llangle\widetilde{\upH}^{\hat{T}},\upB^0\rrangle\right)\notag
        \\
        \overset{\text{(a)}}{\preceq}& \hat{T}\log(\hat{T})\llangle\widetilde{\upH}^{\hat{T}},\upB^0\rrangle\overline{\upH}^{\hat{T}},
    \end{align}
    where (a) follows the similar technique used in Eq.~\eqref{bias-eq-3}. We then proceed to establish bounds on $\calcolI$. It's worth to notice that
    \begin{align}\label{sub-1}
        \left(\upI-\eta\widehat{\upH}\right)^{2(\hat{T}-1)}\upH\preceq\frac{25}{2\eta(\hat{T}-1)}\upI_{1:N_0'}+25\widehat{\upH}_{N_0'+1:D}+\upH_{D+1:M}\preceq\overline{\upH}^{\hat{T}}.
    \end{align}
    Applying Eq.~\eqref{sub-1} to $\calcolI$, we obtain
    \begin{align}
        \calcolI\preceq\overline{\upH}^{\hat{T}}\llangle\upH,\upB^0\rrangle\preceq \hat{T}\overline{\upH}^{\hat{T}}\llangle\widetilde{\upH}^{\hat{T}},\upB^0\rrangle,
    \end{align}
    where the last inequality is derived from the condition $\eta<1/(25\tr(\upH))$, which ensures $\lambda_i(\upH)<1/\eta$ holds for all $i\in[1:N_0]$. Combining the estimation of $\calcolI$ and $\calcolII$ with Eq.~\eqref{bias-lemma-1-main}, we have
    \begin{align}
        \upB_{\diag}^{\hat{T}}\preceq&\left(\calI-\eta\widehat{\calG}\right)^{\hat{T}}\circ\upB_{\diag}^0+(C+2)\eta^2\hat{T}\widetilde{\upH}^{\hat{T}}\llangle\widetilde{\upH}^{\hat{T}},\upB^0\rrangle\notag
        \\
        &+\frac{2(C+2)\eta^2}{1-\widetilde{\calO}(C+2)\eta\tr(\upH)}\hat{T}\log(\hat{T})\llangle\widetilde{\upH}^{\hat{T}},\upB^0\rrangle\overline{\upH}^{\hat{T}}\notag
        \\
        \preceq&\left(\calI-\eta\widehat{\calG}\right)^{\hat{T}}\circ\upB_{\diag}^0+\frac{\widetilde{\calO}(C+2)\eta^2\hat{T}}{1-\widetilde{\calO}(C+2)\eta\tr(\upH)}\llangle\widetilde{\upH}^{\hat{T}},\upB^0\rrangle\overline{\upH}^{T}.
    \end{align}
    By the definition of $\widehat{\calG}$, we complete the proof.
\end{proof}

Notice that in \textbf{Phase II}, the step size 
$\eta_t$ decays geometrically. Thus, we define the bias iteration at the end of the step-size-decaying phase as:
\begin{align}
    \widetilde{\upB}^l:=\begin{cases}
        \upB^h, & l=0,\\
        \upB^{h+Kl}, & l\in[1:L].
    \end{cases}
\end{align}
Based on the step-size iteration in Algorithm \ref{SGD} and preceding definition, we formalize the iterative process of Algorithm \ref{SGD} in Phase II as: 1) Phase when $l=0$: Initialized from $\upB^0$, Algorithm \ref{SGD} runs $h$ iterations with step size $\eta_0$, yielding $\widetilde{\upB}^0$;
2) Phase when $l\geq1$: Initialized from $\widetilde{\upB}^{l-1}$, Algorithm \ref{SGD} runs $K$ iterations with step size $\eta_0/2^{l}$, yielding $\widetilde{\upB}^l$. This multi-phase process terminates at $l=L$, with $\widetilde{\upB}^L = \upB^{T_2}$ as the final output.
\begin{lemma}\label{lemma-Bias-tr-1}
    Suppose Assumption \ref{ass-d} and Assumption \ref{ass-ss} hold. Under the setting of Theorem \ref{phase-II-main}, we have
    \begin{align}
        \llangle\upH,\widetilde{\upB}^l\rrangle\leq K_l:=
        \begin{cases}
            4\llangle\frac{25}{\eta_0 h}\upI_{1:N_0}+\upH_{N_0+1:M},\upB^0\rrangle, & \text{ for }\, l=0,
            \\
            4\llangle\frac{25\cdot2^l}{\eta_0 K}\upI_{1:N_0}+\upH_{N_0+1:M},\widetilde{\upB}^{l-1}\rrangle, & \text{ for }\, l\in[1:L],
        \end{cases}
    \end{align}
    for arbitrary $N_0\in[0:D]$.
\end{lemma}
\begin{proof}
    For $\llangle\upH,\widetilde{\upB}^0\rrangle$, we apply Lemma \ref{period-bias-lemma} with $\eta=\eta_0$ and $\hat{T}=h$, and use the condition that $\widetilde{\calO}(C+2)\eta\tr(\upH)\leq1/4$; For $\llangle\upH,\widetilde{\upB}^l\rrangle$ with $l\geq 2$, we apply Lemma \ref{period-bias-lemma} with $\eta=\eta_0/2^l$, $\hat{T}=K$ and $\upB^0=\widetilde{\upB}^{l-1}$, and use the condition that $\widetilde{\calO}(C+2)\eta\tr(\upH)\leq1/4$.
\end{proof}

\begin{lemma}\label{lemma-Bias-diag-2}
    Suppose Assumption \ref{ass-d} and Assumption \ref{ass-ss} hold. Under the setting of Theorem \ref{phase-II-main}, we have
    \begin{align}
        \widetilde{\upB}_{\diag}^l\preceq\upR^l:=\begin{cases}
            \left(\upI-\eta_0\widehat{\upH}\right)^{h}\upB_{\diag}^0\left(\upI-\eta_0\widehat{\upH}\right)^{h}+P_0\overline{\upH}_0^h, & \text{ for }\, l=0,
            \\
            \left(\upI-\frac{\eta_0}{2^l}\widehat{\upH}\right)^{h}\widetilde{\upB}_{\diag}^{l-1}\left(\upI-\frac{\eta_0}{2^l}\widehat{\upH}\right)^{h}+P_l\overline{\upH}_l^K, & \text{ for }\, l\in[1:L],
        \end{cases}
    \end{align}
    where $\overline{\upH}_0^t:=\frac{25}{\eta_0 t}\upI_{1:N_0'}+\upH_{N_0'+1:M}$ and $\overline{\upH}_l^t:=\frac{25\cdot2^l}{\eta_0 t}\upI_{1:N_0'}+\upH_{N_0'+1:M}$ for any $t\geq1$ and arbitrary $N_0'\in[0:D]$, and $P_0:=\widetilde{\calO}(C+2)\eta_0^2h\langle\widetilde{\upH}_0^h,\upB^0\rangle$ with $\widetilde{\upH}_0^h:=\frac{25}{\eta_0 h}\upI_{1:N_0}+\upH_{N_0+1:M}$ and $P_l:=\widetilde{\calO}(C+2)(\frac{\eta_0}{2^l})^2K\langle\widetilde{\upH}_l^K,\widetilde{\upB}^{l-1}\rangle$ for $l\in[1:L]$ with $\widetilde{\upH}_l^K:=\frac{25\cdot2^l}{\eta_0 K}\upI_{1:N_0}+\upH_{N_0+1:M}$ for arbitrary $N_0\in[0:D]$.
\end{lemma}
\begin{proof}
    For $\widetilde{\upB}^0$, we apply Lemma \ref{bias-diag} with $\eta=\eta_0$ and $\hat{T}=h$, and use the condition that $\widetilde{\calO}(C+2)\eta\tr(\upH)\leq1/4$. For $\widetilde{\upB}^l$ with $l\geq 2$, we apply Lemma \ref{bias-diag} with $\eta=\eta_0/2^l$, $\hat{T}=K$ and $\upB^0=\widetilde{\upB}^{l-1}$, and use the condition that $\widetilde{\calO}(C+2)\eta\tr(H)\leq1/8$.
\end{proof}
\begin{lemma}\label{phase-II-thm-p1}
    Suppose Assumption \ref{ass-d} and Assumption \ref{ass-ss} hold. Under the setting of Theorem \ref{phase-II-main}, we have
    \begin{align}
        \llangle\upH,\upB^{T_2}\rrangle=\llangle\upH,\widetilde{\upB}^L\rrangle\leq e\llangle\upH,\widetilde{\upB}^1\rrangle
    \end{align}
\end{lemma}
\begin{proof}
    Consider $l\geq 1$. According to Lemma \ref{lemma-Bias-diag-2}, we obtain
    \begin{align}\label{exp-dynamic-Bias}
        \widetilde{\upB}_{\diag}^l\preceq&\left(\upI-\frac{\eta_0}{2^l}\widehat{\upH}\right)^{h}\widetilde{\upB}_{\diag}^{l-1}\left(\upI-\frac{\eta_0}{2^l}\widehat{\upH}\right)^{h}+P_l\widetilde{\upH}_l^K\notag
        \\
        \overset{\text{(a)}}{\preceq}&\widetilde{\upB}_{\diag}^{l-1}+\widetilde{\calO}(C+2)\log(K)\cdot\frac{\eta_0}{2^l}\cdot\llangle\upH,\widetilde{\upB}^{l-1}\rrangle\upI.
    \end{align}
    where (a) is derived from choosing $N_0'=D$ and $N_0=0$ in $\overline{\upH}_l^K$ and $\widetilde{\upH}_l^K$ for any $l\in[1:L]$, respectively, and $\upH_{D+1:M}\preceq\frac{\widetilde{\calO}(2^l)}{\eta_0K}\upI_{D+1:M}$. Eq.~\eqref{exp-dynamic-Bias} implies that
    \begin{align}\label{linear-relation-bias}
        \llangle\upH,\widetilde{\upB}^l\rrangle\leq\left(1+\widetilde{\calO}(C+2)\tr(\upH)\log(K)\cdot\frac{\eta_0}{2^l}\right)\llangle\upH,\widetilde{\upB}^{l-1}\rrangle.
    \end{align}
    Therefore, we have following estimation of bias iterations using Eq.~\eqref{linear-relation-bias}:
    \begin{align}
        \llangle\upH,\widetilde{\upB}^L\rrangle\leq&\prod_{l=1}^L\left(1+\widetilde{\calO}(C+2)\tr(\upH)\log(K)\cdot\frac{\eta_0}{2^l}\right)\llangle\upH,\widetilde{\upB}^1\rrangle\notag
        \\
        \leq&\exp\left\{\widetilde{\calO}(C+2)\eta_0\tr(\upH)\log(K)\sum_{l=1}^L2^{-l}\right\}\llangle\upH,\widetilde{\upB}^1\rrangle\notag
        \\
        \leq&e\llangle\upH,\widetilde{\upB}^1\rrangle.
    \end{align}
\end{proof}
\begin{lemma}\label{phase-II-thm-p2}
    Suppose Assumption \ref{ass-d} and Assumption \ref{ass-ss} hold. Under the setting of Theorem \ref{phase-II-main}, we have
    \begin{align}
        \llangle\upH,\widetilde{\upB}^{1}\rrangle\leq&8\llangle\frac{25}{\eta_0 K}\upI_{1:N_0}+\upH_{N_0+1:M},\left(\upI-\eta_0\widehat{\upH}\right)^{2h}\upB^0\rrangle\notag
        \\
        &+\widetilde{\calO}(C+2)\Gamma_K(\upH)\llangle\frac{25}{\eta_0h}\upI_{1:N_0'}+\upH_{N_0'+1:M},\upB^0\rrangle,
    \end{align}
    where $\Gamma_K(\upH):=\left(\frac{625N_0'}{K}+\frac{25\eta_0h}{K}\tr(\upH_{N_0'+1:N_0})+\eta_0^2h\tr(\upH_{N_0+1:M}^2)\right)$ for arbitrary $D\geq N_0\geq N_0'\geq 0$.
\end{lemma}
\begin{proof}
    According to Lemma \ref{lemma-Bias-tr-1}, we have
    \begin{align}
        \llangle\upH,\widetilde{\upB}^1\rrangle\leq8\llangle\frac{25}{\eta_0 K}\upI_{1:N_0}+\upH_{N_0+1:M},\widetilde{\upB}^0\rrangle,\notag
    \end{align}
    for arbitrary $N_0\in[0:D]$. Then, choosing $N_0=N_0'$ in Lemma \ref{lemma-Bias-diag-2}, we obtain
    \begin{align}
        \widetilde{\upB}_{\diag}^0\preceq&\left(\upI-\eta_0\widehat{\upH}\right)^h\upB_{\diag}^0\left(\upI-\eta_0\widehat{\upH}\right)^h\notag
        \\
        &+\widetilde{\calO}(C+2)\eta_0^2h\llangle\frac{25}{\eta_0h}\upI_{1:N_0'}+\upH_{N_0'+1:M},\upB^0\rrangle\left(\frac{25}{\eta_0h}\upI_{1:N_0'}+\upH_{N_0'+1:M}\right).\notag
    \end{align}
    Combining above two inequalities, we have
    \begin{align}
        \llangle\upH,\widetilde{\upB}^1\rrangle\leq&8\llangle\frac{25}{\eta_0 K}\upI_{1:N_0}+\upH_{N_0+1:M},\left(\upI-\eta_0\widehat{\upH}\right)^{2h}\upB^0\rrangle\notag
        \\
        &+\widetilde{\calO}(C+2)\eta_0^2h\llangle\frac{25}{\eta_0h}\upI_{1:N_0'}+\upH_{N_0'+1:M},\upB^0\rrangle\notag
        \\
        &\, \, \, \, \, \, \, \, \, \, \, \, \, \, \, \, \, \, \, \, \, \, \, \, 
        \, \, \, \, \, \, \, \, \, \, \times\llangle\frac{25}{\eta_0 K}\upI_{1:N_0}+\upH_{N_0+1:M},\frac{25}{\eta_0h}\upI_{1:N_0'}+\upH_{N_0'+1:M}\rrangle,\notag
    \end{align}
    where
    \begin{align}
        &\llangle\frac{25}{\eta_0 K}\upI_{1:N_0}+\upH_{N_0+1:M},\frac{25}{\eta_0h}\upI_{1:N_0'}+\upH_{N_0'+1:M}\rrangle\notag
        \\
        \leq&\frac{625N_0'}{\eta_0^2hK}+\frac{25}{\eta_0K}\tr(\upH_{N_0'+1:N_0})+\tr(\upH_{N_0+1:M}^2),
    \end{align}
    when $N_0>N_0'$.
\end{proof}
\begin{lemma}\label{bias-upper-bound}
    Suppose Assumptions \ref{ass-d} and \ref{ass-ss} hold. Under the setting of Theorem \ref{phase-II-main}, we have
    \begin{align}
        \llangle\upH,\upB^{T_2}\rrangle\lesssim&\llangle\frac{1}{\eta_0 K}\upI_{1:N_0}+\upH_{N_0+1:M},\left(\upI-\eta_0\widehat{\upH}\right)^{2h}\upB^0\rrangle\notag
        \\
        &+(C+2)\Gamma_K(\upH)\llangle\frac{1}{\eta_0h}\upI_{1:N_0'}+\upH_{N_0'+1:M},\upB^0\rrangle,
    \end{align}
    where $\Gamma_K(\upH):=\left(\frac{625N_0'}{K}+\frac{25\eta_0h}{K}\tr(\upH_{N_0'+1:N_0})+\eta_0^2h\tr(\upH_{N_0+1:M}^2)\right)$ for arbitrary $D\geq N_0\geq N_0'\geq 0$.
\end{lemma}
\begin{proof}
    Using Lemma \ref{phase-II-thm-p1} and Lemma \ref{phase-II-thm-p2} we directly obtain the results.
\end{proof}

Finally we finish the proof of Theorem \ref{phase-II-main}.
\begin{proof}[Proof of Theorem \ref{phase-II-main}]
    Combining Lemma \ref{variance-upper-bound} with Lemma \ref{bias-upper-bound}, we derive Eq.~\eqref{phaseII-p1}. Based on Theorem \ref{high-probability-phase-II}, the equality $\upw^{T_2}=\upv^{T_1+T_2}$ holds with probability at least $1-\delta$. By setting $N_0'=N_0=N_1'=N_1=D$ in Eq.~\eqref{phaseII-p1} and applying Markov's inequality, we obtain Eq.~\eqref{phaseII-p2}.
\end{proof}

\subsection{Proof of Main Results}

In this section, we finally complete the proof of main results o the global convergence of Algorithm \ref{SGD} in Theorem \ref{theorem-main-convergence}, based on the analysis of \textbf{Phase I} and \textbf{Phase II}. Before we propose the main Theorem \ref{theorem-main-convergence}, we set the parameter as follows:
\begin{equation}\label{para-setting-spec}
\small
    \begin{split}
&L_1=\widetilde{\calO}\left((\sigma^2+\calM^2(\upb))^2+\Hatsigmax(D)\right),\, \, \, L_2=\widetilde{\calO}(\sigma^2+\calM^2(\upb)),\, \, \, 
        L_3=1+\frac{L_1\Tildesigmax(D)\Barsigmin(D)}{\sigmin(D)},
    \end{split}
\end{equation}
\begin{theorem}\label{theorem-main-convergence}[General Version of Theorem \ref{theorem-3}]
	Under Assumption \ref{ass-d} and \ref{ass-ss}, we consider a predictor trained by Algorithm \ref{SGD} with total sample size $T$. Let $h<T$ and $T_1:=\lceil(T-h)/\log(T-h)\rceil$. Suppose there exists $D\leq M$ such that $T_1\in[\frac{L_1L_3}{\sigmin(D)\Barsigmin(D)}, \frac{L_2L_3^2}{\Tildesigmax(D)\Barsigmin^2(D)}]$ with parameter setting Eq.~\eqref{para-setting-spec} and let $\eta=\widetilde{\Omega}(\frac{\Barsigmin(D)}{\sigma^2+\calM^2(\upb)})$. Then we have
    \begin{equation}
        \begin{split}
        \calR_M(\upv^{T})-\calR_M(\upv^*)\lesssim&\frac{\sigma^2D}{T_1}+\sigma^2\eta^2(h+T_1)\sum_{i=D+1}^M\lambda_i^2(\upv_i^*)^4\notag
        \\
        &+\frac{1}{\eta T_1}\tr\left(\left(\upI_{1:D}-\frac{\eta}{4}\upH_{1:D}^*\right)^{2h}\diag\left\{(\upv_{1:D}^*)^{\odot2}\right\}\right)\notag
        \\
        &+\llangle\upH_{D+1:M}^*,\diag\left\{(\upv_{D+1:M}^*)^{\odot2}\right\}\rrangle\notag
        \\
        &+\left(\frac{D}{T_1}+\eta^2h\tr\left((\upH_{D+1:M}^*)^2\right)\right)\llangle\frac{1}{\eta h}\upI_{1:D}+\upH_{D+1:M}^*,\diag\left\{(\upv^*)^{\odot2}\right\}\rrangle,\notag
        \end{split}
    \end{equation}
    with probability at least 0.95. Otherwise, let $T_1\in[\frac{L_1L_3}{\sigmin(M)\Barsigmin(M)},+\infty)$ with parameter setting Eq.~\eqref{para-setting-spec} and $\eta=\widetilde{\Omega}(\frac{\Barsigmin(M)}{\sigma^2+\calM^2(\upb)})$. Then we have
    \begin{equation}
        \begin{split}
        \calR_M(\upv^{T})-\calR_M(\upv^*)\lesssim&\frac{\sigma^2M}{T_1}+\frac{1}{\eta T_1}\tr\left(\left(\upI-\frac{\eta}{4}\upH^*\right)^{2h}\diag\left\{(\upv_{1:M}^*)^{\odot2}\right\}\right)\notag
        \\
        &+\frac{M}{\eta hT_1}\tr\left(\diag\left\{(\upv_{1:M}^*)^{\odot2}\right\}\right),\notag
        \end{split}
    \end{equation}
    with probability at least 0.95.
\end{theorem}
\begin{proof}
    Combining Theorem \ref{theorem_1} and Theorem \ref{phase-II-main}, we complete the proof.
\end{proof}

\section{Proofs of Lower Bound}\label{sec-lower}
In this section, we introduce the proof of the lower bound in Theorem \ref{theorem-lower-bound}. Let $\bar{\sigma}^2:=\bbE[\xi^2]+\sum_{i=M+1}^{\infty}\lambda_i(\upv_i^*)^4$. Recall the analysis in Phase I, $\upv^{T_1}$ satisfies the inequality $\overline{\upb} \leq \upv^{T_1} \leq \widehat{\upb}$ with high probability. Here, $\widehat{\upb}$ is defined as $\widehat{\upb}^{\top}=(\frac{3}{2}(\upv_{1:D}^*)^{\top},3(\upv_{D+1:M}^*)^{\top})$, while $\overline{\upb}$ is defined as $\overline{\upb}^{\top}=(\frac{1}{2}(\upv_{1:D}^*)^{\top},\mathbf{0}^{\top})$. We begin with the required concepts as below. A Markov chain $\{\breve{\upv}^t\}_{t=0}^{T_2}$ is constructed with initialization $\breve{\upv}^0$ satisfying $\overline{\upb} \leq \breve{\upv}^0 \leq \widehat{\upb}$. The update rule is defined by
\begin{align}\label{eq-bv}
     \breve{\upv}^{t+1}=\breve{\upv}^t-\eta_t\upH_{\breve{\upv}}^t\left(\breve{\upv}^t-\upv_{1:M}^*\right)+\eta_t\upR_{\breve{\upv}}^t\Pi_M\upx^t, \quad \forall t\in[0:T_2-1],\notag
\end{align}
where $\upH_{\breve{\upv}}^{t}$ and $\upR_{\breve{\upv}}^t$ satisfy: 
\begin{enumerate}
    \item If $\overline{\upb}\leq\breve{\upv}^t\leq\widehat{\upb}$, $\upH_{\breve{\upv}}^{t}=(\breve{\upv}^t\odot\Pi_M\upx^t)\otimes((\breve{\upv}^t+\upv_{1:M}^*)\odot\Pi_M\upx^t)$ and $\upR_{\breve{\upv}}^t=(\xi^t+\sum_{i=M+1}^{\infty}\upx_i^t(\upv_i^*)^2)\diag\{\breve{\upv}^t\}$,
    \item Otherwise, for any $\tau\in[t:T_2-1]$,  $\upH_{\breve{\upv}}^{\tau}=\frac{25}{4}(\upv_{1:M}^*\odot\Pi_{M}\upx^{\tau})\otimes(\upv_{1:M}^*\odot\Pi_{M}\upx^{\tau})$ and $\upR_{\breve{\upv}}^{\tau}=(\xi^{\tau}+\sum_{i=M+1}^{\infty}\upx_i^{\tau}(\upv_i^*)^2)$ $\diag\{\overline{\upb}\}$.
\end{enumerate}
Let $\breve{\upw}^t:=\breve{\upv}^t-\upv_{1:M}^*$ be the error vector, and let $t_s:=\inf\{t\mid \breve{\upv}^t\nleqslant\widehat{\upb}\bigvee\breve{\upv}^t\ngeqslant\overline{\upb}\}$ be the stopping time. According to Eq.~\eqref{eq-bv}, $\{\breve{\upw}^t\}_{t=1}^{T_2}$ is recursively defined by
\begin{align}
		\breve{\upw}^{t+1}=\left(\upI-\eta_t\upH_{\breve{\upv}}^t\right)\breve{\upw}^t+\eta_t\upR_{\breve{\upv}}^t\Pi_M\upx^t. \notag
\end{align}
We define $\breve{\upV}^t=\bbE\left[\breve{\upw}^t\otimes\breve{\upw}^t\right]$. By the definitions of $\calH_{\cdot}^t$, $\widetilde{\calH}_{\cdot}^t$, $\calG_{\cdot}^t$, and $\widetilde{\calG}_{\cdot}^t$ in Phase II, we derive the iterative relationship governing the sequence $\{\breve{\upV}^t\}_{t=0}^{T_2}$:
\begin{align}
		\breve{\upV}^{t+1}&=\bbE\left[\left(\calI-\eta_t\calG_{\breve{\upv}}^t\right)\circ\left(\breve{\upw}^t\otimes\breve{\upw}^t\right)\right]+\eta_t^2\upSigma_{\breve{\upv}}^t,
\end{align}
for $t\in[0:T_2-1]$ with $\upV^0=\left(\upw^0-\upv_{1:M}^*\right)\otimes\left(\upw^0-\upv_{1:M}^*\right)$. If $t<t_s$, $\upSigma_{\breve{\upv}}^t=\bar{\sigma}^2\upLambda\bbE[\diag\{\breve{\upv}^{t\odot2}\}]$; otherwise, $\upSigma_{\breve{\upv}}^{\tau}=\bar{\sigma}^2\upLambda\diag\{\overline{\upb}^{\odot2}\}$ for any $\tau\geq t$. According to the definitions above, we obtain following estimation of the last-iteration function value:
\begin{align}\label{lower-bound-I}
    \bbE\left[\calR_M(\breve{\upw}^{T_2})-\calR_M(\upv_{1:M}^*)\right]\geq\frac{1}{24}\llangle\breve{\upH},\bbE\left[\breve{\upw}^{T_2}\otimes\breve{\upw}^{T_2}\right]\rrangle\geq\frac{1}{24}\llangle\breve{\upH},\breve{\upV}^{T_2}\rrangle,
\end{align}
where $\breve{\upH}:=12\upLambda\diag\{\upv_{1:M}^*\odot\upv_{1:M}^*\}$. We define $\breve{\calG}^i:=\breve{\upH}\otimes\upI+\upI\otimes\breve{\upH}-\eta_i\breve{\upH}\otimes\breve{\upH}$. We formally propose the lower bound of the estimate in Theorem \ref{theorem-lower-bound} as below.

\begin{theorem}\label{theorem-lower-bound}
Under Assumption \ref{ass-d} and \ref{ass-ss}, we consider a predictor trained by Algorithm \ref{SGD} with iteration number $T$ and middle phase length $h>\lceil(T-h)/\log(T-h)\rceil$. Let $D\asymp\min\{T^{1/\max\{\beta,(\alpha+\beta)/2\}},$ $M\}$ and $\eta\asymp D^{\min\{0,(\alpha-\beta)/4\}}$. Then we have
    \begin{align}\label{last-iterate-lower-1}
    \bbE\left[\mathcal{R}_M(\upv^{T})\right]-\bbE[\xi^2]\gtrsim\frac{1}{M^{\beta-1}}+\frac{\bar{\sigma}^2D}{T}+\frac{1}{D^{\beta-1}}\mathds{1}_{M>D},
    \end{align}
where $\bar{\sigma}^2:=\bbE[\xi^2]+\sum_{i=M+1}^{\infty}\lambda_i(\upv_i^*)^4$.
\end{theorem}
\begin{proof}
    The proof of Theorem \ref{theorem-lower-bound} is divided into two steps. \textbf{Step I} reveals that for coordinates $j\geq\widetilde{\calO}(D)$, the slow ascent rate inherently prevents $\upv_j^t$ from attaining close proximity to the optimal solution $\upv_j^*$ upon algorithmic termination.
    
    \noindent \textbf{Step I:} Let $M\gtrsim T^{1/\max\{\beta,(\alpha+\beta)/2\}}$, and define $T_1:=\lceil(T-h)/\log(T-h)\rceil$ and $D^{\dagger}:=\calO((\eta T)^{2/(\alpha+\beta)})$. Considering the $\upb$-capped coupling process $\{\bar{\upv}^t\}_{t=0}^{T}$ mentioned in Phase I, we denote $\hat{\tau}_j$ as the stopping time when $\bar{\upv}_j^{\hat{\tau}_j}\geq\frac{1}{4}\upv_j^*$ for each coordinate $D^{\dagger}\leq j\leq M$, i.e., 
    \begin{align}
        \hat{\tau}_j=\inf\left\{t:\bar{\upv}_j^{t}\geq\frac{1}{4}\upv_j^*\right\}.\notag
    \end{align}
    We aim to estimate the following probability for coordinates $j\in[D^{\dagger}:M]$ and times $t_1\in[1:T_1]$:
    \begin{align}
        \bbP\left(\mathcal{J}^{\hat{\tau}_j=t_1}(j)=\left\{\bar{\upv}_j^0\leq\frac{1}{8}\upv_j^*\bigwedge\bar{\upv}_j^{0:t_1-1}\in\left[0:\frac{1}{4}\upv_j^*\right]\bigwedge\bar{\upv}_j^{t_1}\geq\frac{1}{4}\upv_j^*\right\}\right).\notag
    \end{align}
    For fixed $j\in[D^{\dagger}:M]$ and any $t\in[0:t_1-1]$, we have
    \begin{equation}\label{lower-concen-1}
    \begin{aligned}
        \bbE\left[\bar{\mathbf{v}}_j^{t+1}\mid\mathcal{F}^t\right]=&\bbE_{{\mathbf{x}}_{1:M}^{t+1},{\xi}^{t+1},{\zeta}_{M+1:\infty}^{t+1}}\left[\bar{\mathbf{v}}_j^{t}-\eta\left(\left((\bar{\upv}_j^t)^2-(\upv_j^*)^2\right)\hat{\upx}_j^{t+1}+\hat{\upz}_j^{t+1}(\bar{\upv}^t)\right.\right.
        \\
        &\, \, \, \, \, \, \, \, \, \, \, \, \, \, \, \, \, \, \, \, \, \, \, \, \, \, \, \, \, \, \, \, \, \, \, \, \, \, \, \, \, \, \, \, \, \left.\left.-\hat{\zeta}_{M+1:\infty}^{t+1}-\hat{\xi}^{t+1}\right)\hat{\mathbf{x}}_j^{t+1}\bar{\mathbf{v}}_j^{t+1}\right]
        \\
		\leq&\left(1+\eta\lambda_{j}(\mathbf{v}_{j}^*)^2\right)\bar{\mathbf{v}}_j^{t}.
    \end{aligned}
    \end{equation}
    Similarly, based on Lemma \ref{aux-1}, we have
    \begin{equation}
    \small
    \begin{split}
        \bbE\left[\exp\left\{\lambda(\bar{\upv}_j^{t+1}-\bbE[\bar{\upv}_j^{t+1}\mid\calF^t])\right\}\mid\calF^t\right]\leq\exp\left\{\frac{\lambda^2\eta^2\lambda_j(\upv_j^*)^2\calO\left(\left[\bar{\sigma}^2+\mathcal{M}^2(\upv_{1:M}^*)\right]\log^4(MT_1/\delta)\right)}{2}\right\},\notag
    \end{split}
    \end{equation}
    for any $\lambda\in\bbR$. According to the setting of stepsize $\eta$, we have $(1+\eta\lambda_i(\upv_i^*)^2)^{T_1}\leq 2$ for any $i\in[D^{\dagger}:M]$. Utilizing Corollary \ref{aux-coro-3} and Eq.~\eqref{lower-concen-1}, we can establish the probability bound for event $\mathcal{J}^{\hat{\tau}_j=t_1}(j)$ for any time $t_1\in[1:T_1]$ as
    \begin{align}\label{probability-lower-1}
		\Pro\left (\mathcal{J}^{\hat{\tau}_j=t_1}(j)\right )\leq\exp\left\{-\frac{1}{T\eta^2\lambda_j\calO\left([\bar{\sigma}^2+\calM^2(\upv_{1:M}^*)]\log^2(MT_1/\delta)\right)}\right\}.
    \end{align}

    Finally, combining the probability bounds Eq.\eqref{probability-lower-1} with the setting of $\eta$, we obtain the following probability bound for complement event $\bigcup_{j=D^{\dagger}}^M\{\max_{t\in[1:T_1]}\bar{\upv}_j^t\geq\frac{1}{4}\upv_j^*\}$:
    \begin{align}
        \bbP\left(\bigcup_{j=D^{\dagger}}^M\left\{\max_{t\in[1:T_1]}\bar{\upv}_j^t\geq\frac{1}{4}\upv_j^*\right\}\right)\leq&\sum_{j=D^{\dagger}}^M\sum_{t_1=1}^{T_1}\Pro\left (\mathcal{J}^{\hat{\tau}_j=t_1}(j)\right)\notag
        \\
        \leq&MT_1\exp\left\{-\frac{\min_{D^{\dagger}\le j\le M}\lambda_j^{-1}}{T_1\eta^2\calO\left([\bar{\sigma}^2+\calM^2(\upv_{1:M}^*)]\log^4(MT_1/\delta)\right)}\right\}\notag
        \\
        \leq&\frac{\delta}{2}.
    \end{align}
    Therefore, we have $\bigcap_{j=D^{\dagger}}^M\{\max_{t\in[1:T_1]}{\upv}_j^t<\frac{1}{4}\upv_j^*\}$ with high probability.

    Similar to \textbf{Phase II}'s analysis, \textbf{Step II} derives the lower bound estimate of the risk by constructing a recursive expression for $\{\breve{\upV}_{\diag}^{t}\}_{t=0}^{T_2}$ where $T_2=T-T_1$. We continue to use use $\upv^{T_1}$, which satisfies Eq.~\eqref{thm-1-eq}, as the initial point for the SGD iterations in \textbf{Step II}. If $M\gtrsim T^{1/\max\{\beta,(\alpha+\beta)/2\}}$, we further require that $\upv^{T_1}$ satisfies
    \begin{align}
        \upv_j^{T_1}<\frac{1}{4}\upv_j^*,\quad\forall j\in\left[\widetilde{\calO}(T^{1/\max\{\beta,(\alpha+\beta)/2\}}), M\right].\notag
    \end{align}
    According to Theorem \ref{theorem_1} and the result of \textbf{Step I}, the assumption on $\upv^{T_1}$ can be satisfied with high probability.

    \noindent \textbf{Step II:} If $M\gtrsim T^{1/\max\{\beta,(\alpha+\beta)/2\}}$, assume that $\breve{\upv}^0$ further satisfies $\breve{\upv}_{D^{\dagger}:M}^0\leq\frac{1}{4}\upv_{D^{\dagger}:M}^*$. Setting $\eta_0 = \eta$ and $K=T_1$, we have 
    \begin{align}
        \breve{\upV}_{\diag}^{t+1}=&\left(\bbE\left[\left(\calI-\eta_t\widetilde{\calG}_{\breve{\upv}}^t\right)\circ\left(\breve{\upw}^t\otimes\breve{\upw}^t\right)\right]\right)_{\diag}+\eta_t^2\left(\bbE\left[\left(\calH_{\breve{\upv}}^t-\widetilde{\calH}_{\breve{\upv}}^t\right)\circ\left(\breve{\upw}^t\otimes\breve{\upw}^t\right)\right]\right)_{\diag}+\eta_t^2\Sigma_{\breve{\upv}}^t\notag
        \\
        \succeq&\left(\calI-\eta_t\breve{\calG}^t\right)\circ\breve{\upV}_{\diag}^t+\eta_t^2\bar{\sigma}^2\upLambda\diag\left\{\overline{\upb}^{\odot2}\right\},\notag
    \end{align}
    for any $t\in[0:T_2-1]$. According to the recursive step above, we obtain
    \begin{align}\label{lower-bound-proof-I}
        \breve{\upV}^{T_2}
        \succeq&\bar{\sigma}^2\sum_{t=1}^{T_2}\eta_t^2\prod_{i=t+1}^{T_2}\left(\upI-\eta_i\breve{\upH}\right)^2\upLambda\diag\left\{\overline{\upb}^{\odot2}\right\}+\underbrace{\left(\upI-\eta_0\breve{\upH}\right)^{2T_2}\left(\breve{\upw}^0\otimes\breve{\upw}^0\right)}_{\calcolII}\notag
        \\
        \succeq&\bar{\sigma}^2\underbrace{\sum_{t=1}^{T_2}\eta_t^2\prod_{i=t+1}^{T_2}\left(\upI-2\eta_i\breve{\upH}\right)\upLambda\diag\left\{\overline{\upb}^{\odot2}\right\}}_{\calcolI}+\calcolII.
    \end{align}
    Recalling the step size decay rule in Algorithm \ref{SGD}, we have
    \begin{align}\label{lower-bound-proof-2}
        \calcolI=&\eta_0^2\sum_{i=1}^h\left(\upI-2\eta_0\breve{\upH}\right)^{h-i}\prod_{j=1}^{L-1}\left(\upI-\frac{\eta_0}{2^{j-1}}\breve{\upH}\right)^K\upLambda\diag\left\{\overline{\upb}^{\odot2}\right\}\notag
	\\
	&+\sum_{l=1}^{L-1}\left(\frac{\eta_0}{2^l}\right)^2\sum_{i=1}^K\left(\upI-\frac{\eta_0}{2^{l-1}}\breve{\upH}\right)^{K-i}\prod_{j=l+1}^{L-1}\left(\upI-\frac{\eta_0}{2^{j-1}}\breve{\upH}\right)^K\upLambda\diag\left\{\overline{\upb}^{\odot2}\right\}\notag
	\\
	=&\frac{\eta_0^2}{12}\sum_{i=1}^h\left(\upI_{1:D}-2\eta_0\breve{\upH}_{1:D}\right)^{h-i}\prod_{j=1}^{L-1}\left(\upI_{1:D}-\frac{\eta_0}{2^{j-1}}\breve{\upH}_{1:D}\right)^K\breve{\upH}_{1:D}\notag
	\\
	&+\frac{1}{12}\sum_{l=1}^{L-1}\left(\frac{\eta_0}{2^l}\right)^2\sum_{i=1}^K\left(\upI_{1:D}-\frac{\eta_0}{2^{l-1}}\breve{\upH}_{1:D}\right)^{K-i}\prod_{j=l+1}^{L-1}\left(\upI_{1:D}-\frac{\eta_0}{2^{j-1}}\breve{\upH}_{1:D}\right)^K\breve{\upH}_{1:D}\notag
	\\
	=&\frac{\eta_0}{24}\left(\upI_{1:D}-\left(\upI_{1:D}-2\eta_0\breve{\upH}_{1:D}\right)^h\right)\left(\prod_{j=1}^{L-1}\left(\upI_{1:D}-\frac{\eta_0}{2^{j-1}}\breve{\upH}_{1:D}\right)\right)^K\notag
	\\
	&+\sum_{l=1}^{L-1}\frac{\eta_0}{12\cdot 2^{l+1}}\left(\upI_{1:D}-\left(\upI_{1:D}-\frac{\eta_0}{2^{l-1}}\breve{\upH}_{1:D}\right)^K\right)\left(\prod_{j=l+1}^{L-1}\left(\upI_{1:D}-\frac{\eta_0}{2^{j-1}}\breve{\upH}_{1:D}\right)\right)^K\notag
	\\
	\overset{\text{(a)}}{\geq}&\frac{\eta_0}{24}\left(\upI_{1:D}-\left(\upI_{1:D}-2\eta_0\breve{\upH}_{1:D}\right)^h\right)\left(\upI_{1:D}-2\eta_0\breve{\upH}_{1:D}\right)^K\notag
	\\
	&+\sum_{l=1}^{L-1}\frac{\eta_0}{12\cdot 2^{l+1}}\left(\upI_{1:D}-\left(\upI_{1:D}-\frac{\eta_0}{2^{l-1}}\breve{\upH}_{1:D}\right)^K\right)\left(\upI_{1:D}-\frac{\eta_0}{2^{l-1}}\breve{\upH}_{1:D}\right)^K,
    \end{align}
    where (a) is derived from following inequality
\begin{align}
	\prod_{i=l+1}^{L-1}\left(\upI_{1:D}-\frac{\eta_0}{2^{i-1}}\breve{\upH}_{1:D}\right)\geq\upI_{1:D}-\sum_{i=l+1}^{L-1}\frac{\eta_0}{2^{i-1}}\breve{\upH}\geq\upI_{1:D}-\frac{\eta_0}{2^{l-1}}\breve{\upH}_{1:D}.\notag
\end{align}
When $h > K$, we apply an auxiliary function analogous to [Lemma D.1, \citet{wu2022last}]'s:
\begin{align}
    f(x):=\frac{x}{2}\left(1-\left(1-2x\right)^h\right)(1-2x)^K+\sum_{l=1}^{L-1}\frac{x}{2^{l+1}}\left(1-\left(1-\frac{x}{2^{l-1}}\right)^K\right)\left(1-\frac{x}{2^{l-1}}\right)^K.\notag
\end{align}
Then, we obtain
\begin{align}\label{lower-bound-proof-III}
    f(\eta_0\breve{\upH})\succeq\frac{1}{4800K}\upI_{1:H_1}+\frac{\eta_0}{480}\breve{\upH}_{H_1+1:H_2}+\frac{\eta_0^2h}{480}\breve{\upH}_{H_2+1:D}^2,
\end{align}
where $H_1:=\min\{D,\max\{i\mid\lambda_i(\upv_i^*)^2\geq\frac{1}{12\eta_0K}\}\}$ and $H_2:=\min\{D,\max\{i\mid\lambda_i(\upv_i^*)^2\geq\frac{1}{12\eta_0h}\}\}$. 

For term $\calcolII$, we have
\begin{align}\label{lower-bias}
    \llangle\breve{\upH},\calcolII\rrangle\gtrsim\begin{cases}
        \sum_{i=D^{\dagger}}^M\lambda_i(\upv_i^*)^4, & \text{ if }M\gtrsim T^{1/\max\{\beta,(\alpha+\beta)/2\}},
        \\
        0, & \text{ otherwise},
    \end{cases}
\end{align}
where the estimation for $\langle\breve{\upH},\calcolII\rangle$ under case $M\gtrsim T^{1/\max\{\beta,(\alpha+\beta)/2\}}$ is derived from the initialization $\breve{\upv}_{D^{\dagger}:M}^0<\frac{1}{4}\upv_{D^{\dagger}:M}^*$ and $(1+\eta\lambda_i(\upv_i^*)^2)^{2T_2}\leq 2$ for any $i\in[D^{\dagger}:M]$.

Therefore, using Eq.~\eqref{lower-bound-I}-Eq.~\eqref{lower-bias}, we derive 
\begin{align}
    \bbE\left[\calR_M(\breve{\upv}^{T_2})-\calR_M(\upv_{1:M}^*)\right]\gtrsim&\bar{\sigma}^2\llangle\breve{\upH},\frac{1}{K}\breve{\upH}_{1:H_1}^{-1}+\eta_0\upI_{H_1+1:H_2}+\eta_0^2h\breve{\upH}_{H_2+1:D}\rrangle+\llangle\breve{\upH},\calcolII\rrangle\notag
    \\
    =&\bar{\sigma}^2\left(\frac{H_1}{K}+\eta_0\sum_{i=H_1+1}^{H_2}\lambda_i(\upv_i^*)^2+\eta_0^2h\sum_{i=H_2+1}^{D}\lambda_i^2(\upv_i^*)^4\right)+\llangle\breve{\upH},\calcolII\rrangle.\notag
\end{align}
According to Lemma \ref{high-probability-phase-II}, we have $\mathbb{P}(t_s\leq T_2)\leq\delta$, which implies that 
\begin{align}\label{lower-bound-proof-IV}
    &\bbE\left[\calR_M(\breve{\upv}^{T_2})-\calR_M(\upv_{1:M}^*)\mid t_s>T_2\right]\notag
    \\
    \geq&\bbE\left[\calR_M(\breve{\upv}^{T_2})-\calR_M(\upv_{1:M}^*)\right]-\sum_{i=1}^{T_2}\bbP(t_s=i)\bbE\left[\calR_M(\breve{\upv}^{T_2})-\calR_M(\upv_{1:M}^*)\mid t_s=i\right]\notag
    \\
    \overset{\text{(b)}}{\gtrsim}&\bar{\sigma}^2\left(\frac{H_1}{K}+\eta_0\sum_{i=H_1+1}^{H_2}\lambda_i(\upv_i^*)^2+\eta_0^2h\sum_{i=H_2+1}^{D}\lambda_i^2(\upv_i^*)^4\right)+\llangle\breve{\upH},\calcolII\rrangle.
\end{align}
Since $\delta$ is sufficiently small, (b) is drawn from two facts: 1) $\breve{\upv}^{t_s}$ resides in a bounded neighborhood of $\widehat{\upb}$ or $\overline{\upb}$; 2) the risk upper bound for SGD established in [Theorem 4.1, \citet{wu2022last}]. The lower bound established in Eq.~\eqref{lower-bound-proof-IV} is uniformly valid for all $\breve{\upv}^0\in[\overline{\upb},\widehat{\upb}]$. Denote event 
$$
\mathcal{K}\left(\upv^{T_1}\right):=\left\{\overline{\upb}\leq\upv^{T_1}\leq\widehat{\upb}\bigwedge \left\{\upv_{D^{\dagger}:M}^{T_1}\leq\frac{1}{4}\upv_{D^{\dagger}:M}^*, \text{ if }M\gtrsim T^{1/\max\{\beta,(\alpha+\beta)/2\}}\right\}\right\}.
$$
For $t_s > T_2$, the trajectory $\{\breve{\upv}^t\}_{t=0}^{T_2}$ aligns with Algorithm \ref{SGD}'s iterations over $[T_1:T]$, given the initialization $\breve{\upv}^0 = \upv^{T_1}$ with $\mathcal{K}(\upv^{T_1})$ occurs. Then, we have
\begin{align}\label{lower-bound-proof-V}
    &\min_{\upv^{T_1}}\bbE\left[\calR_M(\upv^{T})-\calR_M(\upv_{1:M}^*)\mid \mathcal{K}(\upv^{T_1})\right]\notag
    \\
    \geq&(1-\delta)\min_{\upv^{T_1}}\bbE\left[\calR_M(\breve{\upv}^{T_2})-\calR_M(\upv_{1:M}^*)\mid t_s>T_2\bigwedge\breve{\upv}^0 = \upv^{T_1}\bigwedge\mathcal{K}(\upv^{T_1})\right]\notag
    \\
    \gtrsim&\bar{\sigma}^2\left(\frac{H_1}{K}+\eta_0\sum_{i=H_1+1}^{H_2}\lambda_i(\upv_i^*)^2+\eta_0^2h\sum_{i=H_2+1}^{D}\lambda_i^2(\upv_i^*)^4\right)+\mathds{1}_{M\gtrsim T^{1/\max\{\beta,(\alpha+\beta)/2\}}}\sum_{i=D^{\dagger}}^M\lambda_i(\upv_i^*)^4.
\end{align}
Noticing that $\mathcal{K}(\upv^{T_1})$ occurs with probability at least $1-\delta$, and combining Eq.~\eqref{lower-bound-proof-IV} with Eq.~\eqref{lower-bound-proof-V}, we obtain
\begin{align}
    \bbE\left[\calR_M(\upv^{T})-\calR_M(\upv_{1:M}^*)\right]\geq&(1-\delta)\min_{\upv^{T_1}}\bbE\left[\calR_M(\upv^{T})-\calR_M(\upv_{1:M}^*)\mid\mathcal{K}(\upv^{T_1})\right]\notag
    \\
    \gtrsim&\bar{\sigma}^2\left(\frac{H_1}{K}+\eta_0\sum_{i=H_1+1}^{H_2}\lambda_i(\upv_i^*)^2+\eta_0^2h\sum_{i=H_2+1}^{D}\lambda_i^2(\upv_i^*)^4\right)\notag
    \\
    &+\mathds{1}_{M\gtrsim T^{1/\max\{\beta,(\alpha+\beta)/2\}}}\sum_{i=D^{\dagger}}^M\lambda_i(\upv_i^*)^4,\notag
\end{align}
where $H_1:=\min\{D,\max\{i\mid\lambda_i(\upv_i^*)^2\geq\frac{1}{12\eta_0K}\}\}$ and $H_2:=\min\{D,\max\{i\mid\lambda_i(\upv_i^*)^2\geq\frac{1}{12\eta_0h}\}\}$. We complete the proof of lower bound.
\end{proof}

\section{Auxiliary Lemma}
\begin{definition}[Sub-Gaussian Random Variable]\label{sub-Gaussian}
   A random variable $x$ with mean $\bbE x$ is sub-Gaussian if there is $\sigma\in\bbR_+$ such that
   \begin{align}
       \bbE\left[e^{\lambda(x-\bbE x)}\right]\leq e^{\frac{\lambda^2\sigma^2}{2}},\quad \forall\lambda\in\bbR.\notag
   \end{align}
\end{definition}

\begin{proposition}\label{prop-A5}[\citep{wainwright2019high}]
    For a random variable $x$ which satisfies the sub-Gaussian condition \ref{sub-Gaussian} with parameter $\sigma$, we have
    \begin{align}\label{eq-subGaussian-prob}
        \bbP\left(|x-\bbE x|>c\right)\leq 2e^{-\frac{c^2}{2\sigma^2}},\quad \forall c>0.
    \end{align}
\end{proposition}

\begin{lemma}\label{aux-1}
    Let $X_1,\cdots,X_n$ be independent and symmetric stochastic variables with zero mean. Denote $Y=\sum_{i=1}^n \upv_iX_i$ $\mathds{1}_{|X_i|\leq R}$ for any unit vector $\upv\in\bbR^n$ and positive scalar $R$. Then, we have $YX_1\mathds{1}_{|X_1|\leq R}$ is sub-Gaussian with parameter at most $\sigma=2R^2\|\upv\|_2$.
\end{lemma}
\begin{proof}
    For simplicity, we denote $\hat{X}_i:=X_i\mathds{1}_{|X_i|\leq R}$ for any $i\in[1:n]$, and $Y_{-1}=\sum_{i=2}^n\upv_i\hat{X}_i$. One can notice the following holds
    \begin{align}\label{sub-gau-1}
        \bbE\left[e^{\lambda \left(Y\hat{X}_1-\bbE[Y\hat{X}_1]\right)}\right]=\bbE\left[e^{\lambda\upv_i\left(\hat{X}_1^2-\bbE[\hat{X}_1^2]\right)}\bbE\left[e^{\lambda\left(Y_{-1}\hat{X}_1-\bbE[Y_{-1}\hat{X}_1]\right)}\mid \hat{X}_1\right]\right],
    \end{align}
    for any $\lambda\in\bbR$. Letting $\hat{X}_i'$ be an independent copy of $\hat{X}_i$ for any $i\in[1:n]$, then we have
    \begin{align}\label{sub-gau-2}
        \bbE\left[e^{\lambda\left(Y_{-1}\hat{X}_1-\bbE[Y_{-1}\hat{X}_1]\right)}\mid \hat{X}_1\right]=&\bbE\left[e^{\sum_{i=2}^n\lambda\upv_i(\hat{X}_i-\bbE[\hat{X}_i])\hat{X}_1}\mid\hat{X}_1\right]
        \notag
        \\
        \overset{\text{(a)}}{\leq}&\bbE\left[e^{\sum_{i=2}^n\lambda\upv_i(\hat{X}_i-\hat{X}_i')\hat{X}_1}\mid\hat{X}_1\right],
    \end{align}
    where (a) is derived from the convexity of the exponential, and Jensen’s inequality. Letting $\xi$ be an independent Rademacher variable, since the distribution of $\hat{X}_i-\hat{X}_i'$ is the same as that of $\xi(\hat{X}_i-\hat{X}_i')$ for any $i\in[1:n]$, we obtain
    \begin{align}\label{sub-gau-3}
        \bbE\left[e^{\sum_{i=2}^n\lambda\upv_i(\hat{X}_i-\hat{X}_i')\hat{X}_1}\mid\hat{X}_1\right]
        \overset{\text{(b)}}{\leq}\prod_{i=2}^n\bbE\left[e^{\frac{\lambda^2\upv_i^2}{2}\hat{X}_1^2(\hat{X}_i-\hat{X}_i')^2}\mid\hat{X}_1\right].
    \end{align}
    Noticing that $|\hat{X}_i-\hat{X}_i'|\leq 2R$ and $|\hat{X}_i|\leq R$ for any $i\in[1:n]$, we are guarantee that
    \begin{align}\label{sub-gau-4}
        \prod_{i=2}^n\bbE\left[e^{\frac{\lambda^2\upv_i^2}{2}\hat{X}_1^2(\hat{X}_i-\hat{X}_i')^2}\mid\hat{X}_1\right]\leq e^{2\lambda^2R^4\sum_{i=2}^n\upv_i^2}.
    \end{align}
    Combining Eq.~\eqref{sub-gau-1}-Eq.~\eqref{sub-gau-4} and applying similar technique, we have
    \begin{align}
        \bbE\left[e^{\lambda \left(Y\hat{X}_1-\bbE[Y\hat{X}_1]\right)}\right]\leq& e^{2\lambda^2R^4\sum_{i=2}^n\upv_i^2}\bbE\left[e^{\lambda\upv_i\left(\hat{X}_1^2-\bbE[\hat{X}_1^2]\right)}\right]\leq e^{2\lambda^2R^4\|\upv\|_2^2}.\notag
    \end{align}
\end{proof}

\begin{lemma}\label{aux-1.1}
    Consider a stochastic variable $X$ which is zero-mean and sub-Gaussian with parameter $\sigma$ for some $\sigma>0$. Then, there exists $R>0$ which depends on $\sigma$ such that
    \begin{align}
        \bbE\left[X^2\mathds{1}_{|X|\leq R}\right]\geq\frac{1}{2}\bbE\left[X^2\right].
    \end{align}
\end{lemma}
\begin{proof}
    According to Eq.~\eqref{eq-subGaussian-prob}, we have $\bbP(|X|\geq r)\leq 2e^{-\frac{r^2}{2\sigma^2}}$ for any $r>0$. Therefore, we obtain
    \begin{align}
        \bbE\left[X^2\mathds{1}_{|X|>R}\right]\overset{\text{(a)}}{=}&2\int_{0}^{\infty}r\bbP(|X|\mathds{1}_{|X|>R}>r)\text{d}r\notag
        \\
        =&2\int_R^{\infty}r\bbP(|X|>r)\text{d}r+R^2\bbP(|X|>R)\notag
        \\
        \leq&4\int_R^{\infty}re^{-\frac{r^2}{2\sigma^2}}\text{d}r+2R^2e^{-\frac{R^2}{2\sigma^2}}=4\sigma^2e^{-\frac{R^2}{\sigma^2}}+2R^2e^{-\frac{R^2}{2\sigma^2}},
    \end{align}
    where (a) is derived from [Lemma 2.2.13, \citet{wainwright2019high}]. 
\end{proof}

\begin{lemma}\label{aux-2}
	Let $c>0$, $\gamma<1$ and $a_t>0$ for any $t\in[0:T-1]$. Consider a sequence of random variables $\{v^i\}_{i=0}^{T-1}\subset[0,c]$, which satisfies either $v^t=v^{t+1}=\cdots=v^T$, or $\bbE\left[v^{t+1}\mid\calF^t\right]\leq(1-\eta_t)v^t$ with stepsize $\eta_t\geq0$, given $\bbE[e^{\lambda(v^{t+1}-\bbE[v^{t+1}\mid\calF^t])}\mid\calF^t]\leq e^{\frac{\lambda^2a_t^2}{2}}$ almost surely for any $\lambda\in\bbR$. Then, there is 
    \begin{align}
        \bbP\left(v^T>c\bigwedge v^0\leq\gamma c\right)\leq\max_{t\in[1:T]}\exp\left\{-\frac{(1-\gamma)^2c^2}{2\sum_{j=0}^{t-1}a_j^2\prod_{i=j+1}^{t-1}(1-\eta_i)^{2}}\right\}.\notag
    \end{align}
\end{lemma}
\begin{proof}
    Similarly, we begin with constructing a sequence of couplings $\{\Tilde{v}^i\}_{i=0}^T$ as follows: $\Tilde{v}^0=v^0$; if $v^t=v^{t+1}=\cdots=v^T$, let $\Tilde{v}^{t+1}=(1-\eta_t)\Tilde{v}^t$; otherwise, let $\Tilde{v}^{t+1}=v^{t+1}$. Notice that $\prod_{i=0}^{t-1}(1-\eta_i)^{-1}\tilde{v}^t$ is a supermartingale. We define $D_{t+1}:=\prod_{i=0}^{t}(1-\eta_i)^{-1}\tilde{v}^{t+1}-\prod_{i=0}^{t-1}(1-\eta_i)^{-1}\tilde{v}^t$ for any $t\in[0:T-1]$. Therefore, applying iterated expectation yields 
    \begin{align}\label{eq-sub-gaussian}
        \bbE\left[e^{\lambda\left(\sum_{i=1}^tD_i\right)}\right]=&\bbE\left[e^{\lambda\left(\sum_{i=1}^{t-1}D_{i}\right)}\bbE\left[e^{\lambda D_t}\mid\calF^{t-1}\right]\right]\notag
        \\
        =&\bbE\left[e^{\lambda\left(\sum_{i=1}^{t-1}D_{i}\right)}\bbE\left[e^{\frac{\lambda}{\prod_{i=0}^{t-1}(1-\eta_i)} \left(v^{t}-(1-\eta_{t-1})v^{t-1}\right)}\mid\calF^{t-1}\right]\right]
        \notag
        \\
        \overset{\text{(a)}}{\leq}&\bbE\left[e^{\lambda\left(\sum_{i=1}^{t-1}D_{i}\right)}\bbE\left[e^{\frac{\lambda}{\prod_{i=0}^{t-1}(1-\eta_i)} \left(v^{t}-\bbE[v^t\mid\calF^{t-1}]\right)}\mid\calF^{t-1}\right]\right]
        \notag
        \\
        \overset{\text{(b)}}{\leq}&e^{\frac{\lambda^2a_{t-1}^2}{2\prod_{i=0}^{t-1}(1-\eta_i)^2}}\bbE\left[e^{\lambda\left(\sum_{i=1}^{t-1}D_{i}\right)}\right]
        \notag
        \\
        \leq&e^{\frac{\lambda^2\sum_{j=0}^{t-1}a_j^2\prod_{i=0}^j(1-\eta_i)^{-2}}{2}},
    \end{align}
    for any $\lambda\in\bbR^+$ and $t\in[1:T]$, where (a) is derived from that $\lambda(\bbE[v^t\mid\calF^{t-1}]-(1-\eta_{t-1})v^{t-1})\leq 0$ and (b) follows from the condition that $\bbE[e^{\lambda(v^{t+1}-(1-\eta_t)v^t)}\mid\calF^t]\leq e^{\frac{\lambda^2a^2}{2}}$ almost surely for any $\lambda\in\bbR$. Then we obtain 
    \begin{align}
        \bbP\left(v^T>c\bigwedge v^0\leq\gamma c\right){\leq}&\max_{t\in[1:T]}\bbP\left(\prod_{i=0}^{t-1}(1-\eta_i)^{-1}\Tilde{v}^{t}>\prod_{i=0}^{t-1}(1-\eta_i)^{-1}c\bigwedge\Tilde{v}^0\leq\gamma c\right)\notag
        \\
        \leq&\max_{t\in[1:T]}\min_{\lambda>0}\frac{\bbE\left[e^{\lambda(\sum_{i=1}^tD_i)}\right]}{e^{\lambda\left(\prod_{i=0}^{t-1}(1-\eta_i)^{-1}c-\gamma c\right)}}
        \notag
        \\
        \overset{\text{(b)}}{\leq}&\max_{t\in[1:T]}\exp\left\{-\frac{\left(\prod_{i=0}^{t-1}(1-\eta_i)^{-1}c-\gamma c\right)^2}{2\sum_{j=0}^{t-1}a_j^2\prod_{i=0}^{j}(1-\eta_i)^{-2}}\right\}\notag
        \\
        \leq&\max_{t\in[1:T]}\exp\left\{-\frac{(1-\gamma)^2\left(\prod_{i=0}^{t-1}(1-\eta_i)^{-1}c\right)^2}{2\sum_{j=0}^{t-1}a_j^2\prod_{i=0}^{j}(1-\eta_i)^{-2}}\right\}\notag
        \\
        =&\max_{t\in[1:T]}\exp\left\{-\frac{(1-\gamma)^2c^2}{2\sum_{j=0}^{t-1}a_j^2\prod_{i=j+1}^{t-1}(1-\eta_i)^{2}}\right\},
    \end{align}
    where (b) is derived from Eq.~\eqref{eq-sub-gaussian}.
\end{proof}

\begin{corollary}\label{aux-coro-3}
    Let $c>0$, $\gamma<1$ and $a_t>0$ for any $t\in[0:T-1]$. Consider a sequence of random variables $\{v^i\}_{i=0}^{T-1}\subset[0,c]$, which satisfies $\prod_{i=0}^{T-1}(1+\eta_t)^{-1}c-v^0\geq\gamma c$ and $\bbE\left[v^{t+1}\mid\calF^t\right]\leq(1+\eta_t)v^t$ with stepsize $\eta_t\geq0$, given $\bbE[e^{\lambda(v^{t+1}-\bbE[v^{t+1}\mid\calF^t])}\mid\calF^t]\leq e^{\frac{\lambda^2a_t^2}{2}}$ almost surely for any $\lambda\in\bbR$. Then, there is 
    \begin{align}
        \bbP\left(v^T>c\right)\leq\max_{t\in[1:T]}\exp\left\{-\frac{\gamma^2c^2}{2\sum_{j=0}^{t-1}a_j^2\prod_{i=0}^{j}(1+\eta_i)^{-2}}\right\}\notag.
    \end{align}
\end{corollary}

\begin{lemma}\label{aux-3}
    For $L,K\in\bbN_+$, consider $T\in\bbN^+$ such that $LK\leq T<(L+1)K$. Then we have 
    \begin{align}
        \sum_{t=0}^{T}\left(\prod_{i=t}^{T}(1-c\eta_t)\right)\eta_t^2\leq\frac{2\eta_0}{c},
    \end{align}
    where $\eta_t=\frac{\eta_0}{2^l}$ if $lK\leq t\leq \min\{(l+1)K-1,T\}$ for any $l\in[0:L]$ and $c>0$ is a constant.
\end{lemma}
\begin{proof}
    For any $l\in[0:L]$, we have
    \begin{align}\label{aux-eq-lemma3}
        \sum_{t=lK}^{(l+1)K-1}\left(\prod_{i=t}^{T}\left(1-c\eta_t\right)\right)\eta_t^2=&\eta_{lK}^2\left(\prod_{i=(l+1)K}^{T}\left(1-c\eta_t\right)\right)\sum_{t=lK}^{(l+1)K-1}(1-c\eta_{lK})^{(l+1)K-1-t}\notag
        \\
        \leq&\frac{\eta_{lK}}{c}\left(\prod_{i=(l+1)K}^{T}\left(1-c\eta_t\right)\right).
    \end{align}
    Therefore, we obtain the following estimation
    \begin{align}
        \sum_{t=0}^{T}\left(\prod_{i=t}^{T}(1-c\eta_t)\right)\eta_t^2\leq&\sum_{t=0}^{LK-1}\left(\prod_{i=t}^{T}(1-c\eta_t)\right)\eta_t^2+\sum_{t=LK}^T(1-c\eta_{LK})^{T-t}\eta_{LK}^2\notag
        \\
        \overset{\text{(a)}}{\leq}&\frac{\sum_{l=0}^L\eta_{lK}}{c}\leq\frac{2\eta_0}{c}.
    \end{align}
\end{proof}

    

\begin{lemma}\label{aux-5}
    Under Assumption \ref{ass-ss} and the setting of Theorem \ref{phase-II-main}, we have 
    $$
    \eta(\upI-\eta\widehat{\upH})^{2t}\upH\preceq\frac{25}{t+1}\upI,
    $$
    for any $t\in[0:T-1]$.
\end{lemma}
\begin{proof}
    For index $i\in[1:D]$, we have 
    \begin{align}
    \eta(\upI_{1:D}-\eta\widehat{\upH}_{1:D})^{2t}\upH_{1:D}=25 \eta(\upI_{1:D}-\eta\widehat{\upH}_{1:D})^{2t}\widehat{\upH}_{1:D}\preceq\frac{25}{t+1}\upI_{1:D},\notag
    \end{align}
    since $(1-x)^t\leq\frac{1}{(t+1)x}$ for any $x\in(0,1)$. For index $i\in[D+1:M]$, we obtain 
    \begin{align}
        \eta\upH_{i,i}\leq\frac{1}{T}\leq\frac{1}{t+1},
    \end{align}
    according to the parameter setting in Theorem \ref{phase-II-main} for any $t\in[1:T-1]$.
\end{proof}

\begin{lemma}\label{aux-6}
    Suppose Assumption \ref{ass-d} hold and let $\upz=\Pi_M\upx\in\bbR^M$. Then there exists a constant $\gamma>0$ such that 
    \begin{align}\label{aux-matrix}
        \bbE\left[\upA\upz\|\upz\|_{\upA^{\top}\upB\upA}^2\upz^{\top}\upA^{\top}\right]\preceq\gamma\llangle\upA\bbE\left[\upz\upz^{\top}\right]\upA^{\top},\upB\rrangle\upA\bbE\left[\upz\upz^{\top}\right]\upA^{\top},
    \end{align}
    for any diagonal PSD matrix $\upA\in\bbR^{M\times M}$ and PSD matrix $\upB\in\bbR^{M\times M}$.
\end{lemma}
\begin{proof}
    We denote $\upD:=\bbE[\upA\upz\|\upz\|_{\upA^{\top}\upB\upA}^2\upz^{\top}\upA^{\top}]$. For any $i,j\in[1:M]$ and $i\neq j$, we have $\upD_{i,j}=2\lambda_i\lambda_j\upA_{i,i}\upA_{j,j}(\upA^{\top}\upB\upA)_{i,j}$. In addition, we also have 
    $$
    \upD_{i,i}=\bbE\left[\|\upz\|_{\upA^{\top}\upB\upA}^2\right]\upA_{i,i}^2\lambda_i+(\upA^{\top}\upB\upA)_{i,i}\upA_{i,i}^2\Var\left[\upz_i^2\right]\leq (C+1)\bbE\left[\|\upz\|_{\upA^{\top}\upB\upA}^2\right]\upA_{i,i}^2\lambda_i.
    $$
    Therefore, we obtain that
    \begin{align}
        \upD\preceq&(C+1)\bbE\left[\|\upz\|_{\upA^{\top}\upB\upA}^2\right]\upA\bbE[\upz\upz^{\top}]\upA^{\top}+2\upA\bbE[\upz\upz^{\top}]\upA^{\top}\upB\upA\bbE[\upz\upz^{\top}]\upA^{\top}\notag
        \\
        \preceq&(C+1)\bbE\left[\|\upz\|_{\upA^{\top}\upB\upA}^2\right]\upA\bbE[\upz\upz^{\top}]\upA^{\top}\notag
        \\
        &+2\left\|\left(\upA\bbE[\upz\upz^{\top}]\upA^{\top}\right)^{1/2}\upB\left(\upA\bbE[\upz\upz^{\top}]\upA^{\top}\right)^{1/2}\right\|_2^2\upA\bbE[\upz\upz^{\top}]\upA^{\top}\notag
        \\
        \overset{\text{(a)}}{\preceq}&(C+2)\llangle\upA\bbE\left[\upz\upz^{\top}\right]\upA^{\top},\upB\rrangle\upA\bbE[\upz\upz^{\top}]\upA^{\top},\notag
    \end{align}
    where (a) is derived from that $\langle\upA\bbE[\upz\upz^{\top}]\upA^{\top},\upB\rangle=\bbE[\|\upz\|_{\upA^{\top}\upB\upA}^2]$ and $\|\upH^{1/2}\upB\upH^{1/2}\|_2^2\leq\langle\upH,\upB\rangle$ for any PSD matrix $\upB,\upD\in\bbR^{M\times M}$ since 
    \begin{align}
        \upa^{\top}\upH^{1/2}\upB\upH^{1/2}\upa=&\llangle\upH^{1/2}\upa\upa^{\top}\upH^{1/2},\upB\rrangle\notag
        \\
        \leq&\llangle\upH^{1/2}\upa\upa^{\top}\upH^{1/2},\upB\rrangle+\llangle\upH^{1/2}\upa_{\perp}\upa_{\perp}^{\top}\upH^{1/2},\upB\rrangle\notag
        \\
        =&\llangle\upH,\upB\rrangle.\notag
    \end{align}
    Therefore, by choosing $\gamma=(C+2)$, we obtain Eq.~\eqref{aux-matrix}.
\end{proof}

\begin{lemma}\label{aux-7}
    Under the setting of Theorem \ref{phase-II-main}, suppose following inequality holds 
    \begin{align}
        \upB_{\diag}^{t+1}\preceq\left(\calI-\eta\widehat{\calG}\right)^{t+1}\circ\upB_{\diag}^0+\tau\eta\sum_{i=0}^t\frac{\llangle\upH,\upB^i\rrangle}{t+1-i}\cdot\upI,\notag
    \end{align}
    for any $t\in[0:T-1]$ and some constant $\tau>0$. Then, we have
    \begin{align}
        \sum_{i=0}^{t}\frac{\langle\upH,\upB^i\rangle}{t+1-i}\leq\llangle\sum_{i=0}^t\frac{(\upI-\eta\widehat{\upH})^{2i}\upH}{t+1-i},\upB^0\rrangle+2\tau\eta\log(t)\tr(\upH)\sum_{i=0}^t\frac{\langle\upH,\upB^i\rangle}{t+1-i},\notag
    \end{align}
    for any $t\in[1:T]$.
\end{lemma}
\begin{proof}
    According to the condition of this lemma, we have 
    \begin{align}\label{aux-eq-5.7-1}
        \llangle\upH,\upB^{t}\rrangle\leq&\llangle(\upI-\eta\widehat{\upH})^{2t}\upH,\upB^0\rrangle+\tau\eta\tr(\upH)\sum_{i=0}^{t-1}\frac{\llangle\upH,\upB^i\rrangle}{t-i}.
    \end{align}
    Applying Eq.~\eqref{aux-eq-5.7-1} to each $\langle\upH,\upB^t\rangle$, we obtain
    \begin{align}
        \sum_{i=0}^{t}\frac{\langle\upH,\upB^i\rangle}{t+1-i}\leq&\llangle\sum_{i=0}^t\frac{(\upI-\eta\widehat{\upH})^{2i}\upH}{t+1-i},\upB^0\rrangle+\tau\eta\tr(\upH)\sum_{i=0}^t\sum_{k=0}^{i-1}\frac{\langle\upH,\upB^i\rangle}{(t+1-i)(i-k)}\notag
        \\
        \leq&\llangle\sum_{i=0}^t\frac{(\upI-\eta\widehat{\upH})^{2i}\upH}{t+1-i},\upB^0\rrangle+\tau\eta\tr(\upH)\sum_{k=0}^{t-1}\frac{\langle\upH,\upB^k\rangle}{t+1-k}\sum_{i=k+1}^t\left(\frac{1}{t+1-i}+\frac{1}{i-k}\right)\notag
        \\
        \leq&\llangle\sum_{i=0}^t\frac{(\upI-\eta\widehat{\upH})^{2i}\upH}{t+1-i},\upB^0\rrangle+2\tau\eta\log(t)\tr(\upH)\sum_{k=0}^t\frac{\langle\upH,\upB^k\rangle}{t+1-k}.\notag
    \end{align}
\end{proof}

\section{Simulations}
In this paper, we present simulations in a finite but large dimension ($d=10,000$). We artificially generate samples from the model $y=\left \langle \mathbf{x}, \left ( \mathbf{v}^* \right )^{\odot 2}  \right \rangle+\xi $, where $\mathbf{x}\sim \mathcal{N} \left ( \mathbf{0},\mathbf{H}   \right ) $, $\mathbf{H}=\mathrm{diag}\left \{ i^{-\alpha } \right \} $, $ \mathbf{w} ^*_i=i^{-\frac{\beta-\alpha }{4} }$, and $\xi\sim \mathcal{N} \left (0,1   \right )$ is independent of $\mathbf{x}$. In our simulations, given a total of $T$ iteration, we assume that Algorithm~\ref{SGD} can access $T$ independent samples $\left \{ \left ( \mathbf{x}_i,y_i  \right )  \right \} _{i=1}^T$ generated by the above model. $h$ in Algorithm~\ref{SGD} is set to $\frac{T}{\log_2(T)} $. We numerically approximate the expected error by averaging the results of 100 independent repetitions of the experiment.
In the following, we detail the specific experimental settings and present the results obtained for each scenario.
\begin{itemize}
    \item \textbf{Figure~\ref{simulation-appendix} (a):} We compare the curve of mean error of SGD against the number of iteration steps for both linear and quadratic models, under the setting $\alpha=3$, $\beta=2$ and $T=500$.  The results show that the quadratic model exhibits a phase of diminishing error, while the linear model demonstrates a continuous, steady decrease in error.
    \item \textbf{Figure~\ref{simulation-appendix} (b):} We compare the curve of mean error of SGD against the number of iteration steps for both linear and quadratic models, under the setting $\alpha=2.5$, $\beta=1.5$ and $T=500$.  The results show that the quadratic model exhibits a phase of diminishing error, while the linear model demonstrates a continuous, steady decrease in error.
    \item \textbf{Figure~\ref{simulation-appendix} (c):} We compare the curve of mean error of SGD against the number of sample size for both linear and quadratic models, under the setting $\alpha=3$, $\beta=2$ and $T$ ranging from $1000$ to $5000$. The results indicate that the quadratic model outperforms the linear model and exhibits convergence behavior that is closer to the theoretical algorithm rate.
    \item \textbf{Figure~\ref{simulation-appendix} (d):} We compare the curve of mean error of SGD against the number of sample size for both linear and quadratic models, under the setting $\alpha=2.5$, $\beta=1.5$ and $T$ ranging from $1000$ to $5000$. The results indicate that the quadratic model outperforms the linear model and exhibits convergence behavior that is closer to the theoretical algorithm rate.
    \item \textbf{Figure~\ref{simulation-appendix} (e):} We compare the curve of mean error of SGD against the number of sample size for quadratic models with model size $M=10,30,50,100,200$, under the setting $\alpha=3$, $\beta=2$ and $T$ ranging from $1$ to $10000$. 
    The results show that for a fixed $M$, when $T$ is small, the convergence rate approaches the rate observed as $M\to \infty$. As $T$ increases sufficiently, the convergence rate stabilizes. Increasing $M$ results in an increase in the value of at which this stabilization occurs, which is consistent with the scaling law.
    \item \textbf{Figure~\ref{simulation-appendix} (f):} We compare the curve of mean error of SGD against the number of sample size for quadratic models with model size $M=10,30,50,100,200$, under the setting $\alpha=2.5$, $\beta=1.5$ and $T$ ranging from $1$ to $10000$. The results exhibit similar patterns to those observed in the previous figure.
\end{itemize}

\begin{figure}[t]\label{simulation-appendix}
\centering

\subfigure[\scriptsize{Quadratic v.s. Linear Model}]{
\begin{minipage}[t]{0.33\linewidth}
\centering
\includegraphics[width=2.2 in]{phase2.pdf}
\end{minipage}%
}%
\subfigure[\scriptsize{Quadratic v.s. Linear Model }]{
\begin{minipage}[t]{0.33\linewidth}
\centering
\includegraphics[width=2.2 in]{phase.pdf}
\end{minipage}%
}%
\subfigure[\scriptsize{Empirical v.s. Theoretical Results}]{
\begin{minipage}[t]{0.33\linewidth}
\centering
\includegraphics[width=2.2 in]{lin_vs_qua.pdf}
\end{minipage}%
}%

\subfigure[\scriptsize{Empirical v.s. Theoretical Results}]{
\begin{minipage}[t]{0.33\linewidth}
\centering
\includegraphics[width=2.2 in]{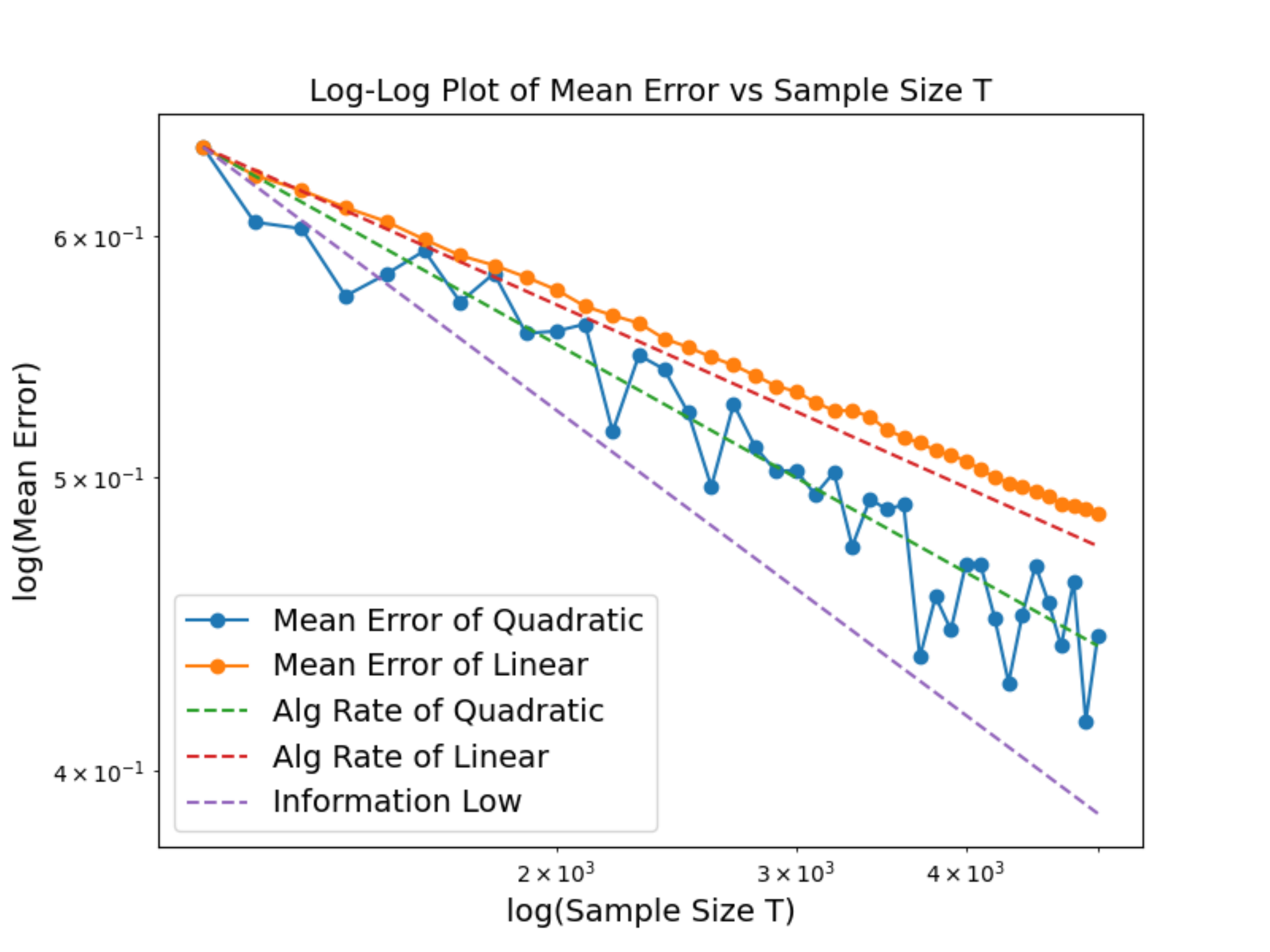}
\end{minipage}%
}%
\subfigure[\scriptsize{Scaling Law }]{
\begin{minipage}[t]{0.33\linewidth}
\centering
\includegraphics[width=2.2 in]{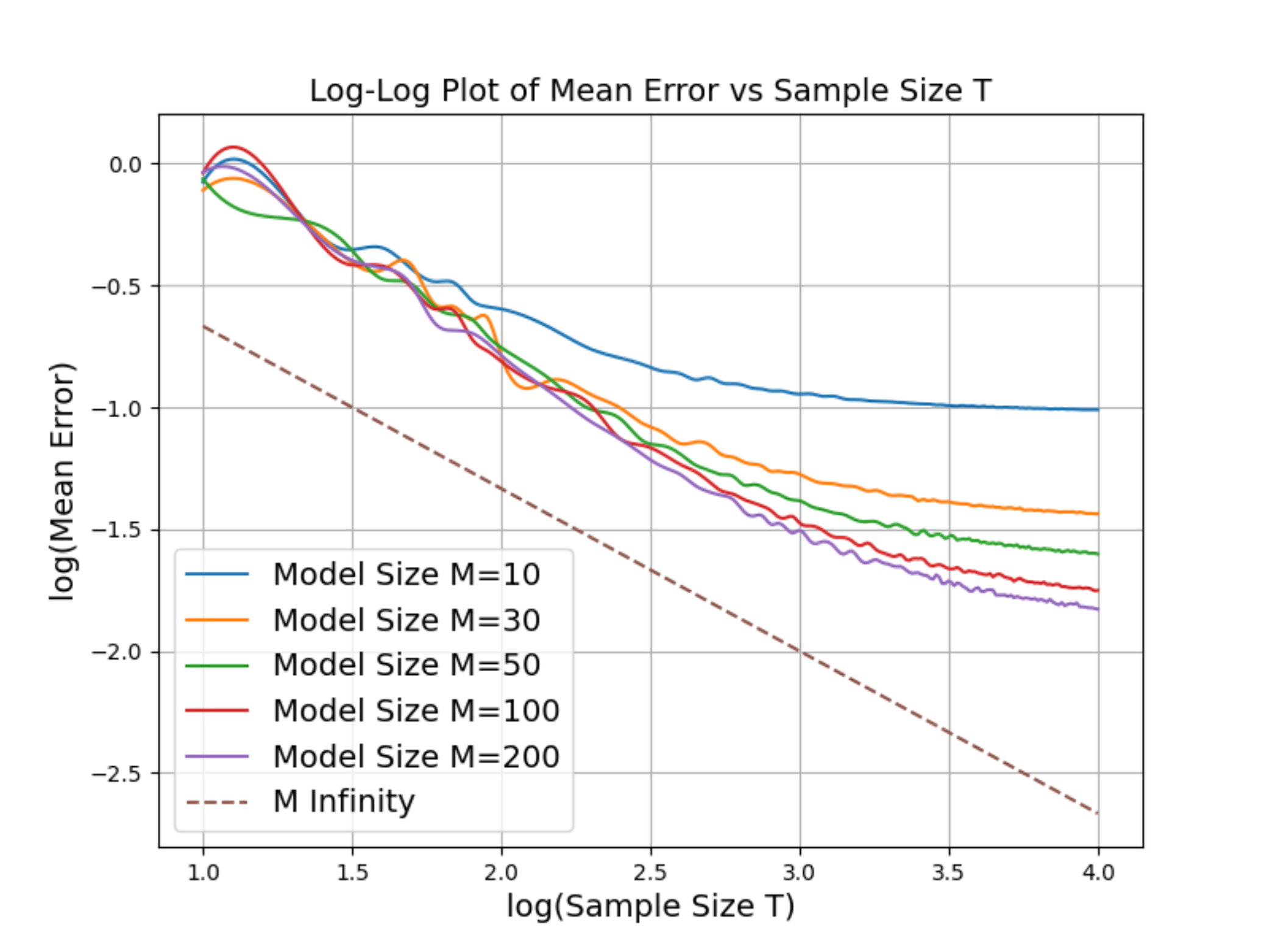}
\end{minipage}%
}%
\subfigure[\scriptsize{Scaling Law}]{
\begin{minipage}[t]{0.33\linewidth}
\centering
\includegraphics[width=2.2 in]{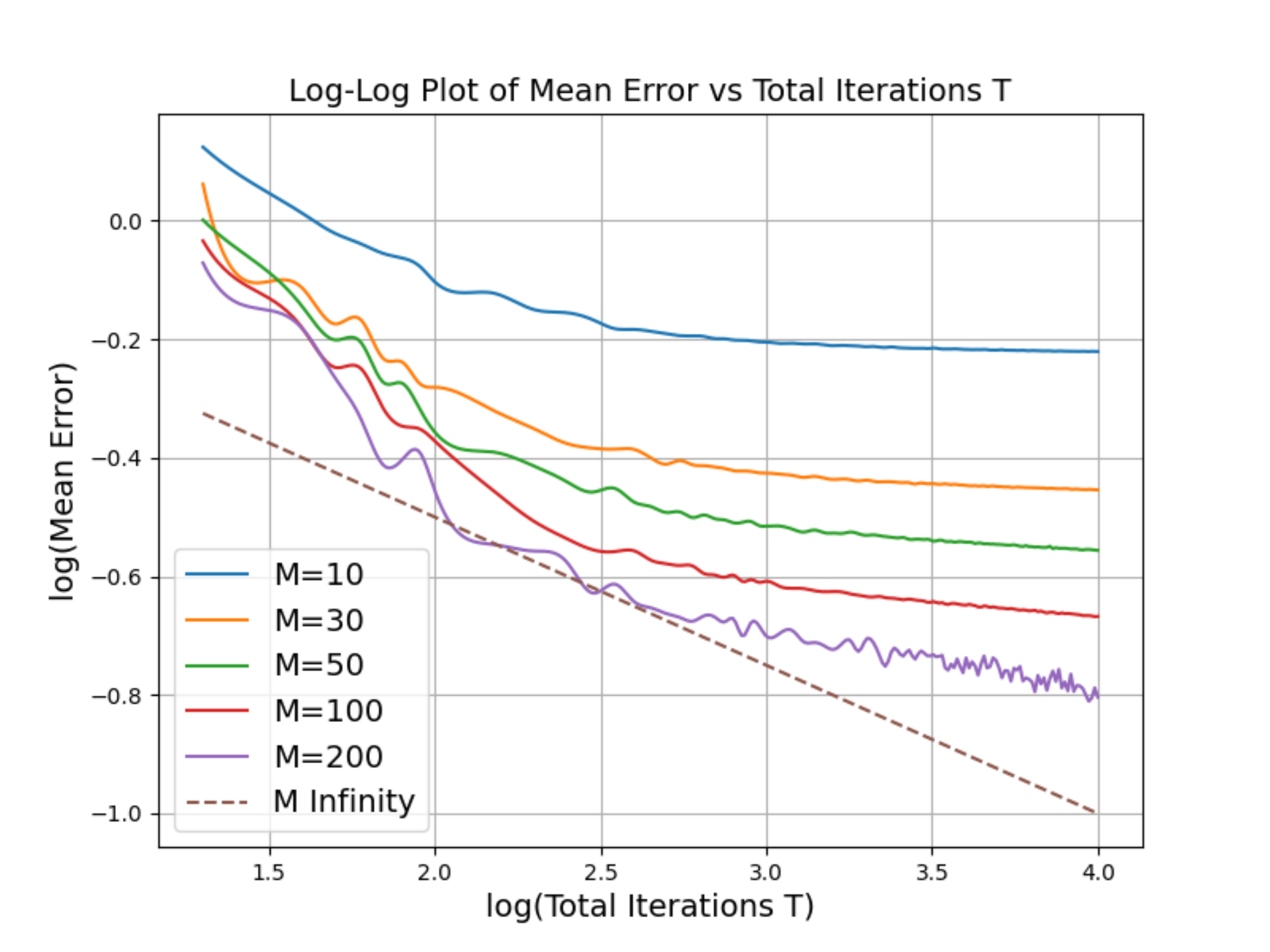}
\end{minipage}%
}%
\centering
\caption{Numerical simulation results.}
\end{figure}

\bibliography{yourbibfile}
\bibliographystyle{johd}
\newpage

\end{document}